\documentclass[oneside,letterpaper]{article}

\usepackage[utf8]{inputenc}
\usepackage{amsmath, amssymb, amsthm}
\usepackage{graphicx}
\usepackage{hyperref}
\usepackage{geometry}
\usepackage{times}  
\usepackage{authblk}  

\usepackage[utf8]{inputenc} 
\usepackage[T1]{fontenc}    
\usepackage{hyperref}       
\usepackage{url}            
\usepackage{booktabs}       
\usepackage{amsfonts}       
\usepackage{nicefrac}       
\usepackage{microtype}      
\usepackage{xcolor}         
\usepackage{bm}
\usepackage{mathtools}
\usepackage[noend, linesnumbered, ruled,vlined]{algorithm2e}
\SetKw{and}{\textbf{ and }}
\SetKwInput{init}{Init}
\usepackage{mathrsfs}
\usepackage{tabularx,ragged2e,caption}
\usepackage{thmtools}
\usepackage{thm-restate}
\usepackage{graphicx} 
\usepackage{subcaption} 

\definecolor{darkblue}{rgb}{0,0,0.7}
\usepackage{hyperref}
\hypersetup{colorlinks=true,linkcolor=blue,filecolor=magenta,urlcolor=cyan,citecolor=darkblue}
\usepackage{amsmath,amssymb,amsthm,graphicx}

\usepackage[round,sort&compress]{natbib}

\newcommand{\ushort}[1]{{\underline{#1}}} 
\newcommand{\stepname}[2]{\stackrel{\text{(#1)}}{#2}}

\newcommand{\mmu}{{\mu}}

\newcommand{\X}{\mathbb{X}}
\newcommand{\E}{\mathbb{E}}

\newcommand{\states}{\mathcal{S}}
\newcommand{\histories}{\mathcal{H}}
\newcommand{\actions}{\mathcal{A}}
\newcommand{\opt}{^\star}

\newcommand{\tr}{^{\top}}
\newcommand{\Real}{\mathbb{R}}
\newcommand{\Natural}{\mathbb{N}}

\newcommand{\inits}{s_0}
\newcommand{\ualpha}{\ushort{\alpha}}

\newcommand{\upi}{{\ushort{\pi}}}

\DeclareMathOperator*{\argmin}{argmin}
\DeclareMathOperator*{\argmax}{argmax}

\newcommand{\quant}{{\mathfrak{q}}}

\DeclareMathOperator{\varo}{VaR}

\newcommand{\PiHR}{\tilde{\Pi}_{\mathrm{HR}}}

\newcommand{\PiHD}{\Pi_{\mathrm{HD}}}

\renewcommand{\ss}{\mid}    
\newcommand{\probs}[1]{\Delta_{#1}}  

\newcommand{\Int}{\mathcal{I}}

\newcommand{\discretized}{\mathrm{d}}
\newcommand{\initalpha}{\alpha_0}
\newcommand{\theirs}[1]{\ddot{#1}}

\renewcommand{\P}[1]{\mathbb{P}\left[ #1 \right]}
\renewcommand{\Pr}[1]{\mathbb{P}\left[ #1 \right]}

\newcommand{\var}[2]{\varo_{#1} \left[#2\right]}

\newcommand{\qud}{\ushort{q}^{\discretized}}
\newcommand{\bud}{\ushort{\mathcal{B}}^{\discretized}_{\mathrm{u}}}
\newcommand{\tqud}{\tilde{\ushort{q}}^{\discretized}}
\newcommand{\budk}{\ushort{B}^{\discretized}_{\kappa}}

\theoremstyle{plain}
\newtheorem{theorem}{Theorem}[section]
\newtheorem*{theorem*}{Theorem}

\newtheorem{proposition}[theorem]{Proposition}
\newtheorem{lemma}[theorem]{Lemma}
\newtheorem{corollary}[theorem]{Corollary}
\theoremstyle{definition}
\newtheorem{definition}[theorem]{Definition}
\newtheorem{assumption}[theorem]{Assumption}

\theoremstyle{remark}

\newtheorem{example}[theorem]{Example}

\newcommand{\esther}[1]{{\color{magenta} Esther: #1}}
\newcommand{\mm}[1]{\textcolor{red}{[#1]}}
\newcommand{\monkie}[1]{\textcolor{cyan}{#1}}
\newcommand{\mgh}[1]{\textcolor{orange}{[#1]}}

\usepackage[capitalize]{cleveref}
\Crefname{section}{Section}{Sections}
\Crefname{appendix}{Appendix}{Appendices}
\Crefname{algocf}{Algorithm}{Algorithms}
\Crefname{algorithm}{Algorithm}{Algorithms}
\Crefname{theorem}{Theorem}{Theorems}
\Crefname{definition}{Definition}{Definitions}
\Crefname{example}{Example}{Examples}
\Crefname{lemma}{Lemma}{Lemmas}
\Crefname{assumption}{Assumption}{Assumptions}
\Crefname{proposition}{Proposition}{Propositions}

\IncMargin{1em}  

\newif\ifFinalVersion
 \FinalVersiontrue
\ifFinalVersion
\renewcommand{\esther}[1]{{}}
\renewcommand{\mm}[1]{}
\renewcommand{\monkie}[1]{}
\renewcommand{\mgh}[1]{}
\fi

\newif\ifCRVersion
\CRVersiontrue
\ifCRVersion
\fi
\newcommand{\zeroOneIntervalOptions}[2]{#2}

\author[1]{Jia Lin Hau}
\author[2,3,4]{Erick Delage }
\author[4,5]{Esther Derman}
\author[6]{\\Mohammad Ghavamzadeh}
\author[1]{Marek Petrik}
\affil[1]{University of New Hampshire, Durham, NH}
\affil[2]{HEC Montréal, Canada}
\affil[3]{GERAD, Canada}
\affil[4]{MILA - Quebec AI Institute, Canada}
\affil[5]{Université de Montréal, Canada}
\affil[6]{Amazon AGI, Sunnyvale, CA}

\title{Q-learning for Quantile MDPs:\\A Decomposition, Performance, and Convergence Analysis}

\date{}

\begin{document}

\maketitle

\begin{abstract}
In Markov decision processes (MDPs), quantile risk measures such as Value-at-Risk are a standard metric for modeling RL agents' preferences for certain outcomes. This paper proposes a new Q-learning algorithm for quantile optimization in MDPs with strong convergence and performance guarantees. The algorithm leverages a new, simple dynamic program (DP) decomposition for quantile MDPs. Compared with prior work, our DP decomposition requires neither known transition probabilities nor solving complex saddle point equations and serves as a suitable foundation for other model-free RL algorithms. Our numerical results in tabular domains show that our Q-learning algorithm converges to its DP variant and outperforms earlier algorithms. 
\end{abstract}

\section{Introduction}
\label{sec: intro}
The practicality of reinforcement learning (RL) models has led to their widespread integration in autonomous decision-making \citep{kiumarsi2017optimal, kiran2021deep}. Traditional metrics focusing on expected value fail to address the uncertainties of random returns, thus acknowledging that a one-size-fits-all model is insufficient \citep{howard1972risk, tamar2012policy}. This observation is blatant in various applications such as motion control \citep{braun2011risk,hakobyan2021wasserstein,ahmadi2021risk}, autonomous systems \citep{wang2022risk,jin2019risk}, healthcare \citep{kose2016optimal,singh2020building}, or capital investment \citep{min2022risk}, where each necessitates a specific risk-sensitive objective. Risk-averse RL (RARL) endeavors to align with the decision-maker's preferences by tailoring the objective function according to their risk interests \citep{yoo2024risk}. 
Common measures in RARL include exponential utility \citep{borkar2002q,hau2023entropic}, target value \citep{wu1999minimizing,lin2003optimal}, value-at-risk (VaR) \citep{hau2023dynamic,li2022quantile,chow2018percentilerisk}, conditional VaR (CVaR) \citep{chow2014cvar,lim2022distributional,bauerle2011markov}, or mean-variance objectives \citep{tamar2012policy,luo2024alternative}, among others. 

Quantile measures such as VaR are used in data centers to ensure the reliability of cloud computing services \citep{decandia2007dynamo}, in epidemiology to understand how exposure to disease differs across continuous health outcomes distributions, e.g., BMI \citep{wei2019applications}, or more standardly in financial markets to assess counterparty risk \citep{emmer2015best, alexander2012quantile}. Although VaR does not enjoy the mathematical properties of CVaR and thus is not a coherent risk, its numerical advantages for backtesting have led the Basel Committee to retain it for credit value adjustment \citep{bis, embrechts2018quantile}. Unlike expected value, quantile risk measures account for the return distribution. For example, the expectation may favor low probability outcomes yielding a high return, whereas the $\alpha$-level-VaR guarantees the gain is greater or equal to that value with high probability $1-\alpha$. In RARL, recent years have witnessed an increasing interest in quantile measures, especially since the emergence of distributional RL \citep{bellemare2017distributional, dabney2018distributional} and its successful application in robotics \citep{majumdar2017should}. 

A major challenge in deriving practical RL solutions is that a full knowledge of the environment is often inaccessible. This can be even more problematic in RARL where risk-averse strategies must involve the full return distribution rather than just its expectation. Previous works have presented model-free methods to find optimal CVaR policies \citep{dabney2018implicit,keramati2020being,lim2022distributional}, but all assumed the existence of an optimal Markov policy. Such assumption is valid for mean and entropic risk objectives but generally not for quantile-based objectives such as CVaR or VaR \citep{hau2023dynamic,ben2007old}. Indeed, optimal-VaR policies can potentially all be history-dependent, so restricting the search to Markov policies can produce suboptimal return \citep{hau2023dynamic,li2022quantile}.

In this study, we propose a new dynamic programming (DP) formulation for quantile MDPs (see \citealt{li2022quantile}) that we refer to as VaR-MDP since it seeks a policy that is optimal with respect to the VaR of the total discounted reward. Compared with prior work, our DP decomposition requires neither known transition probabilities nor solving complex saddle point equations.
As a second contribution, we introduce VaR-Q-learning, the first model-free method that provably optimizes the return distribution's VaR. We develop a rigorous proof of convergence to an optimal policy, thus ensuring the validity of our approach. Although our VaR-Q-learning can be seen as a simple variant of the risk-sensitive Q-learning methods developed in distributional RL (DiRL) \citep{bellemare2017distributional,dabney2018distributional,dabney2018implicit}, we show that our slight modification makes a meaningful difference in the quality of the computed policy.

We first delineate our research setting and preliminary concepts in \cref{sec:sett-prel}. Then, we introduce new DP equations in \cref{sec:var-dp-equations} to solve VaR-MDPs while being simpler than existing methods. This enables us to propose a novel quantile Q-learning algorithm in \cref{sec:var-q-learning}, which optimizes the VaR-MDP from sampled trajectories, and establish its convergence property. Finally, numerical experiments presented in \cref{sec:numerical-experiment} illustrate the effectiveness of our algorithm. 

\section{Preliminaries and Formal Model}
\label{sec:sett-prel}

We first introduce our notations and overview relevant properties for quantile and VaR risk measures. We then formalize the MDP framework with VaR objective, which we name VaR-MDP in short. We prove the claims of this section in \cref{sec:preliminaries-proofs}.

\textbf{Notation.}
The augmented reals are $\bar{\Real} := \Real \cup \left\{ -\infty, \infty \right\}$ and we denote by $\Int$ the class of closed intervals in $\Real$. Given a measurable space $\mathcal{E}$, we abuse notation and denote by $\mathbb{R}^{\mathcal{E}}$ the set of all measurable functions from $\mathcal{E}$ to $\mathbb{R}$, with $\bar{\mathbb{R}}^{\mathcal{E}}$ the measurable functions to $\bar{\mathbb{R}}$, and finally with $\mathcal{I}^{\mathcal{E}}$ the functions to $\mathcal{I}$.
Given a finite set $\mathcal{Z} = \left\{ 1, \dots , Z \right\}$, the probability simplex is $\probs{\mathcal{Z}} := \{ y \in \Real_+^Z \mid  \mathbf{1}\tr  y = 1 \}$. 
For conciseness, we denote by $[n]$ the sequence of integers from $0$ to $n$. 
The set of discrete real-valued random variables with finite support is denoted by $\mathbb{X}$. Random variables are marked with a tilde, e.g., $\tilde{x} \in \mathbb{X}$.

\paragraph{Quantiles and Value-at-Risk}

The quantile of a random variable $\tilde{x} \in \mathbb{X}$ at level $\alpha \in [0,1]$ is any $\tau\in \Real$ such that $\Pr{\tilde{x} \le \tau} \ge \alpha$ and $\P{\tilde{x} \ge \tau} \ge 1-\alpha$. It may not be unique but lies in the interval $[\quant^-_{\alpha}(\tilde{x}), \quant^+_{\alpha}(\tilde{x})]$, where
\begin{align} \label{eq:quantile-definition}
\quant_{\alpha}^-(\tilde{x})
  &:=
    \min \left\{\tau \in \bar{\Real} \mid \Pr{ \tilde{x}  \le \tau } \ge \alpha\right\},\nonumber\\
  \quant_{\alpha}^+(\tilde{x})
  &:=\max \left\{\tau \in \bar{\Real} \mid \Pr{ \tilde{x} < \tau } \le \alpha\right\}.
\end{align}
The maximum in \cref{eq:quantile-definition} exists because the mapping $\tau \mapsto \Pr{\tilde{x} < \tau}$ is lower semi-continuous. So does the minimum, since $\tau \mapsto \Pr{\tilde{x} \le \tau}$ is upper semi-continuous. Also, $\quant^-_0(\tilde{x}) = -\infty$ and $\quant^+_1(\tilde{x}) = \infty$, while 
$\quant^-_\alpha(\tilde{x}) \in \Real$ for all $\alpha \in (0,1]$ and $\quant^+_\alpha(\tilde{x}) \in \Real$ for all $\alpha \in [0,1)$.

Monetary risk measures generalize the average criterion to account for uncertain outcomes. Among them, quantile-based measures like VaR are the most common \citep{Follmer2016, Shapiro2014}. Given a risk level $\alpha \in [0,1]$, the VaR function $\varo_{\alpha} \colon \mathbb{X} \to \bar{\Real}$ is defined as the largest $1-\alpha$ confidence lower bound on the value of $\tilde{x}$, i.e.,
\begin{equation}
  \label{eq:var-definition}
\var{\alpha}{\tilde{x}} \;:=\;  \quant^+_{ \alpha }(\tilde{x}).  
\end{equation}
By convention, $\var{\alpha}{0} = 0$ if $\alpha\in[0,1)$ and $\infty$ otherwise.  

\textbf{Elicitability. }
Based on the works of \cite{Gneiting2011} and \cite{Bellini2015}, a risk measure is \textit{elicitable} if it is the solution of an empirical risk minimization problem. 
In particular, a quantile can be estimated via quantile regression with the loss~\citep{Koenker1978}:
\begin{equation} \label{eq:quantile-loss}
  \ell_{\alpha}(\delta)
  := \max (\alpha \delta , -(1-\alpha) \delta). 
\end{equation}
Thus, quantile measures are elicitable as stated below.

\begin{lemma} \label{lem:var-elicitable}
  For any $\tilde{x}\in\mathbb{X}$
   and $\alpha\in [0,1]$, it holds that
  $\argmin_{y \in \Real} \E[\ell_\alpha(\tilde{x}-y)] = [\quant^-_{\alpha}(\tilde{x}), \quant^+_{\alpha}(\tilde{x})] \cap \Real.$
\end{lemma}

As explained in~\citet[Ex.~3.8]{Bellini2015}, VaR is not elicitable unless $\tilde{x}$ is continuous, i.e. $\quant^-_{\alpha}(\tilde{x})=\quant^+_{\alpha}(\tilde{x})$. Hence, methods that rely on statistical estimation of conditional VaR using \cref{lem:var-elicitable} must contend with potential underestimation.

\paragraph{MDPs with Value-at-Risk  Objective}
We formulate the decision process as an MDP  $(\mathcal{S},\mathcal{A},r,p, \gamma, T)$ that comprises a set of states $\mathcal{S} = \left\{ 1, \dots, S \right\}$, a set of actions $\mathcal{A} = \left\{ 1, \dots , A \right\}$, a reward function $r \colon \mathcal{S} \times \mathcal{A} \to [\ushort{R}, \bar{R}]$, 
and 
a transition probability $p \colon \mathcal{S} \times \mathcal{A} \to \probs{\states}$, where $p(s,a,s')$ denotes the  probability to transit from $s \in \states$ to $s'\in \states$ after taking action $a \in \actions$. 
The coefficient $\gamma \in [0,1]$ is a discount factor and $T\in\mathbb{N}$ is the decision horizon. 
The agent aims to find a policy $\pi$ that optimizes the \emph{static} VaR of the discounted sum of returns, 
\begin{equation} \label{eq:main-var-objective}
\rho(\pi):=\varo_{\initalpha}^{\pi,s_0}\left[\sum_{k=0}^{T-1} \gamma^k  r(\tilde{s}_{k},\tilde{a}_{k})\right],
\end{equation}
for some initial state $s_0 \in \states$ and reference risk level $\initalpha\in(0,1)$. In~\cref{eq:main-var-objective}, the distribution of $\tilde{s}_k$ for $k \geq 1$ is implicitly governed by the transition model $p$ while the superscript $s_0$ fixes the initial state. The policy 
$\pi$ governs the realization of actions $\tilde{a}_{k}$ at all steps $k$, which we formalize next. 

Defining a history at time $k \in [T-1]$ as $h_k := (s_0, a_0, s_1, a_1, \dots , s_{k}) \in  \histories_k := (\states  \times \actions)^k \times  \states$, its appending to $a\in \actions $ and $s'\in \states $ is denoted by $\langle h_k, a, s' \rangle \in \mathcal{H}_{k+1}$.
Given a time horizon $t\in 1{:}T$, a \emph{history-dependent policy} $\pi := (\pi_k)_{k=0}^{t-1}$ is a sequence of decision rules $\pi_k\colon \histories_k \to \actions$ from histories to actions. Focusing on the class of Markov or even stationary policies is standard for risk-neutral objectives because they are optimal \citep{puterman2014markov}. For risk-averse objectives, they are generally not, so we must optimize over history-dependent policies. We note that \cite{hau2023dynamic} established the existence of optimal deterministic policies in \cref{eq:main-var-objective}, so we can ignore stochastic policies without impairing optimality. 
Let $\PiHD^t$ be the set of all history-dependent deterministic policies over horizon $t$. 
All in all, given a quantile level $\initalpha$ and an initial state $s_0$, we aim to find 
$\max_{\pi\in\PiHD^T} \rho(\pi)$. 

Although the optimal policy for~\cref{eq:main-var-objective} is history-dependent, it can still be computed using DP and value iteration~\citep{hau2023dynamic, li2022quantile}. As we will see, the difference with standard DP lies in the fact that the optimal state-action value function must also adapt the quantile level $\alpha \in [0,1]$ at each state $s\in \states$ and $t \in [T]$. 
Let thus $q_t\opt \colon \mathcal{S} \times [0,1] \times  \actions \to \bar{\Real}$ be the optimal state-action value function for $t\in [T]$:
\begin{align}
\label{eq:quantile_q_defn}
  q_t\opt(s,\alpha,a) &:= \max_{\substack{\pi\in \PiHD^t:\\\pi_0(s)=a}} \varo^{\pi,s}_{\alpha}\left[ \sum_{k=0}^{t-1} \gamma^k  r(\tilde{s}_k, \tilde{a}_k)\right].
\end{align}
It is also convenient to define the optimal state value function $v_t\opt \colon \mathcal{S} \times [0,1] \to \bar{\Real}$ for horizon $t\in [T]$ as
\begin{equation*}
    v_t\opt(s,\alpha) := \max_{\pi\in \PiHD^t}\varo^{\pi,s}_{\alpha}\left[\sum_{k=0}^{t-1} \gamma^k \cdot r(\tilde{s}_{k},\tilde{a}_{k})\right].
\end{equation*}
Similar to risk-neutral MDPs, the state value function is related to the state-action value function through $ v_t\opt(s,\alpha) \; =\;  \max_{a\in \mathcal{A}} \, q_t\opt(s,\alpha,a), \forall t\in [T]$ 
(see \cref{apx: proof value-qvalue}).

\section{VaR Dynamic Programming} 
\label{sec:var-dp-equations}

In this section, we devise a DP method to compute an optimal policy for the static VaR objective in \eqref{eq:main-var-objective}. This section assumes that the model is known and builds on the analysis in \cite{hau2023dynamic, li2022quantile}. \cref{sec:var-q-learning} extends the approach to the model-free setting. Regardless of the methodology, the key idea of VaR-DP is to augment the state space with a risk-level input and to perform Bellman recursions on the augmented state-action value function. This augmentation should not be arbitrary, as the `risk-level state' must evolve in a specific way to yield an optimal policy. 
We report the proofs of this section in \cref{sec:proofs-crefs-dp}.

\subsection{Model-based VaR-DP}
We present the VaR-DP introduced by \cite{li2022quantile} and revised in \cite{hau2023dynamic}. 
Let $B_{\max}\colon \bar{\Real}^{\states\times [0,1]\times \actions} \to \bar{\Real}^{\states\times [0,1]\times \actions}$ be the following Bellman operator. For all $s\in \states$, $\alpha \in [0,1]$, and $a\in\actions$: 
\begin{align}
(B_{\max} q)(s,\alpha, a) &:= r(s,a) + \gamma \; \cdot \max_{o \in \mathcal{O}_{sa}(\alpha)} \min_{s' \in\states}  \max_{a'\in \actions} q( s', o_{s'},a'),
  \label{eq:q-optimal-dp} \\
\mathcal{O}_{sa}(\alpha) &:=  \left\{ o\in [0,1]^{S} \mid \sum_{s'\in \states}  o_{s'} \cdot p(s,a,s')  \le \alpha
  \right\} \nonumber.
\end{align}
Further consider the sequence $q_0(s,\alpha,a):= \var{\alpha}{0}$, and $q_{t+1}:= B_{\max} q_t$ for all $t\in[T-1]$. The constraint set above depends on the transition model $p$ so Bellman recursions are model-based. Correspondingly, a proper recursion on risk levels leads to an optimal policy in terms of VaR return. Namely, for $k\in[T-1]$, let $\alpha_k:\histories_k\to [0,1]$ such that $\alpha_0(s) = \initalpha, \quad \forall s\in\states$, while 
$\alpha_{k+1}(h_{k+1})$ satisfies both $\alpha_{k+1}(\langle h_k, a, \cdot \rangle) \in \mathcal{O}_{sa}(\alpha_k(h_k))$ and 
\begin{align*}
    q_{T-k}(s,\alpha_k(h_k),a) 
    = r(s,a) 
     + \gamma \min_{s' \in\states}  \max_{a'\in \actions} \; q_{T-k-1}(s', \alpha_{k+1}(\langle h_k, a, s' \rangle), a').
\end{align*}
Then, one can derive a greedy policy  $\pi := (\pi_k)_{k=0}^{T-1}$ from the constructed risk level mappings $(\alpha_k)_{k=0}^{T-1}$, i.e., at each step $k \in [T-1]$ as
\begin{align}
\label{eq:VaRpol} 
    \pi_k (h_k) \in \argmax_{a\in \actions} q_{T-k}(s,\alpha_k(h_k),a),  \quad \forall h_k\in \mathcal{H}_k.
\end{align}
The following theorem shows that the above Bellman operator is optimal, in the sense that the sequence $q:= (q_t)_{t=0}^{T-1}$ resulting from these recursions yields \cref{eq:quantile_q_defn}. Accordingly, a  policy $\pi$ that is greedy for that sequence $q$ of state-action value functions at the constructed risk-level sequence is optimal for the VaR objective \eqref{eq:main-var-objective}.

\begin{theorem} 
\label{thm:q-optimal:old}
Let a sequence $q=(q_t)_{t=0}^{T}$ be such that $q_0(s,\alpha,a)= \var{\alpha}{0}$ and $q_{t+1}:= B_{\max} q_t$ for $t\in[T-1]$. Then, $q_t=q\opt_t$ for all $t\in[T]$, where $q\opt_t$ is defined in~\eqref{eq:quantile_q_defn}. Moreover, if a policy $\pi = (\pi_k)_{k=0}^{T-1}$ is greedy w.r.t.~$q$ as in~\eqref{eq:VaRpol}, then it maximizes the VaR objective~\eqref{eq:main-var-objective}.
\end{theorem}

\cref{thm:q-optimal:old} enables solving MDPs with static VaR objectives through DP equations. An important limitation of this representation is that it requires access to the underlying transition model. This implicitly appears in the constraint set $\mathcal{O}_{sa}(\alpha)$ from \cref{eq:q-optimal-dp}, which is required to find an optimal greedy policy according to the risk level mappings $(\alpha_k)_{k=0}^{T-1}$. Additionally, each Bellman update requires solving a constrained optimization problem, which slows down the learning process. 

This paper proposes a model-free learning algorithm with theoretical convergence guarantees. To assess the optimality of our value function updates, we reformulate VaR-DP equations in terms of Bellman operators in \cref{sec: nested var dp}. Our VaR-DP formulation can be seamlessly used for known or unknown transition probability models, as it allows for statistical estimation from sampled trajectories.

\subsection{Nested VaR-DP}
\label{sec: nested var dp}
 
Let $B_{\mathrm{u}}\colon \bar{\Real}^{\mathcal{S} \times  [0,1] \times \mathcal{A}} \to \bar{\Real}^{\mathcal{S} \times  [0,1] \times \mathcal{A}}$ be the following VaR Bellman operator:
\begin{equation*} 
(B_{\mathrm{u}} q)(s,\alpha, a) := \varo_{\alpha}^{a,s}[r(s,a) + \gamma \cdot  \max_{a'\in \mathcal{A}} q(\tilde{s}_1,\tilde{u},a')],
\end{equation*}
where the VaR is based on the joint distribution of $(\tilde{s}_1,\tilde{u})$ with $\tilde{s}_1 \sim p(s,a,\cdot)$ and an independent $\tilde{u}$ uniformly distributed on $[0,1]$ ($\tilde{u}\sim U([0,1]$).
Correspondingly, consider the following DP equations:
\begin{align}
q_0^{\mathrm{u}}(s,\alpha,a) &:= \var{\alpha}{0}, \quad \forall  s\in \mathcal{S}, \alpha \in [0,1], a\in \mathcal{A}, \nonumber\\
\label{eq:q-optimal-bellman}
q_{t+1}^{\mathrm{u}}  &:= B_{\mathrm{u}} q^{\mathrm{u}}_t , \quad \forall t\in [T-1].
\end{align}
We must also adapt the risk level mappings to this new Bellman recursion. Let $\hat{\alpha}_0^{\mathrm{u}}(s_0):=\initalpha$  and for $k\in[T-1]$, 
\begin{equation}
\hat{\alpha}_{k+1}^{\mathrm{u}}(h_{k+1}) :=  \min \left\{ o \in [0,1]
  \mid \max_{a\in\actions}q_{T-k-1}^{ \mathrm{u}}(s_{k+1},o,a) 
  \geq \frac{q_{T-k}^{\mathrm{u}}(s_{k},\hat{\alpha}_{k}(h_{k}),a_{k})-r(s_{k},a_{k})}{\gamma} \right\}. 
  \label{eq:newPolDef}
\end{equation}
The greediness criterion remains unchanged: at each step $k \in [T-1]$, construct
\begin{align}
\label{eq:greedy_var}
    \pi_k ^{\mathrm{u}}(h_k) \in \argmax_{a\in \actions} q_{T-k}^{\mathrm{u}}(s,\hat{\alpha}_k^{\mathrm{u}}(h_k),a),  \quad h_k\in \mathcal{H}_k.
\end{align}
The following theorem states that this alternative method is still valid for finding an optimal solution.

\begin{theorem}
\label{thm:q-optimal:new}
Let a sequence $q^{\mathrm{u}}=(q^{\mathrm{u}}_t)_{t=0}^{T}$ be such that $q^{\mathrm{u}}_0(s,\alpha,a):= \var{\alpha}{0}$ and $q^{\mathrm{u}}_{t+1}:= B_{\mathrm{u}} q^{\mathrm{u}}_t$ for $t\in[T-1]$. Then, $q^{\mathrm{u}}_t=q\opt_t$ for all $t\in[T]$, where $q\opt_t$ is defined in~\eqref{eq:quantile_q_defn}. Moreover, if a policy $\pi^{\mathrm{u}} = (\pi^{\mathrm{u}}_k)_{k=0}^{T-1}$ is greedy w.r.t.~$q^{\mathrm{u}}$ as in~\eqref{eq:greedy_var}, then it maximizes the VaR objective in~\eqref{eq:main-var-objective}.
\end{theorem}

In contrast to \cite{hau2023dynamic,li2022quantile}, \cref{thm:q-optimal:new} does not require knowing the transition model. More importantly, it reduces VaR-MDPs to a nested VaR conditional mapping: 
\begin{align*}
v_T\opt&(s_0,\alpha_0) = \max_{a_0\in\actions} \varo_{\alpha_0}^{a_0,s_0} \bigg[r(s_0,a_0) + \gamma \;\cdot \\
&\quad~\max_{a_1\in\actions} \varo_{\tilde{u}_1}\big[r(\tilde{s}_1,a_1) + \dots +\gamma\;\cdot \\
&\quad\max_{a_{T-2}\in\actions}\varo_{\tilde{u}_{T-2}}[r(\tilde{s}_{T-2},a_{T-2})+\gamma \;\cdot\\
&\quad\max_{a_{T-1}\in\actions} r(\tilde{s}_{T-1},a_{T-1})|\tilde{s}_{1:T-2},\tilde{u}_{1:T-2}]\dots|\tilde{s}_1,\tilde{u}_1\big]\bigg].    
\end{align*}
The value is still over an augmented state-space, but this time, the risk tolerance is independently and uniformly drawn from $[0,1]$ at each step. This is particularly suitable for sample-based RL and more amenable to deep settings such as~\cite{dabney2018distributional,dabney2018implicit,lim2022distributional}. Yet, two issues remain 
to produce a Q-learning procedure. First, the state space is infinite because the risk level $\alpha$ is continuous. Second, the elicitation procedure described in \cref{lem:var-elicitable} underestimates the VaR when the risk-to-go variable is discrete, so we cannot directly employ quantile regression. To address these issues, in the next section, we propose an approximation scheme that replaces the Bellman operator of \cref{eq:q-optimal-bellman} with either lower or upper bounds of the appropriate quantiles.

\section{Q-learning Algorithm and Analysis} 
\label{sec:var-q-learning}

This section builds on the DP of \cref{sec:var-dp-equations} to derive a new Q-learning algorithm for the static VaR objective.
\Cref{sec:discretize-q-values} introduces an approximate Bellman operator that can be used to compute \cref{eq:q-optimal-bellman} in a tractable way. Then, \cref{sec:Q-learning-algorithm} proposes the VaR-Q-learning algorithm and shows its convergence guarantees. 

\subsection{Discretized Quantile Q-functions} \label{sec:discretize-q-values}

The challenge in computing the value function in \cref{eq:q-optimal-bellman} stems from the fact that it is defined over a continuous $\alpha \in [0,1]$. To make the computation tractable, we propose to approximate the risk level $\alpha$ with properly defined functions that yield lower and upper bounds on the quantile value function. 

\begin{definition}\label{ass:falpha}
  Let $\bar{f},\ushort{f}\colon [0,1] \to [0,1]$ be two non-decreasing right-continuous functions  such that $\bar{f}(\alpha)>\alpha\geq\ushort{f}(\alpha)$ for all $\alpha\in[0,1)$, while $\bar{f}(1) = 1\geq\ushort{f}(1)$.
\end{definition}

The following result shows how the functions $\ushort{f}, \bar{f}$ can yield upper and lower bounds on VaR, and thus, on the state-action value function $q$. This development exploits that $\alpha \mapsto q(s,\alpha, a)$ is non-decreasing for each $s\in \states$ and $a\in \actions$.
\begin{lemma} \label{lem:var-bounds}
For any $\tilde{x}\in\mathbb{X}$ and $\alpha\in(0,1)$, we have
\begin{align*}
\varo_\alpha(\tilde{x}) &\ge \quant^+_{\ushort{f}(\alpha)}(\tilde{x}) = \max \argmin_{q\in \Real} \E[\ell_{\ushort{f}(\alpha)}(\tilde{x}-q)],  \\
 \varo_\alpha(\tilde{x}) &\le \quant^-_{\bar{f}(\alpha)}(\tilde{x}) = \min \argmin_{q\in \Real} \E[\ell_{\bar{f}(\alpha)}(\tilde{x}-q)],
\end{align*}
where $\ell_\alpha$ is defined as in~\cref{eq:quantile-loss} and $\ushort{f}, \bar{f}$ as in  \cref{ass:falpha}. 
\end{lemma}

We now derive Bellman operators that facilitate the construction of the Q-learning algorithm. First, the set-valued operator $\mathcal{B}_{\mathrm{u}}\colon {\Real}^{\states \times  [0,1] \times \actions} \to \mathcal{I}^{\states \times  [0,1] \times \actions}$ generalizes $B_{\mathrm{u}}$ as an empirical risk minimizer, i.e.,
\begin{equation}
\label{eq:quantile-q-bellman-operator}
  (\mathcal{B}_{\mathrm{u}} q)(s,\alpha, a)  :=    
  \argmin_{x\in \Real} \E^{a,s}\left[\ell_{\alpha}\left(r(s,a)+\gamma\cdot \max_{a'\in \actions} q(\tilde{s}_1,\tilde{u},a')-x\right)\right],
\end{equation}
where $\tilde{u} \sim U([0,1])$. Second, we define the upper and lower bounding Bellman operators $\mathcal{B}^{\bar{f}}_{\mathrm{u}}$ and $\mathcal{B}^{\ushort{f}}_{\mathrm{u}}$ for each $b = (s,\alpha,a)$, $s\in \states$, $\alpha \in [0,1]$, and $a\in \actions$ as
\begin{align*}
 (\mathcal{B}^{\bar{f}}_{\mathrm{u}} q)(b)
  &:= \begin{cases}
    (\mathcal{B}_{\mathrm{u}} q)(s,\bar{f}(\alpha), a)
    & \text{ if }  \bar{f}(\alpha) < 1, \\
    \displaystyle \bar{R}+\max_{s'\in\states,a'\in\actions}q(s',1,a') 
    & \text{ if }\bar{f}(\alpha)=1,
\end{cases} \\                                   
 (\mathcal{B}^{\ushort{f}}_{\mathrm{u}} q)(b)
  &:= \begin{cases}
    (\mathcal{B}_{\mathrm{u}} q)(s,\ushort{f}(\alpha), a)
    & \text{ if }  \ushort{f}(\alpha) > 0, \\
    \displaystyle  \ushort{R}+\min_{s'\in\states,a'\in\actions}q(s',0,a')
    & \text{ if }\ushort{f}(\alpha)=0.
\end{cases}       
\end{align*}
The following theorem shows how to use these operators to bound the value functions and the performance of the computed policy.
\begin{theorem} \label{thm:bound_v2}
Suppose that $\bar{q}^{\mathrm{u}}_0=\ushort{q}^{\mathrm{u}}_0=0$, and that $\bar{q}_{t+1}^{\mathrm{u}}$ and $\ushort{q}_{t+1}^{\mathrm{u}}$ are right-continuous non-decreasing functions in $\alpha$ 
satisfying 
$\bar{q}_{t+1}^{\mathrm{u}} \in (\mathcal{B}^{\bar{f}}_{\mathrm{u}} \bar{q}^{\mathrm{u}}_t)$ and $ \ushort{q}^{\mathrm{u}}_{t+1} \in (\mathcal{B}^\ushort{f}_{\mathrm{u}} \ushort{q}^{\mathrm{u}}_t)$ for all $t\in [T-1]$.\footnote{Right-continuity and non-decreasingness can  be obtained by ensuring $\bar{q}_t(s, \alpha ,a)=\bar{q}_t(s, \alpha' ,a)$ if $\bar{f}(\alpha)=\bar{f}(\alpha')$.} 
Then,  
$\bar{q}^{\mathrm{u}}_t
  \; \ge\;
  q\opt_t
  \; \ge\;  \ushort{q}^{\mathrm{u}}_t$
for all $t \in [T]$, where $q\opt_t$ is defined in~\cref{eq:quantile_q_defn}.
Moreover, if a policy $\ushort{\pi}:= (\ushort{\pi}_k)_{k=0}^{T-1}$ is greedy for $\ushort{q}^{\mathrm{u}}$, in the sense that
\begin{equation*}
        \upi_k(h_k) \in \argmax_{a\in \actions} \ushort{q}^{\mathrm{u}}_{T-k}(s,\ushort{\alpha}^{\mathrm{u}}_k(h_k),a),
\end{equation*} 
with $\ualpha^{\mathrm{u}}$ as in \cref{eq:newPolDef}, then it satisfies 
$\max_{a\in\actions}\ushort{q}^{\mathrm{u}}_T(\inits,\initalpha,a)
  \; \leq\;
  \rho(\upi).$  
\end{theorem}

To simplify the exposition, we focus on the simple approximation scheme that discretizes the risk-level $\alpha$ using a uniform grid as below. 

\begin{example}[$J$-uniform discretization] \label{exm:discrete-approx}
Define $\ushort{f}, \bar{f} \colon [0,1] \to [0,1]$ as
\begin{align*}
  \ushort{f}(\alpha)
  &:= \max \left\{ \nicefrac{j}{J} \mid \nicefrac{j}{J}  \le \alpha, j\in [J-1]   \right\}, \\
  \bar{f}(\alpha)
  &:= \max \left\{ \nicefrac{j+1}{J} \mid \nicefrac{j}{J} \le \alpha, j \in  [J-1] \right\},
\end{align*}
for $J \ge 2$. These functions satisfy the conditions of \cref{ass:falpha}.
\end{example}

Under this uniform discretization scheme, $\mathcal{B}_{\mathrm{u}}$ becomes  $\mathcal{B}_{\mathrm{u}}^\discretized\colon  {\Real}^{\states \times  [J-1] \times \actions} \to \mathcal{I}^{\states \times  [J-1] \times \actions}$, defined for $s\in \states$, $j\in [J-1]$, and $a\in \actions$ as
\begin{gather*}
  (\mathcal{B}^\discretized_{\mathrm{u}} q)(s,j,a) \; :=\;  \argmin_{x\in \Real} 
\frac{1}{J} \sum_{j'=0}^{J-1} \E^{a,s}\left[\ell_{\frac{j}{J}}\left(r(s,a)+\gamma \max_{a'\in \actions}  q(\tilde{s}_1,j',a')-x\right)\right].    
\end{gather*} 
Similarly, the Bellman operator $\mathcal{B}^{\ushort{f}}_{\mathrm{u}}$ becomes
$\bud\colon {\Real}^{\states \times  [J-1] \times \actions} \to \mathcal{I}^{\states \times  [J-1] \times \actions}$, which is defined for $b=(s, j, a)$, $s\in \states$, $j\in [J-1]$, and $a\in \actions$ as
\begin{gather} \label{eq:discreteBellmanEq}
(\bud q) (b)
:= \begin{cases}
   (\mathcal{B}^\discretized_{\mathrm{u}} q)(s,j,a)
    & \text{if }  j\geq 1, \\
     \ushort{R}+ \displaystyle\min_{s\in\states,a\in\actions}q(s,0,a)
    & \text{if }j=0.
\end{cases}      
\end{gather}
The following proposition states the correctness of the discretized Bellman operators.
\begin{proposition} \label{lem:unif_discretize_bellman}
Given the $J$-uniform discretization of \cref{exm:discrete-approx}, let $\qud_0:=0$ and $\qud_{t+1} \in \bud \qud_{t}, t \in [T-1]$. Then, the sequence $\ushort{q}:= (\ushort{q}_t)_{t=0}^{T}$ defined for $t\in[T]$ as
\begin{align*} 
  \ushort{q}_t(s, \alpha ,a) := \qud_t(s, J \cdot \ushort{f}(\alpha) ,a),
  \qquad \forall s\in\states,\alpha \in [0,1],a\in\actions,
\end{align*}
provides a lower-bound for the optimal value function $q\opt$ defined in \cref{eq:quantile_q_defn}. Moreover, it can be used to build a policy $\ushort{\pi}:= (\ushort{\pi}_t)_{t=0}^{T}$ that achieves this lower-bound.
\end{proposition}
\cref{alg:quantile_policy_exe} presents a procedure that constructs the policy $\ushort{\pi}$ described in  \cref{lem:unif_discretize_bellman}.

\begin{algorithm}
\SetAlgoLined 
 \KwIn{$s_0 \in \states, \alpha_0 \in (0,1), T, J \in \Natural$, $\ushort{q}^\discretized : [T] \times \states \times [J-1] \times \actions \to \bar{\Real}$}
$(s, j) \gets \left(s_0,\lfloor J \cdot \alpha_0  \rfloor \right)$\\
\For{$t = T,\dots,1$}{
    $\displaystyle a\opt \gets \argmax_{a\in\actions}\ushort{q}_t^\discretized(s,j,a)$\\
    Execute $a\opt$ and observe $r$ and $s'$\\
    $\tau \gets  \gamma^{-1} (\ushort{q}_t^\discretized(s,j,a\opt)-r)$\\
    $\displaystyle\mathcal{J} \gets \left\{j' \in [J-1]\mid\max_{a'\in\actions}\ushort{q}_{t-1}^\discretized(s',j',a')\geq \tau\right\}$\\
    $j \gets J-1$ \tcp*{arbitrary initialization}
    \uIf {$\displaystyle\mathcal{J}$ \text{is not empty}}{
        $j \gets \min\mathcal{J}$
    }
    $s\gets s'$ \\
}    
 \caption{Static VaR Policy Execution}
  \label{alg:quantile_policy_exe}
\end{algorithm}

\subsection{VaR-Q-learning Algorithm} \label{sec:Q-learning-algorithm}

We now use the Bellman operators from  \cref{sec:discretize-q-values} to develop and analyze a new Q-learning algorithm that solves the VaR objective in a model-free way.

An important obstacle in developing the Q-learning algorithm is that the Bellman operators in \cref{sec:discretize-q-values} are set-valued and lack a unique solution. The operators are set-valued because the quantile is not unique. As a result, even a well-designed iterative algorithm may oscillate among multiple possible solutions. It is common in RARL to replace the quantile loss function with Huber's loss to guarantee differentiability~\citep{dabney2018implicit}; however, as we show in \cref{sec:why-not-huber}, Huber's loss is insufficient to guarantee the uniqueness of value function. Instead, we replace the loss $\ell_{\alpha}$ of \cref{eq:quantile-loss} with the \emph{soft-quantile loss} $\ell_{\alpha}^{\kappa}\colon \Real \to \Real$ defined for $\kappa \in (0,1]$ and $\alpha \in (0,1)$ as 
\begin{equation}\label{eq:delageLoss1}
  \ell^{\kappa}_\alpha(\delta) :=
\begin{cases}
        \frac{(1-\alpha )\kappa}{2} \left((\delta+\kappa)^2 - \frac{2\delta}{\kappa} - 1  \right)& \text{if } \delta  < -\kappa, \\
        (1 - \alpha)   \left( \frac{\delta^2}{2\kappa}\right)& \text{if} ~\delta \in [- \kappa, 0), \\
        \alpha  \left( \frac{\delta^2}{2\kappa}\right)& \text{if} ~\delta \in [0, \kappa), \\
        \frac{\alpha \kappa}{2} \left( (\delta - \kappa)^2 + \frac{2\delta}{\kappa} - 1 \right)& \text{if } \delta \geq \kappa.
    \end{cases} 
\end{equation}
We also need the derivative $\partial \ell_{\alpha}^{\kappa}$ of $\ell_{\alpha}^{\kappa }$, which is  
\begin{equation}\label{eq:delageLoss2}
\partial \ell^{\kappa}_\alpha(\delta) :=
\begin{cases}
  (1-\alpha) \left(\kappa \delta + \kappa^2 - 1  \right) & \text{if } \delta < -\kappa, \\
   \frac{1 - \alpha}{\kappa} \delta & \text{if} ~ \delta \in [- \kappa, 0),  \\
    \frac{\alpha}{\kappa} \delta  & \text{if} ~\delta \in [0, \kappa), \\
  \alpha \left(\kappa\delta -\kappa^2 + 1  \right) & \text{if } \delta \ge \kappa.
\end{cases} 
\end{equation}

As the following lemma states, the function ${\ell_\alpha^\kappa}$ is strongly convex and has a Lipschitz-continuous gradient. These properties are instrumental in showing the value function's uniqueness and analyzing our Q-learning algorithm. 
\begin{lemma} \label{lem:elicitable-objective-lips-strong}
The function $m \mapsto \mathbb{E}[\ell_\alpha^\kappa(\tilde{x}-m)]$ with $\tilde{x}$ discrete is \emph{$\mmu$-strongly convex} and has an \emph{$L$-Lipschitz continuous derivative} for each $\alpha\in (0,1),\kappa \in (0,1]$ with 
\[
\mmu = \min \left\{ \alpha , 1-\alpha  \right\} \kappa,
\qquad
L = \max \left\{ \alpha , 1-\alpha  \right\} \kappa^{-1}.
\]
\end{lemma}

\begin{algorithm}
\SetAlgoLined 
\KwIn{Step sizes $\beta_i$, stream of sampled transitions $(t_i, s_i, j_i, a_i,s_i')$, for all $i\in \Natural$}
\init{$\qud_0(b) \gets t\ushort{R},~\forall b \in  [T] \times  \states \times [J-1] \times \actions$}
\For{$i=0, 1, 2, \dots $}{
$b_i \gets (t_i, s_i, j_i, a_i)$\\
\uIf{$j_i > 0$ \and $t_i > 0$}{
$\qud_{i+1}(b_i) \gets \qud_i(b_i) 
+\frac{\beta_i}{J}\sum_{j' =0}^{J-1} \partial \ell_{\frac{j_i}{J}}^{\kappa}\big(r(s_i,a_i)$ $+\gamma  \cdot \displaystyle\max_{a'\in \mathcal{A}}\qud_i(t_i-1, s_i',j',a') - \qud_i(b_i)\big)$\\
}
\Else{
\tcp{Effectively do nothing}
$ \qud_{i+1}(b_i) \gets \qud_i(b_i) 
+ \beta_i ( t_i\ushort{R} - \qud_i(b_i) ) $ \label{line:empty-step}
}
}
\caption{VaR-Q-learning Algorithm}
\label{alg:q-learning}
\end{algorithm}

Having introduced the soft-quantile function in \eqref{eq:delageLoss1}, we are now ready to adapt the Bellman operators to ensure that the Q-learning algorithm converges to a unique solution. In particular, we replace $\mathcal{B}_{\mathrm{u}}^{\mathrm{d}}$ with $B_\kappa^\discretized$  which is defined for $b = (t, s, j, a)$, $t \in 1{:}T$, $s\in \mathcal{S}$, $j\in 1{:}J-1$, and $a\in \mathcal{A}$ 
as
\begin{gather*}
(B_\kappa^\discretized q)(b) \; :=\;   \argmin_{x\in \Real} 
\E^{a,s}\left[\ell^\kappa_{\frac{j}{J}}\left( r(s,a)+\gamma \max_{a'\in \actions} q(t-1,\tilde{s}_1,\tilde{j}',a') - x\right)\right],  
\end{gather*}
and we replace the lower-bound operator $\bud$ by $\budk$:
\begin{equation*} 
(\budk q)(b)
:= \begin{cases}
\ushort{R} \cdot t   & \mbox{if} ~j=0 \vee t = 0,\\
(B_\kappa^\discretized q)(b)  & \mbox{otherwise} .
\end{cases} \end{equation*} 

Note that the operators $B_\kappa^\discretized$ and $\budk$ are not calligraphic because their objective functions possess unique minimizers. In addition, $B_\kappa^\discretized$ and $\budk$ are applied to value functions defined across all time steps simultaneously. This representation is convenient because we focus on the finite-horizon objective and must separate the time step from the Q-learning iteration. 
That is, $\qud_i(t,s,j,a)$ represents the value function in the $i$-th iteration evaluated at  time $t$, state $s$, risk level $j$, and action $a$.

\begin{figure*}
  \centering
  \begin{minipage}{0.4\textwidth}  
    \centering
    \includegraphics[width=\textwidth]{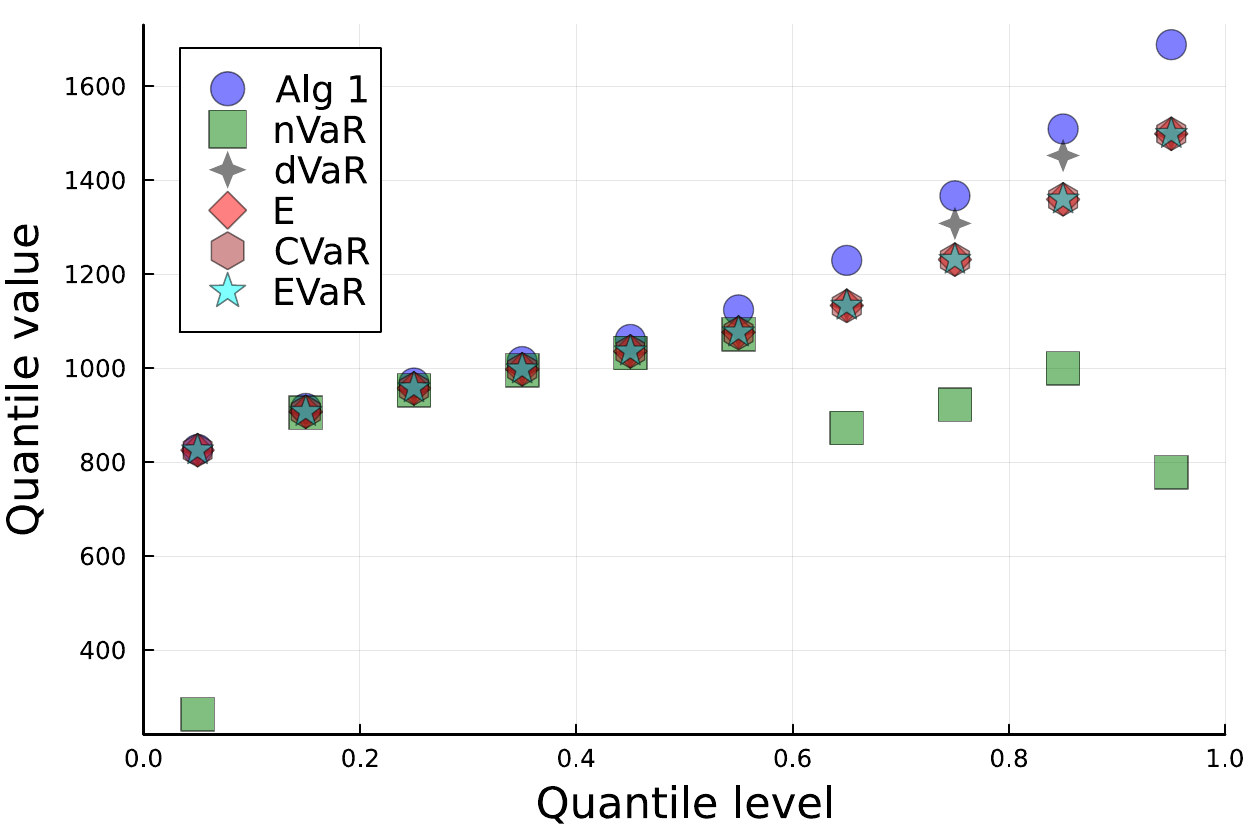} 
    \caption{Policy performance INV2}
    \label{fig:inv2_comparison}
  \end{minipage}
\end{figure*}

\begin{figure*}
  \begin{minipage}{\textwidth}
    \centering
    \begin{tabular} {|l|r|r|r|r|r|r|r|r|r|} 
\cline{2-8}
\multicolumn{1}{c|}{} & CW & INV1 & INV2 & MR & POP & RS & GR\\
\hline
$\bar{q}^\discretized$ & -9.11 & 237.19 & 970.08 & -2.79 & -14348.60 & 50.0 & 4.78\\
\cref{alg:quantile_policy_exe} & \textbf{-9.11} & \textbf{237.05} & \textbf{968.31} & \textbf{-2.84} & \textbf{-14348.60} & \textbf{50.0} & \textbf{4.78}\\
$\ushort{q}^\discretized$ & -9.11 & 236.88 & 967.60 & -2.85 & -14348.60 & 50.0 & 4.78\\
\hline
nVaR & -86.10 & 202.49 & 953.00 & -20.00 & \textbf{-14348.60} & \textbf{50.0} & 0.00\\
dVaR & -9.20 & 234.75 & 953.00 & -18.21 & \textbf{-14348.60} & \textbf{50.0} & 0.00\\
E & -9.72 & 234.60 & 957.45 & -2.95 & -15077.83 & 32.8 & 3.14\\
CVaR & -9.72 & 232.31 & 956.57 & -4.12 & \textbf{-14348.60} & \textbf{50.0} & 3.14\\
EVaR & -9.72 & 234.98 & 957.45 & -2.95 & \textbf{-14348.60} & \textbf{50.0} & 2.82\\
\hline
\end{tabular} 
\captionof{table}{$25\%$-quantile performance} \label{tab:quantile_performance_table}
  \end{minipage}
\end{figure*}

Equipped with the above definitions, we now introduce the \emph{VaR-Q-learning} algorithm in \cref{alg:q-learning}. The algorithm seeks to identify the fixed-point $\qud=\budk \qud$, which is unique, as we show in \cref{sec:appProofSection5}. The algorithm adapts the standard Q-learning approach to the risk-averse setting. The standard Q-learning algorithm can be seen as a stochastic gradient descent on the quadratic loss function. \Cref{alg:q-learning} also follows a sequence of stochastic gradient steps, but it replaces the quadratic loss function with the soft-quantile loss function $\ell_{\kappa}^{\alpha}$.

\Cref{alg:q-learning} takes a stream of samples as input, and thus, implies that it is an offline algorithm. However, the algorithm and its analysis also apply to the online setting in which the sample $(t_i, s_i, j_i, a_i, s_i')$ and step size $\beta_i$ can depend on the values $\qud_0, \dots, \qud_i$ and are generated during the execution of the algorithm.

In an actual implementation, \cref{line:empty-step} in \cref{alg:q-learning} may be omitted since it does not modify the value function. We include this step to simplify our convergence analysis, which considers each update as a stochastic gradient step towards a contractive minimizer.

We require the following standard assumption to prove the convergence of \cref{alg:q-learning}.
\begin{assumption}\label{ass:assTransitions}
The input to \cref{alg:q-learning} satisfies $\forall i\in \mathbb{N}$:
\begin{align*} 
  \P{\tilde{s}_i'=s' \mid \mathcal{G}_{i-1}, \tilde{t}_i, \tilde{s}_i, \tilde{j}_i, \tilde{a}_i, \tilde{\beta}_i} 
  =
  p(\tilde{s}_i, \tilde{a}_i, s'), \quad\forall s'\in \mathcal{S}, 
\end{align*}
almost surely, where $\mathcal{G}_{i-1} := (\tilde{\beta}_l,(\tilde{t}_l, \tilde{s}_l,\tilde{j}_l,\tilde{a}_l,\tilde{s}_l'))_{l=0}^{i-1}$.
\end{assumption}

The following theorem shows that \cref{alg:q-learning} enjoys convergence guarantees that are comparable to those in standard Q-learning.
\begin{theorem}\label{thm:convergenceQlearning}
Let $\kappa \in  (0,1]$. Assume that the sequences $\{\tilde{\beta}_i\}_{i=0}^\infty$ and $\{(\tilde{t}_i,\tilde{s}_i,\tilde{j}_i,\tilde{a}_i,\tilde{s}_i')\}_{i=0}^\infty$ used in \cref{alg:q-learning} satisfy \cref{ass:assTransitions} and the step-size conditions
\begin{align*}
  \sum_{i\in\tilde{\mathfrak{I}}(t,s,j,a)}^\infty \tilde{\beta}_i =\infty,
  \quad
  \sum_{i\in\tilde{\mathfrak{I}}(t,s,j,a)} \tilde{\beta}_i^2 < \infty,   \quad \mbox{a.s.} \; ,
\end{align*} 
where $\tilde{\mathfrak{I}}(t,s,j,a):=\{i \in\mathbb{N} \mid  (\tilde{t}_i,\tilde{s}_i,\tilde{j}_i,\tilde{a}_i)=(t,s,j,a)\}$. Then, the sequence  $(\tqud_i)_{i=0}^{\infty}$ produced by \cref{alg:q-learning} converges almost surely to $\qud_\infty$ such that $\qud_{\infty} = \budk \qud_{\infty}$. 
\end{theorem}

The proof of \cref{thm:convergenceQlearning} follows an approach similar to that in the proofs of standard Q-learning~\citep{Bertsekas1996} with two main differences. First, the algorithm converges even when $\gamma = 1$ and $\budk$ is not an $L_{\infty}$ contraction. Instead, we show that $\budk$ is a contraction w.r.t.~a particular weighted norm. Second, the use of a non-quadratic function $\ell^{\kappa}_{\alpha}$ requires a more careful choice of the step-sizes than the standard analysis. Moreover, our analysis of the non-quadratic function extends the Q-learning analysis for risk-sensitive RL with nested risk measures in \cite{shen2014risk}. 

\section{Numerical Experiments} 
\label{sec:numerical-experiment}

\begin{figure*}
  \centering
  \begin{minipage}{0.95\textwidth} 
    \centering
    \includegraphics[width=\textwidth]{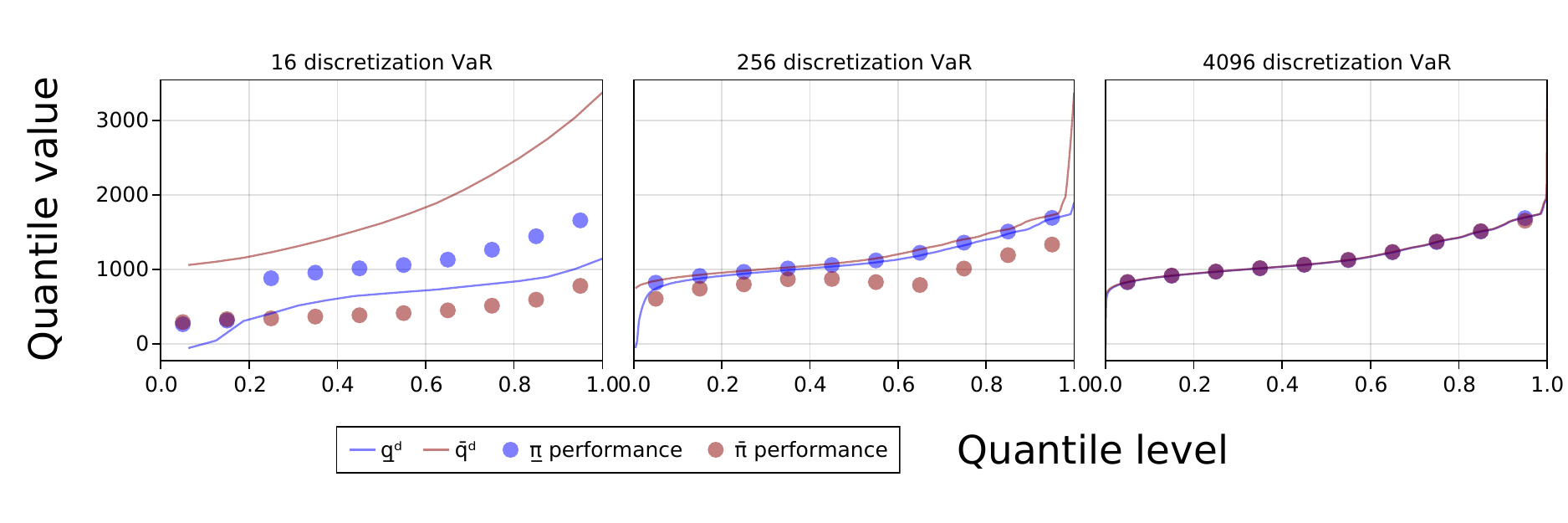} 
    \caption{Different discretization VaR MDP performance}
    \label{fig:inv2_multiple_discretization}
  \end{minipage}
\end{figure*}

\begin{figure*}
  \centering
  \begin{minipage}{0.4\textwidth}  
    \includegraphics[width=\textwidth]{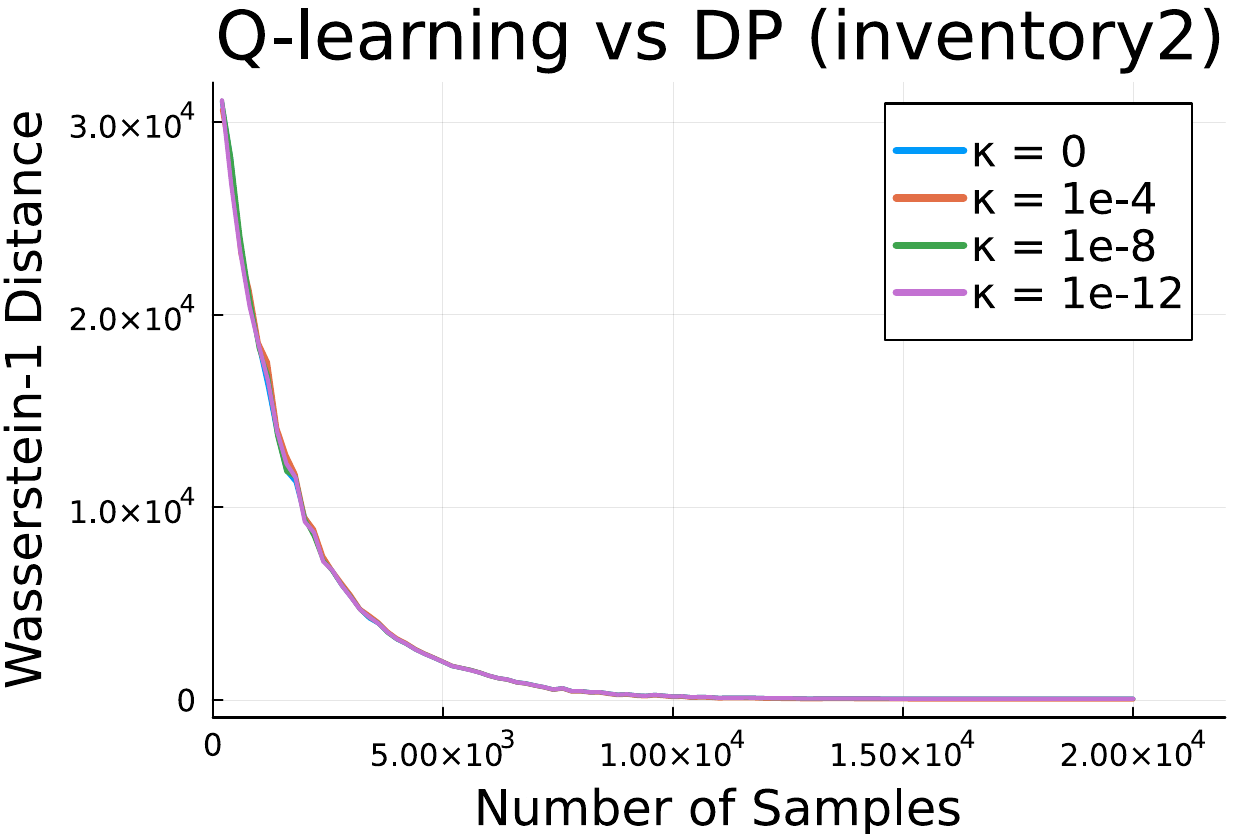} 
    \caption{Q-values of \cref{alg:q-learning} for all small $\kappa$'s converge to the DP value function.}
    \label{fig:inv2_wasserstein}
  \end{minipage}
\end{figure*}

In this section, we empirically test our theoretical results and algorithms on $7$ tabular domains: machine replacement (MR) \citep{Delage2014}, gamblers ruin (GR) \citep{bauerle2011markov,li2022quantile}, two inventory management problems (INV1 and INV2) \citep{Ho2021}, riverswim (RS) \citep{Strehl2008model}, population (POP) \citep{hau2023entropic}, and cliffwalk (CW) \citep{sutton2018reinforcement}. We set the horizon to $T = 100$ with $\gamma=0.9$ to evaluate the risk of the random discounted return. More details on each experiment can be found in \cref{sec:more_experiments} and our code can be found on at {\small\url{https://github.com/MonkieDein/DRA-Q-LA}}.

\textbf{Policy execution. }  
We first validate the discretization scheme presented in \cref{alg:quantile_policy_exe} for model-free policy execution. To this aim, we compare the performance of \cref{alg:quantile_policy_exe} with other risk-averse algorithms. As a standard baseline, we include the risk-neutral objective (E). Other baselines are nested VaR (nVaR) \citep{ruszczynski2010risk}, conditional VaR (CVaR) \citep{chow2015risk}, distributional VaR (dVaR) with risk sensitive objective \citep{dabney2018distributional, dabney2018implicit}, and entropic VaR (EVaR) \citep{hau2023entropic}. We take a quantile discretization level of $J=4096$ to train \cref{alg:quantile_policy_exe}, CVaR, and dVaR. Detail of all algorithms can be found in 
\cref{sec:more_algorithm_details}.  

\cref{tab:quantile_performance_table} shows the $25\%$-quantile value obtained after training, where each entry is the performance obtained from 10,000 episodes generated from the final policy.  
As we can see, our algorithm consistently outperforms all other algorithms across all tested domains. We also test our algorithm on a range of quantile levels $\initalpha \in \{0.05,0.15,\dots,0.85,0.95\}$.  
\cref{fig:inv2_comparison} shows the quantile value obtained on INV2 after training each baseline. Our method shows an insensitive behavior to risk levels. All other domains exhibit a similar trend across quantile levels (see \cref{sec:more_experiments}), thus illustrating the robustness of \cref{alg:quantile_policy_exe} to different environments and risk levels.

We perform an ablative study to understand how the discretization and selection of $\ushort{q}^\discretized$ in \cref{alg:quantile_policy_exe} contribute to the solution quality. We compare the performance of $\ushort{\pi}$ with that of $\bar{\pi}$ (defined analogously to \cref{eq:greedy_var}) by confronting them to the bounds $\ushort{q}^\discretized$ and $\bar{q}^\discretized$ from \cref{exm:discrete-approx}. We take $J \in \{16 , 256 , 4096\}$. \cref{fig:inv2_multiple_discretization} demonstrates that the performance of $\ushort{\pi}$ lies within $[\ushort{q}^\discretized,\bar{q}^\discretized]$, whereas $\bar{\pi}$ sometimes performs worse than $\ushort{q}$. Furthermore, as the discretization level increases, the bounding gap $\bar{q}^\discretized-\ushort{q}^\discretized$ shrinks, suggesting that  $\ushort{\pi}$ converges to $\pi \opt$.

\textbf{Quantile Q-learning. }  
We now check the convergence and performance of \cref{alg:q-learning}, which approximates the VaR computed from DP. 
In \cref{sec:Q-learning-algorithm}, we introduced a general VaR-Q-learning handling sampled time horizons. However, in practice, we remove the time index to reduce the computation overhead associated with updating time-indexed value functions.  
We take $\kappa \in \{10^{-4},10^{-8},10^{-12},0\}$ for the $\kappa$-soft quantile loss with a uniform discretization of $J=256$. For $\kappa = 0$, the loss is that of \cref{eq:quantile-loss} while for positive values, it is that of \cref{eq:delageLoss1}.
\cref{fig:inv2_wasserstein} displays the 1-Wasserstein distance between the quantile value estimated from VaR-Q-learning and the quantile value $\ushort{q}^\discretized$ computed via DP (\cref{eq:discreteBellmanEq}). For all $\kappa$'s, we see that the distance converges to zero as the number of samples increases. Furthermore, the VaR-Q-learning policy performs similarly to DP (see also \cref{fig:all_algorithms_Q_learning_value} in \cref{sec:more_experiments}).

To summarize, our experiments illustrate that the policy returned by the VaR-Q-learning algorithm: (1) Outperforms other baselines across both domains and quantile levels; (2) Lies in $[\ushort{q},\bar{q}]$; and (3) Performs similarly as the DP optimal policy $\ushort{\pi}$.

\section{Related Work and Discussion}

Several works propose model-free methods for RARL. \cite{mihatsch2002risk} introduce a temporal difference scheme for prediction and control with convergence guarantees, but focus on a specific utility-based shortfall risk. Converging Q-learning algorithms are further extended in \cite{shen2014risk, borkar2021prospect} to a larger class of utility functions. However, all these works focus on a \emph{nested} risk measure, which provides Bellman equations at the expense of interpretable policies. Differently, \cite{stanko2021cvar} study Q-learning for \emph{static} CVaR, but their analysis relies on the DP equations of \cite{chow2015risk}, which were shown to be incorrect by \cite{hau2023dynamic}. When considering static risk measures, one must use a proper state augmentation to guarantee that an optimal policy is identified \citep{bauerle2011markov, hau2023dynamic}. Otherwise, limiting assumptions such as the existence of an optimal Markov policy are required \citep{lim2022distributional}.

\cite{dabney2018implicit} present a similar goal as our study and propose to train risk-sensitive policies using a quantile representation of the return distribution. By leveraging Q-learning for DiRL \citep{bellemare2017distributional}, they introduce IQN, an algorithm capable of achieving risk-sensitive behavior. Yet, as first pointed out in \cite{lim2022distributional}, the greedy step employed in IQN lacks a clear criterion of optimality for the trained policy. In this regard, our Q-learning algorithm modifies the optimal action selected at each state-risk level pair, thus fixing IQN's deficiency in policy optimization. For policy evaluation, our algorithm reduces to a variant of IQN ensuring that the resulting approximated distribution functions under-estimate (in terms of stochastic ordering) the return distribution. A detailed discussion on the differences between IQN and our VaR-Q-learning algorithm can be found in \cref{sec:appIQN}.

Most related to our work is the one by \cite{gilbert2016quantile} in which the authors propose a Q-learning algorithm to identify a VaR-optimal policy on a special class of MDPs with end states. There, preferences are expressed using an ordering over end states. Following \cite{borkar1997}, their algorithm is based on stochastic approximation with two time-scales and its convergence is only empirically demonstrated. They leave open the question of how to generalize the approach to other forms of MDPs and raise the question of whether quantile regression methods could be used, which we address it in this work.

Looking forward, the question of extending our results to an infinite horizon setting with continuous state and/or action spaces is definitely interesting. One might also be able to adapt the convergence analysis of policy evaluation for DiRL in \cite{rowland2023analysis} to formally establish the convergence properties of \cref{alg:q-learning} under the non-strongly convex objective $\ell_\alpha$.

\section*{Acknowledgements}
We would like to thank Marc G. Bellemare for valuable discussions on the topic of distributional RL and the IQN algorithm. Esther Derman was partially funded by IVADO. Erick Delage was partially supported by the Canadian Natural Sciences and Engineering Research Council [Grant RGPIN-2022-05261] and by the Canada Research Chair program [950-230057]. Jia Lin Hau was partially funded by the University of New Hampshire dissertation year fellowship 2024-25. Jia Lin Hau and Marek Petrik were partially supported by NSF grants 2144601 and 2218063.

\bibliographystyle{apalike}
\bibliography{main}

\appendix

 \section{Proofs of Section \ref{sec:sett-prel}}
 \label{sec:preliminaries-proofs}

\subsection{Proof of Lemma \ref{lem:var-elicitable}}
\label{apx: lemma elicitable}

\begin{proof}
    We aim to show that for $\tilde{x}\in\X, \alpha\in [0,1]$,
    \begin{align*}
        \argmin_{y\in\Real}\E[\max(\alpha(\tilde{x}-y), -(1-\alpha)(\tilde{x}-y))] = [\quant^-_{\alpha}(\tilde{x}),\quant^+_{\alpha}(\tilde{x})]\cap\Real,
    \end{align*}
    where $[\quant^-_{0}(\tilde{x}),\quant^+_{0}(\tilde{x})]\cap\Real=(-\infty,\quant^+_{0}(\tilde{x})]$ and $[\quant^-_{1}(\tilde{x}),\quant^+_{1}(\tilde{x})]=[\quant^-_{0}(\tilde{x}),\infty)$.
    For the case $\alpha\in(0,1)$, we refer the reader to \cite[Thm. 9]{Gneiting2011}. We study the case $\alpha=0$ as a similar set of arguments holds when $\alpha=1$. 
    By definition of $\quant^+_{0}(\tilde{x})$, for all $\bar{y}\in\Real$ such that $\bar{y}\leq \quant^+_{0}(\tilde{x})$, $\P{\tilde{x}<\bar{y}}\leq \P{\tilde{x}<\quant^+_{0}(\tilde{x})}\leq 0$ (i.e., equals zero) and therefore, by the law of total probability: 
    \[0\leq\E[\ell_0(\tilde{x}-\bar{y})]=\E[0|\tilde{x}\geq\bar{y}]\P{\tilde{x}\geq\bar{y}} + \E[-(\tilde{x}-\bar{y})|\tilde{x}<\bar{y}]\P{\tilde{x}<\bar{y}}=0.\]
    Hence, $\argmin_{y \in \Real} \E[\ell_0(\tilde{x}-y)] \supseteq (-\infty, \quant^+_0(\tilde{x})]$.

    By definition of $\quant^+_{0}(\tilde{x})$, for all $\bar{y}\in\Real$ such that $\bar{y}>\quant^+_{0}(\tilde{x})$, $\P{\tilde{x}<\bar{y}}> 0$ and:
    \[\E[\ell_0(\tilde{x}-\bar{y})]=\E[0|\tilde{x}\geq\bar{y}]\P{\tilde{x}\geq\bar{y}} + \E[-(\tilde{x}-\bar{y})|\tilde{x}<\bar{y}]\P{\tilde{x}<\bar{y}}=\E[\bar{y}-\tilde{x}|\tilde{x}<\bar{y}]\P{\tilde{x}<\bar{y}}>0.\]
      The last inequality exploits two facts. First, by left continuity of $h(z):=\P{\tilde{x}<\bar{y}+z}$, $\P{\tilde{x}<\bar{y}}> 0$ implies that there must be some $\epsilon>0$ for which $\P{\tilde{x}\leq \bar{y}-\epsilon}>0$. Second, by the Markov inequality $\E[\bar{y}-\tilde{x}|\bar{y}-\tilde{x}>0]\geq \epsilon \P{\bar{y}-\tilde{x}\geq \epsilon | \bar{y}-\tilde{x}>0} = \epsilon \P{\bar{y}-\tilde{x}\geq \epsilon}/\P{\bar{y}-\tilde{x}>0}>0$. Hence, $\argmin_{y \in \Real} \E[\ell_0(\tilde{x}-y)] \subseteq (-\infty, \quant^+_0(\tilde{x})]$. 
      We have shown $\argmin_{y \in \Real} \E[\ell_0(\tilde{x}-y)] = (-\infty, \quant^+_0(\tilde{x})]$, which ends the proof for $\alpha=0$ since $\quant^-_{0}(\tilde{x})=-\infty$.
\end{proof}

\subsection{From VaR state-action to VaR state value function}
\label{apx: proof value-qvalue}

\begin{proposition}
\label{prop:value-qvalue}
The optimal value functions satisfy for each $s\in \mathcal{S}$, $\alpha\in [0,1]$, and $t \in  [T]$ that
  \[
    v_t\opt(s,\alpha) \; =\;  \max_{a\in \mathcal{A}} \, q_t\opt(s,\alpha,a). 
  \]
\end{proposition}

\begin{proof}
The result follows straightforwardly from the definition of $v_t\opt$ and $q_t\opt$. Namely,
\begin{align*}
  v_t\opt(s,\alpha) = \max_{\pi\in \PiHD^t}\varo_{\alpha}^{\pi,s}\left[\sum_{k=0}^{t-1} \gamma^{k} r(\tilde{s}_{k},\tilde{a}_{k})\right]
  = \max_{a\in\actions,\pi\in \PiHD^t:\pi_0(s)=a}\varo_{\alpha}^{\pi,s}\left[\sum_{k=0}^{t-1} \gamma^{k} r(\tilde{s}_{k},\tilde{a}_{k})\right]
  = \max_{a\in \mathcal{A}} q_t\opt(s,\alpha,a).
\end{align*}    
\end{proof}

\section{Proofs of Section \ref{sec:var-dp-equations}}
\label{sec:proofs-crefs-dp}

\subsection{Proof of Theorem \ref{thm:q-optimal:old}}

\begin{proof}
This proof extends the results obtained in \cite[Appx. C]{hau2023dynamic} to the case where there exist $(s,a,s')$ tuples for which $p(s,a,s')=0$, and reparameterizes the representation. The result on the definition and optimality of $\pi\opt$ comes directly from \citep{hau2023dynamic}. Specifically, by \cite[Appx. C]{hau2023dynamic}, it holds that $\theirs{q}_t = q\opt_t,\forall t\in[T]$, for all sequences $\theirs{q}:=(\theirs{q}_t)_{t=0}^{T}$ such that $\theirs{q}_0(s,\alpha,a)= \var{\alpha}{0}$ and
\begin{align*}
  \theirs{q}_{t+1}(s, \alpha ,a)
  &\;=\;  r(s,a) + \gamma  \max_{\zeta \in \Xi_{sa}(\alpha)} \min_{s' \in\states:p(s,a,s')>0} \max_{a'\in \actions} \theirs{q}_{t}\left( s', \frac{\alpha\zeta_{s'}}{p(s,a,s')},a'\right), \quad \forall t\in [T-1],
\end{align*}
where 
\begin{align*} 
    \Xi_{sa}(\alpha):=\{\zeta\in[0,1]^S \mid\sum_{s'\in\states}\zeta_{s'}=1, \alpha \zeta_{s'}\leq p(s,a,s'),\quad\forall s' \in \states\}.
\end{align*}
By construction of $q=(q_t)_{t=0}^{T}$ from the theorem statement, we have $q_0 = \theirs{q}_0$, so it remains to establish $B_{\max}q_{t} = \theirs{q}_{t+1}$ for all $t\in [T-1]$. By mathematical induction, assuming that $q_t = \theirs{q}_t$ for some $t\in[T-1]$, we show that $q_{t+1} := B_{\max}q_{t} = \theirs{q}_{t+1}$ in two steps : (1) $B_{\max}q_{t}\leq \theirs{q}_{t+1}$ and (2) $B_{\max}q_{t}\geq \theirs{q}_{t+1}$ .

\textbf{Step 1: Establishing $B_{\max}q_{t}\leq \theirs{q}_{t+1}$. }

Let $o \opt \in \mathcal{O}_{sa} (\alpha)$ (which is non-empty given that $0\in \mathcal{O}_{sa} (\alpha)$) be an optimal point for operator $B_{\max}$ in \cref{eq:q-optimal-dp} 
and consider the function: 
\begin{align*}
g(o):&=\min_{s' \in\states}\max_{a'\in \actions} q_{t}( s', o_{s'},a') \\
&=\min_{s' \in\states} \max_{a'\in \actions} q_{t}\opt( s', o_{s'},a') \\
&=\min_{s' \in\states}   \max_{\pi\in \PiHD^{t}} \varo^{\pi,s'}_{o_{s'}}\left[ \sum_{k=0}^{t-1} \gamma^k \cdot  r(\tilde{s}_k, \tilde{a}_k)\right]. 
\end{align*}
By properties of VaR and $\min$ operators, $g$ is non-decreasing in $o$. Therefore, $\max_{o \in \mathcal{O}_{sa} (\alpha)}g(o)$ is achieved inside $\{o\in [0,1]^S|\sum_{s'\in\states}o_{s'} p(s,a,s')=\alpha\}$ and necessarily, $\sum_{s'\in\states}o_{s'}\opt p(s,a,s')=\alpha$.

Now let $\zeta\in\Real^{S}$ be defined as:
\begin{align*}
\zeta_{s'}:=
    \begin{cases}
        \frac{o_{s'}\opt p(s,a,s')}{\alpha} \text{ if } \alpha > 0\\
        p(s,a,s') \text{ otherwise.}
    \end{cases}
\end{align*}
for all $s'\in\states$.
We aim to show that
$\zeta\in \Xi_{sa}(\alpha)$. When $\alpha =0$, the claim follows by definition of a transition kernel $p$. When $\alpha>0$,  we get
\begin{itemize}
    \item $\sum_{s'\in \states}\zeta_{s'} = \frac{\sum_{s'\in\states}o_{s'}\opt p(s,a,s')}{\alpha} = 1$
    \item $p(s,a,s')\geq 0,o\opt_{s'} \ge 0 \implies \zeta_{s'} = \frac{o_{s'}\opt p(s,a,s')}{\alpha} \ge 0~ ,\forall s'\in \states$
    \item $(\zeta_{s'}\ge 0 ~\forall s'\in \states~,~ \sum_{s'\in \states}\zeta_{s'} = 1) \implies (\zeta_{s'} \le 1 ~\forall s'\in \states$)
    \item $o\opt_{s'} \le 1 \implies \alpha \zeta_{s'} = o_{s'}\opt p(s,a,s')\le p(s,a,s')~ ,\forall s'\in \states$ 
\end{itemize}
In both cases, $\alpha\zeta_{s'}=o_{s'}\opt p(s,a,s'), \forall s'\in\states$. Indeed, for $\alpha =0$, the condition $\sum_{s'\in\states}o_{s'}\opt p(s,a,s')=\alpha$ implies $o\opt_{s'}=0$ whenever $p(s,a,s') > 0$, so $o_{s'}\opt p(s,a,s') = 0 = \alpha\zeta_{s'}$ for all $s'\in\states$. 
We can thus deduce: 
\begin{align*}
    B_{\max}q_{t}(s,\alpha, a) &=  r(s,a) + \gamma \max_{o \in \mathcal{O}_{sa}(\alpha)} \min_{s' \in\states}  \max_{a'\in \actions} q_t( s', o_{s'},a')\\
    &=  r(s,a) + \gamma \min_{s' \in\states}  \max_{a'\in \actions} q_t( s', o\opt_{s'},a')\\
    &\leq  r(s,a) + \gamma \min_{s' \in\states:p(s,a,s')>0}\max_{a'\in \actions} q_t( s', o\opt_{s'},a')\\
    &= r(s,a) + \gamma \min_{s' \in\states:p(s,a,s')>0}\max_{a'\in \actions} q_t\left( s', \frac{\alpha\zeta_{s'}}{p(s,a,s')},a'\right)\\
    &\leq r(s,a) + \gamma \max_{\zeta \in \Xi_{sa}(\alpha)}\min_{s' \in\states:p(s,a,s')>0}\max_{a'\in \actions} q_t\left( s', \frac{\alpha\zeta_{s'}}{p(s,a,s')},a'\right) = \theirs{q}_{t+1}(s,\alpha, a).
\end{align*}

\textbf{Step 2: Establishing $B_{\max}q_{t}\geq \theirs{q}_{t+1}$. }
We proceed similarly. Let an optimal $\zeta\opt \in \Xi_{sa}(\alpha)$ (it exists since $\zeta:=p(s,a,\cdot)$ is always feasible) satisfying: 
\begin{align*}
    \theirs{q}_{t+1}(s,\alpha,a)= r(s,a) + \gamma \min_{s' \in\states:p(s,a,s')>0}\max_{a'\in \actions} q_t\left( s', \frac{\alpha\zeta\opt_{s'}}{p(s,a,s')},a'\right),
\end{align*}
and define $o\in \Real^{S}$ as 
\begin{align*}
    o_{s'}:=
    \begin{cases}
        \frac{\alpha\zeta_{s'}\opt}{p(s,a,s')} &\text{ if } p(s,a,s')> 0,\\
        1 &\text{ otherwise.}
    \end{cases}
\end{align*}

To check if $o\in\mathcal{O}_{sa}(\alpha)$, we remark that for any $s'\in\states$ with $p(s,a,s')>0$, 
\begin{align*}
    \frac{\alpha\zeta_{s'}\opt}{p(s,a,s')}&\geq 0  &&[\alpha\geq 0, \zeta\opt\geq 0, p(s,a,s')> 0]\\
    \frac{\alpha\zeta_{s'}\opt}{p(s,a,s')}&\leq 1, &&[\alpha\zeta_{s'}\opt\leq p(s,a,s')]
\end{align*}
so $o_{s'}\in[0,1]$ when $p(s,a,s')>0$. Otherwise, $o_{s'}=1\in[0,1]$. Additionally, 
\[\sum_{s'\in\states}o_{s'}p(s,a,s')=  \sum_{s'\in\states:p(s,a,s')>0}\frac{\alpha\zeta_{s'}\opt}{p(s,a,s')} \cdot p(s,a,s') = \sum_{s'\in\states:p(s,a,s')>0}\alpha\zeta_{s'}\opt=\alpha,\]
so $o\in\mathcal{O}_{sa}(\alpha)$.
We can thus establish:
\begin{align*}
    B_{\max}q_{t}(s,\alpha, a) &= r(s,a) + \gamma \max_{o \in \mathcal{O}_{sa}(\alpha)} \min_{s' \in\states}  \max_{a'\in \actions} q_t( s', o_{s'},a')\\
    &\geq r(s,a) + \gamma  \min_{s' \in\states}  \max_{a'\in \actions} q_t( s', o_{s'},a')\\
    &\stackrel{(\text{a})}{=} r(s,a) + \gamma  \min_{s' \in\states: p(s,a,s')>0}  \max_{a'\in \actions} q_t( s', o_{s'},a')\\
    &= r(s,a) + \gamma  \min_{s' \in\states: p(s,a,s')>0}  \max_{a'\in \actions} q_t\left( s', \frac{\alpha\zeta_{s'}\opt}{p(s,a,s')},a'\right)\\
    &= r(s,a) + \gamma  \max_{\zeta \in \Xi_{sa}(\alpha)} \min_{s' \in\states: p(s,a,s')>0}  \max_{a'\in \actions} q_t\left( s', \frac{\alpha\zeta_{s'}}{p(s,a,s')},a'\right) \\
    &= \theirs{q}_{t+1}(s,\alpha, a). 
\end{align*}
Equality (a) stems from the implications: $p(s,a,s')=0 \implies o_{s'}=1 \implies q_t(s',o_{s'},a') = \infty$ and for any $c \in \bar{\Real}$, $\min(\infty,c) = c$.
We now conclude the proof by induction, as $q_{t+1} := B_{\max}q_{t} = \theirs{q}_{t+1} \forall t \in [T-1]$.
\end{proof}

\subsection{Proof of Theorem \ref{thm:q-optimal:new}}
Before diving into the proof of the theorem, we show the following general lemmas which will eventually be used for establishing \cref{thm:q-optimal:new}.

\begin{lemma}
\label{thm:propVsalpha}
For all $t\in [T]$ and $s\in\states$, $\alpha\mapsto v_t\opt(s,\alpha)$ is non-decreasing, right-continuous on $[0,1]$.
\end{lemma}

\begin{proof}
    \cref{prop:value-qvalue} indicates that for any $t\in [T]$,
    \[v_t\opt(s,\alpha)=\varo^{\pi\opt,s}_{\alpha}\left[\sum_{k=0}^{t-1} \gamma^{k} r(\tilde{s}_{k},\tilde{a}_{k})\right].\] 
    By \citep[Lem. A.19]{Follmer2016},  $\alpha\mapsto\varo_\alpha[\tilde{x}]$ is non-decreasing and right-continuous for any $\tilde{x}\in\mathbb{X}$, so the result follows.
\end{proof}

\begin{lemma}\label{thm:varRcNondecrease}
    Let $f:[0,1]\rightarrow \bar{\Real}$ be non-decreasing and $\tilde{u}$ be a uniform random variable over $[0,1]$. Then, we have that $\varo_\alpha[f(\tilde{u})]=f(\alpha)$ for all $\alpha\in[0,1)$ where $f(\alpha)$ is right-continuous. Moreover, if $f(1)=\infty$, the equality also holds at $\alpha=1$.
\end{lemma}

\begin{proof}
    By definition of VaR (see \cref{eq:var-definition}), we have:
$\var{\alpha}{f(\tilde{u})} =\max\{z\in \bar{\Real}|\P{f(\tilde{u})<z}\leq \alpha\}$.
When $\alpha=1$, $\P{f(\tilde{u})<z}\leq \alpha$ for all $z\in \bar{\Real}$ so that $\varo_\alpha[f(\tilde{u})]=\infty$ and the second part of the statement holds. Let thus $\alpha\in[0,\,1)$ be such that $f(\alpha)$ is right-continuous at $\alpha$. By assumption on $f$ being non-decreasing, $\tilde{u} \geq \alpha$ implies $f(\tilde{u})\geq f(\alpha)$ and we can establish: 
\begin{align*}
    &\P{\tilde{u} \geq\alpha} \leq \P{f(\tilde{u}) \geq f(\alpha)}\\
    \iff&1-\P{\tilde{u}  < \alpha} \leq 1- \P{f(\tilde{u}) < f(\alpha)}\\
    \iff& 1-\alpha \leq 1- \P{f(\tilde{u}) < f(\alpha)} &&[\P{\tilde{u}  < \alpha} = \P{\tilde{u}  \leq \alpha} = \alpha]\\
    \iff& \alpha \geq \P{f(\tilde{u}) < f(\alpha)}.
\end{align*}
As a result, $f(\alpha) \leq\var{\alpha}{f(\tilde{u})}$.

On the other hand, the right-continuity of $f$ at $\alpha$ ensures that for all $\epsilon>0$ there exists a $\delta>0$ such that $f(\alpha+\delta)< f(\alpha)+\epsilon$. Thus
\begin{align*}
    \P{f(\tilde{u})<f(\alpha)+\epsilon} &\geq \P{f(\tilde{u})\leq f(\alpha+\delta)}  &&[\text{By construction:  }f(\alpha+\delta)< f(\alpha)+\epsilon]\\
    &\geq\P{\tilde{u}\leq\alpha+\delta} &&[\tilde{u} \leq \alpha+\delta \implies f(\tilde{u})\leq f(\alpha+\delta)] \\
    &=\alpha+\delta \\
    &> \alpha.
\end{align*}
Hence, $\var{\alpha}{f(\tilde{u})}\leq f(\alpha)+\epsilon$ for all $\epsilon>0$. Setting $\epsilon\to 0$, $\var{\alpha}{f(\tilde{u})}\leq f(\alpha)$. 

We conclude that $\var{\alpha}{f(\tilde{u})}= f(\alpha)$ for all $\alpha\in[0,\,1)$ at which $f(\alpha)$ is right-continuous.
\end{proof}

The result below directly follows from \cref{thm:propVsalpha,thm:varRcNondecrease}. 
\begin{corollary}\label{thm:vIsVaR}
    For all $t\in [T]$ and $s\in\states$, we have $\varo_\alpha[v_t\opt(s,\tilde{u})]=v_t\opt(s,\alpha)$, where $\tilde{u}$ is a uniform random variable on $[0,1]$.
\end{corollary}

\begin{lemma}
    \label{thm:var:decomp2}
    Let $S \in \Natural$ non-decreasing functions \zeroOneIntervalOptions{$f_i\colon (0,1)\rightarrow\mathbb{R}, i\in 1{:}S$}{$f_i\colon [0,1]\rightarrow\bar{\mathbb{R}}, i\in 1{:}S$ with each $f_i(\alpha)\in\Real$ for all $\alpha\in(0,1)$}. Let also $\tilde{y}$ be a discrete random variable on $1{:}S$ with probability mass function $\hat{p}_i := \mathbb{P}[\tilde{y} = i],\quad \forall i \in 1{:}S$, and $\tilde{u}$ an independent random variable with uniform distribution on $[0,1]$. Then, we have
    \begin{equation} \label{eq:var:decomp2}
      \var{\alpha}{f_{\tilde{y}}(\tilde{u})}
      \quad = 
      \max_{\bm{o}\in[0,1]^S} \; \left\{\min_{i\in 1:S}\; \var{o_i}{f_i(\tilde{u})} \mid \sum_{j=1}^S o_j \hat{p}_j \leq \alpha\right\}.
    \end{equation}
    \end{lemma}
\begin{proof} 
This proof closely follows that of \citep[Thm. 5.1]{hau2023dynamic} but relaxes the assumption $\hat{p}_i > 0$ and simplifies the notation.
Let $\mathfrak{I} := \{i \in 1{:}S \mid \hat{p}_{i} > 0\}$. We decompose VaR based on its definition in \cref{eq:quantile-definition,eq:var-definition}: 
\begin{align}
  \nonumber
  \var{\alpha}{f_{\tilde{y}}(\tilde{u})}
  &= \max \; \left\{\tau \in \bar{\Real} \mid \Pr{ f_{\tilde{y}}(\tilde{u}) < \tau } \le  \alpha\right\}
  \\ \nonumber
  &\overset{\textrm{(b)}}{=} \max\; \left\{\tau\in \bar{\Real}  \mid \sum_{i \in \mathfrak{I}} \Pr{ f_{\tilde{y}}(\tilde{u})< \tau \mid \tilde{y}=i }\cdot \hat{p}_i \leq \alpha\right\} \\ \nonumber
  &= \max\; \left\{\tau\in \bar{\Real}  \mid \sum_{i \in \mathfrak{I}} \Pr{ f_i(\tilde{u})< \tau  }\cdot \hat{p}_i \leq \alpha\right\}\\ \nonumber
  &\overset{\textrm{(c)}}{=} \max\; \left\{\tau \in \bar{\Real} \mid \exists \bm{o}\in[0,1]^{S},\;  \Pr{ f_i(\tilde{u}) < \tau  } \leq o_i ,  \; \forall i \in \mathfrak{I}, \; \sum_{j \in \mathfrak{I}} o_j\hat{p}_j \le \alpha\right\}\\  \nonumber
  &= \max\; \left\{\tau \in \bar{\Real} \mid \exists \bm{o}\in[0,1]^{S} , \; \Pr{ f_i(\tilde{u}) < \tau  } \leq o_i,  \; \forall i \in \mathfrak{I},  \; \sum_{j=1}^S o_j\hat{p}_j \le \alpha\right\}\\  \nonumber
  &\overset{\textrm{(d)}}{=} \max_{\bm{o}\in[0,1]^{S}} \left\{  \max\; \left\{\tau \in \bar{\Real} \mid   \Pr{ f_i(\tilde{u}) < \tau   } \leq o_i , \;\forall i \in \mathfrak{I} \right\} \mid  \sum_{j=1}^S o_j\hat{p}_j \le \alpha\right\}  \\  \nonumber
  &= \max_{\bm{o}\in[0,1]^{S}} \left\{  \max\; \bigcap_{i\in \mathfrak{I}}\left\{\tau \in \bar{\Real} \mid   \Pr{ f_i(\tilde{u}) < \tau   } \leq o_i \right\} \mid  \sum_{j=1}^S o_j\hat{p}_j \le \alpha\right\}  \\  \nonumber
  &\overset{\textrm{(e)}}{=} \max_{\bm{o}\in[0,1]^{S}} \left\{  \min_{i\in \mathfrak{I}} \max  \; \left\{\tau \in \bar{\Real} \mid   \Pr{ f_i(\tilde{u}) < \tau   } \leq o_i  \right\} \mid  \sum_{j=1}^S o_j\hat{p}_j \le \alpha\right\}  \\
 &\overset{\textrm{(f)}}{=} \max_{\bm{o}\in[0,1]^{S}} \; \left\{\min_{i\in \mathfrak{I}}\; \var{o_i}{f_i(\tilde{u})  } \mid \sum_{j=1}^S o_j \hat{p}_j \leq \alpha\right\}. \label{eq:var-decomposition-part2}
\end{align}
In the derivation above, step (b) follows from the law of total probability and omitting zero probability events. Then we lower-bound them by an auxiliary variable $o_i$ in step (c). In step (d) we replace the joint maximum over $\tau$ and $\bm{o}$ by sequential max, and then we replace the max of an intersection by the minimum of the maxima of sets in (e). The equality in (e) holds because $\tau \mapsto\Pr{ f_i(\tilde{u}) < \tau}$ is monotone and, therefore, the sets $\left\{\tau \in \bar{\Real} \mid   \Pr{ f_i(\tilde{u}) < \tau \mid \tilde{y}=i } \leq o_i \right\}$ are nested. Step (f) holds by definition of VaR. 

It remains to show that \eqref{eq:var-decomposition-part2} equals \eqref{eq:var:decomp2}. 
Suppose that $o\opt \in [0,1]^S$ is optimal in \eqref{eq:var-decomposition-part2} and construct $\bar{o} \in [0,1]^{S}$ as
\[
 \bar{o}_i =
 \begin{cases}
   o\opt_i &\text{if  } i\in \mathfrak{I}, \\
   1 &\text{otherwise},
 \end{cases} \qquad \forall i\in 1{:}S.
\]
Since $\bar{o}$ is feasible in \eqref{eq:var:decomp2} with the same objective, \eqref{eq:var-decomposition-part2} $\le$ \eqref{eq:var:decomp2}. To show that \eqref{eq:var-decomposition-part2} $\ge$ \eqref{eq:var:decomp2} suppose that $o\opt \in [0,1]^{S}$ is optimal in \eqref{eq:var:decomp2}. The inequality then holds because $o\opt$ is feasible in \eqref{eq:var-decomposition-part2} and because
\[
  \min_{i\in \mathfrak{I}}\; \var{o_i\opt}{f_i(\tilde{u})}  \ge
  \min_{i\in 1{:}S}\; \var{o_i\opt}{f_i(\tilde{u}) }.
\]

\end{proof}

\begin{lemma}\label{thm:maxVarmax}
Assume non-decreasing functions \zeroOneIntervalOptions{$f_i\colon (0,1)\rightarrow\mathbb{R}, i\in 1{:}S$}{$f_i\colon [0,1]\rightarrow\bar{\mathbb{R}}, i\in 1{:}S$ with each $f_i(\alpha)\in\Real$ for all $\alpha\in(0,1)$}. Then, it holds that $\max_{i\in 1:S}\varo_{\tilde{u}_1}[f_i(\tilde{u}_2)] = \max_{i\in 1:S} f_i(\tilde{u}_1)$ almost surely, where $\tilde{u}_1$ and $\tilde{u}_2$ are two independent uniform random variables on $[0,1]$.
\end{lemma}

\begin{proof}
By \cref{thm:varRcNondecrease}, for any $j\in 1{:}S$,  $\varo_{\bar{\alpha}}[f_j(\tilde{u})]=f_j(\bar{\alpha})$ at any value of $\bar{\alpha}\zeroOneIntervalOptions{\in(0,1)}{\in [0,1)}$ where $f_j(\cdot)$ is right-continuous. This implies that for all $\alpha\zeroOneIntervalOptions{\in(0,1)}{\in [0,1)}$ where  all $f_i$'s are right-continuous, $\max_{i\in 1:S}\varo_{\alpha}[f_i(\tilde{u})]=\max_{i\in 1:S}f_i(\alpha)$  since the maximum is necessarily right-continuous then. For $i\in 1{:}S$, $f_i$ is monotone on the interval $(0,1)$, so by Froda's theorem \citep[Thm. 4.30]{rudin1964principles}, the number of discontinuities of $f_i$ must be at most countable \zeroOneIntervalOptions{}{on $(0,1)$ so therefore also on $[0,1]$}.  
This implies that the number of points $\alpha$ at which some $f_i$ from $i\in 1{:}S$ is discontinuous is at most countable. 
We thus conclude that $\max_{i\in 1:S} f_i(\tilde{u}_1)=\max_{i\in 1:S}\varo_{\tilde{u}_1}[f_i(\tilde{u}_2)]$ with probability one.
\end{proof}

\begin{lemma}\label{thm:maxVaR} Let $S$ non-decreasing functions \zeroOneIntervalOptions{$f_i\colon (0,1)\rightarrow\mathbb{R}, i\in 1{:}S$}{$f_i\colon [0,1]\rightarrow\bar{\mathbb{R}}, i\in 1{:}S$ with each $f_i(\alpha)\in\Real$ for all $\alpha\in(0,1)$}. Then $\varo_{\alpha}[\max_{i\in 1:S}f_i(\tilde{u})]=\max_{i\in 1:S}\varo_\alpha[f_i(\tilde{u})]$, where $\tilde{u}$ is a uniform random variable on $[0,1]$  
\end{lemma}

\begin{proof}
    First, the claim trivially applies to $\alpha=1$ since $\varo_1[\tilde{x}]=\infty$ for all random variables $\tilde{x}\in\X$. Thus, we focus on the case $\alpha\in[0,1)$. Since $\alpha\mapsto\varo_{\alpha}[\cdot]$ is non-decreasing (see \citep[Lem. A.19]{Follmer2016}), $\varo_{\alpha}[\max_{i\in 1:S}f_i(\tilde{u})]\geq \varo_{\alpha}[f_j(\tilde{u})]$ for all $j\in 1{:}S$, hence $\varo_{\alpha}[\max_{i\in 1:S}f_i(\tilde{u})]\geq\max_{i\in 1:S}\varo_\alpha[f_i(\tilde{u})]$. We are therefore left with showing that $\varo_{\alpha}[\max_{i\in 1:S}f_i(\tilde{u})]\leq\max_{i\in 1:S}\varo_\alpha[f_i(\tilde{u})]$. We do so by contradiction. Assume that
    \[\varo_{\alpha}[\max_{i\in 1:S}f_i(\tilde{u})]>\max_{i\in 1:S}\varo_\alpha[f_i(\tilde{u})]=:\nu\opt.\]
Applying \cref{thm:maxVarmax}, $\max_{i\in 1:S}f_i(\tilde{u}) = \max_{i\in 1:S}\varo_{\tilde{u}}[f_i(\tilde{u}_2)]$ almost surely, so we must have $\nu\opt<\varo_{\alpha}[\max_{i\in 1:S}\varo_{\tilde{u}}[f_i(\tilde{u}_2)]]$ where $\tilde{u}_2$ is uniformly distributed on $[0,1]$.
    By definition of $\varo_\alpha[\cdot]$ (Eq.~\eqref{eq:var-definition}), this implies that there exists $\epsilon>0$ such that:
    \[\P{\max_{i\in 1:S}\varo_{\tilde{u}}[f_i(\tilde{u}_2)]<\nu\opt+\epsilon}\leq \alpha.\]
    Since $\nu\opt = \max_{i \in 1:S} \varo_\alpha[f_i(\tilde{u})]$ and $\alpha\mapsto\varo_\alpha[f_i(\tilde{u})]$ is non-decreasing,
    we must have $\max_{i\in 1:S}\varo_{\alpha'}[f_i(\tilde{u})]\leq \nu\opt<\nu\opt+\epsilon$ for all $\alpha'\leq \alpha$. In addition, by the law of total probability:
    \begin{align*}
        \alpha &\geq \P{\max_{i\in 1:S}\varo_{\tilde{u}}[f_i(\tilde{u}_2)]<\nu\opt+\epsilon}\\
        &= \P{\max_{i\in 1:S}\varo_{\tilde{u}}[f_i(\tilde{u}_2)]<\nu\opt+\epsilon |\tilde{u}>\alpha}\P{\tilde{u}>\alpha} \\
        &\qquad+
        \P{\max_{i\in 1:S}\varo_{\tilde{u}}[f_i(\tilde{u}_2)]<\nu\opt+\epsilon |\tilde{u}\leq\alpha}\P{\tilde{u}\leq\alpha}\\
        &= (1-\alpha)\P{\max_{i\in 1:S}\varo_{\tilde{u}}[f_i(\tilde{u}_2)]<\nu\opt+\epsilon |\tilde{u}>\alpha} +\alpha\P{\max_{i\in 1:S}\varo_{\tilde{u}}[f_i(\tilde{u}_2)]<\nu\opt+\epsilon |\tilde{u}\leq\alpha}\\
        &= (1-\alpha)\P{\max_{i\in 1:S}\varo_{\tilde{u}}[f_i(\tilde{u}_2)]<\nu\opt+\epsilon |\tilde{u}>\alpha} +\alpha,
    \end{align*}
    so necessarily: 
    \[\P{\max_{i\in 1:S}\varo_{\tilde{u}}[f_i(\tilde{u}_2)]<\nu\opt+\epsilon|\tilde{u}>\alpha}=0.\]
    Yet, since $\max_{i\in 1:S}\varo_{\alpha'}[f_i(\tilde{u}_2)]$ is right-continuous, non-decreasing in $\alpha'$ and evaluates at $\nu\opt$ for $\alpha'=\alpha$, there must be a $\delta>0$ such that $\max_{i\in 1:S}\varo_{\alpha'}[f_i(\tilde{u}_2)]\leq \nu\opt+\epsilon$ for all $\alpha'\leq \alpha+\delta$. This leads to a contradiction since:
\begin{align*}
0&<\delta= \P{\tilde{u}\in(\alpha,\alpha+\delta]}\\
&= \P{\max_{i\in 1:S}\varo_{\tilde{u}}[f_i(\tilde{u}_2)]<\nu\opt+\epsilon| \tilde{u}\in(\alpha,\alpha+\delta]}\P{\tilde{u}\in(\alpha,\alpha+\delta]}\\
&= \P{\max_{i\in 1:S}\varo_{\tilde{u}}[f_i(\tilde{u}_2)]<\nu\opt+\epsilon , \tilde{u}\in(\alpha,\alpha+\delta]}\\
&\leq \P{\max_{i\in 1:S}\varo_{\tilde{u}}[f_i(\tilde{u}_2)]<\nu\opt+\epsilon , \tilde{u}>\alpha}\\
&= \P{\max_{i\in 1:S}\varo_{\tilde{u}}[f_i(\tilde{u}_2)]<\nu\opt+\epsilon|\tilde{u}>\alpha}(1-\alpha) = 0.    
\end{align*}

\end{proof}

We are now ready to prove \cref{thm:q-optimal:new}, whose statement is recalled below. 

\begin{theorem*} 
Let a sequence $q^{\mathrm{u}}=(q^{\mathrm{u}}_t)_{t=0}^{T}$ be such that $q^{\mathrm{u}}_0(s,\alpha,a)= \var{\alpha}{0}$ and $q^{\mathrm{u}}_{t+1}(s, \alpha ,a):= B_{\mathrm{u}} q^{\mathrm{u}}_t(s,\alpha, a)$ for $t\in[T-1]$. Then, $q^{\mathrm{u}}_t=q\opt_t$ for all $t\in[T]$, where $q\opt_t$ is defined in Eq.~\eqref{eq:quantile_q_defn}. Moreover, if a policy $\pi^{\mathrm{u}} = (\pi^{\mathrm{u}}_k)_{k=0}^{T-1}$ is greedy for $q^{\mathrm{u}}$ as in Eq.~\eqref{eq:greedy_var}, then it maximizes the value-at-risk objective~\eqref{eq:main-var-objective}. 
\end{theorem*}

\begin{proof}
We start with demonstrating that $q\opt_{t} = q^\mathrm{u}_t~, \forall t \in [T]$, recursively with
mathematical induction. Assuming that $q^\mathrm{u}_t = q\opt_t$ we want to show that $q^\mathrm{u}_{t+1} := B_{\mathrm{u}} q^\mathrm{u}_t = q\opt_{t+1}$ as written in \cref{eq:q-optimal-bellman}. We then prove the optimality of $\pi\opt$ constructed using $\hat{\alpha}^\mathrm{u}$.

    \textbf{Step 1:} For all $s\in \states, \alpha \in [0,1]$ and $a\in\actions$, $q\opt_0(s,\alpha,a)=q^{\mathrm{u}}_0(s,\alpha,a)= \var{\alpha}{0}$ so the base case holds. Assume that $q^\mathrm{u}_t = q\opt_t$ for some $t\in [T-1]$. Then, one can derive:
\begin{equation*} 
  \begin{aligned}
  q_{t+1}\opt(s,\alpha,a)&=   r(s,a) + \gamma \cdot \max_{o \in \mathcal{O}_{sa}(\alpha)} \min_{s' \in\states} \max_{a'\in \actions} q_{t}\opt( s', o_{s'},a')  &&[\text{By  \cref{thm:q-optimal:old}}]\\
&= r(s,a) + \gamma \cdot \max_{o \in \mathcal{O}_{sa}(\alpha)} \min_{s' \in\states} v_{t}\opt(s',o_{s'}) &&[\text{By \cref{prop:value-qvalue}}]\\
&= r(s,a) + \gamma \cdot \max_{o \in \mathcal{O}_{sa}(\alpha)} \min_{s' \in\states} \varo_{o_{s'}}[v_{t}\opt(s',\tilde{u})] &&[\text{By  \cref{thm:vIsVaR}}]\\
&= r(s,a) + \gamma\cdot\varo_{\alpha}^{a,s}[v_{t}\opt(\tilde{s}_1,\tilde{u})] &&[\text{By  \cref{thm:var:decomp2}}]\\
&= \varo_{\alpha}^{a,s}[r(s,a) + \gamma\cdot \max_{a'\in\actions}q_{t}\opt(\tilde{s}_1,\tilde{u},a')] &&[\text{By  \cref{prop:value-qvalue}}]\\
&= \varo_{\alpha}^{a,s}[r(s,a) + \gamma\cdot \max_{a'\in\actions}q^\mathrm{u}_{t}(\tilde{s}_1,\tilde{u},a')] &&[\text{By inductive assumption}]\\
&= B_{\mathrm{u}} q^\mathrm{u}_t(s,\alpha,a) = q^\mathrm{u}_{t+1}(s,\alpha,a) &&[\text{By \cref{eq:q-optimal-bellman}}].
  \end{aligned}
\end{equation*}
This confirms that $q\opt$ satisfies $q_{t+1}\opt = B_{\mathrm{u}} q_t\opt = B_{\mathrm{u}} q^{\mathrm{u}}_t = q_{t+1}^{\mathrm{u}}$ for all $t\in[T-1]$, so that $q\opt = q^{\mathrm{u}}$.

\textbf{Step 2:} We now show that $\hat{\alpha}^{\mathrm{u}}_k(\cdot)$ constructed according to: 
\begin{align*}
  \hat{\alpha}_{k+1}^{\mathrm{u}}(h_{k+1}):=  \min \left\{ o \in [0,1]
  \mid \max_{a\in\actions}q_{T-k-1}^{ \mathrm{u}}(s_{k+1},o,a)  \geq \frac{q_{T-k}^{\mathrm{u}}(s_{k},\hat{\alpha}_{k}(h_{k}),a_{k})-r(s_{k},a_{k})}{\gamma} \right\}
\end{align*}
defines an optimal policy. Namely,
\begin{align*}
    \hat{\alpha}_{k+1}^{\mathrm{u}}(h_{k+1}):=  \min \left\{ o \in [0,1]
  \mid \max_{a\in\actions}q_{T-k-1}^{ \mathrm{u}}(s_{k+1},o,a)  \geq \frac{q_{T-k}^{\mathrm{u}}(s_{k},\hat{\alpha}_{k}^{\mathrm{u}}(h_{k}),a_{k})-r(s_{k},a_{k})}{\gamma} \right\}.
\end{align*}
Based on what has been shown in Step 1, we can interchangeably write $q_{T-k-1}^{\mathrm{u}}$ or $q_{T-k-1}\opt$ in the construction of $\hat{\alpha}^{\mathrm{u}}_k, k\in [T-1]$, so that 
\begin{align*}
    \hat{\alpha}_{k+1}^{\mathrm{u}}(h_{k+1}):=  \min \left\{ o \in [0,1]
  \mid \max_{a\in\actions}q_{T-k-1}\opt(s_{k+1},o,a)  \geq \frac{q_{T-k}\opt(s_{k},\hat{\alpha}_{k}^{\mathrm{u}}(h_{k}),a_{k})-r(s_{k},a_{k})}{\gamma} \right\}.
\end{align*}
By \cref{thm:propVsalpha}, $\alpha\mapsto v_t\opt(s,\alpha)$ is right-continuous and non-decreasing so the minimum above is well-defined. 
We are left to check that $\hat{\alpha}^{\mathrm{u}}_k(\cdot)$ leads to an associated $\hat{o}^{ka}(h_k)\in\mathcal{O}_{sa}(\hat{\alpha}^{\mathrm{u}}_k(h_k))$ satisfying:
\begin{align}
\label{eq:alpha-construction}
    q_{T-k}&(s,\alpha_k(h_k),a) 
    = r(s,a) + \gamma\cdot \min_{s' \in\states}  \max_{a'\in \actions} q_{T-k-1}(s', \alpha_{k+1}(\langle h_k, a, s' \rangle), a')
\end{align}
This can be done in two steps. 

We can first show that for all $h_k\in\mathcal{H}_k, a\in\actions$, the vector $\hat{o}^k\in\mathbb{R}^{\states}$ composed of $\hat{o}_{s'}^k:=\hat{\alpha}^{\mathrm{u}}_{k+1}(\langle h_k,a,s'\rangle)$ for all $s'\in\states$, is in $\mathcal{O}_{sa}(\hat{\alpha}^{\mathrm{u}}_k(h_k))$. To do so, we make use of the fact that for all $t$ the maximum over $\max_{o \in \mathcal{O}_{sa}}$ in \eqref{eq:q-optimal-dp} is achieved thus implying that it is achieved at $T-k$ by some $\hat{o}^{k\star}$ when $s=s_k$ and $\alpha=\hat{\alpha}^{\mathrm{u}}_k(h_k)$. Namely, $\hat{o}^{k\star}$ satisfies:
\[q_{T-k}\opt(s_k,\hat{\alpha}^{\mathrm{u}}_k(h_k),a)=\min_{s'\in\states} r(s_k,a)+\gamma v_{T-k-1}\opt(s',\hat{o}^{k\star}_{s'})\leq r(s_k,a)+\gamma v_{T-k-1}\opt(s',\hat{o}^{k\star}_{s'}) \,,\,\forall s'\in\states\]
where we replaced $\max_{a'\in\actions}q\opt_{T-k-1}(s',\hat{o}^{k\star}_{s'},a')$ with $v_{T-k-1}\opt(s',\hat{o}^{k\star}_{s'})$.
We can therefore easily conclude that:
\begin{align*}
\sum_{s'\in\states}\hat{o}_{s'}^kp(s_k,a,s')=\sum_{s'\in\states}\hat{\alpha}^{\mathrm{u}}_{k+1}(\langle h_k,a,s' \rangle)p(s_k,a,s')
\sum_{s'\in\states}\hat{o}_{s'}^{k\star} p(s_k,a,s')\leq \hat{\alpha}^{\mathrm{u}}_k(h_k),
\end{align*}
which establishes that $\hat{o}^k\in \mathcal{O}_{sa}(\hat{\alpha}^{\mathrm{u}}_k(h_k))$.

We can then show that equation \eqref{eq:alpha-construction}, with $q_{T-k}$ and $q_{T-k-1}$ respectively replaced by their equivalent $q\opt_{T-k}$ and $q\opt_{T-k-1}$ (based on \cref{thm:q-optimal:old}), is satisfied based on:
\begin{align*}
    q_{T-k}\opt(s,\hat{\alpha}^{\mathrm{u}}_k(h_k),a)&=\max_{o \in \mathcal{O}_{sa}(\hat{\alpha}^{\mathrm{u}}_k(h_k))} \min_{s' \in\states} \left( r(s,a) + \gamma \cdot  \max_{a'\in\actions}q_{T-k-1}\opt( s', o_{s'},a') \right) &&[\text{By Eq.~\eqref{eq:q-optimal-dp}}]\\
    &\geq \min_{s' \in\states} \left( r(s,a) + \gamma \cdot  \max_{a'\in\actions}q_{T-k-1}\opt( s', \hat{\alpha}^{\mathrm{u}}_{k+1}(\langle h_k,a,s'\rangle),a')\right) &&[\text{}\hat{o}^k\in\mathcal{O}_{sa}(\hat{\alpha}^{\mathrm{u}}_k(h_k))]\\    
    &= \min_{s' \in\states} \left( r(s,a) + \gamma \cdot  v_{T-k-1}\opt( s', \hat{\alpha}^{\mathrm{u}}_{k+1}(\langle h_k,a,s'\rangle)) \right)&&[\text{By \cref{prop:value-qvalue}}]\\
&\geq \min_{s' \in\states} \left( r(s,a) + \gamma \cdot  \frac{q_{T-k}\opt(s,\hat{\alpha}^{\mathrm{u}}_k(h_k),a)-r(s,a)}{\gamma}\right)&&[\text{Definition of }\hat{\alpha}^{\mathrm{u}}_k]\\    
&= \min_{s' \in\states} q_{T-k}\opt(s,\hat{\alpha}^{\mathrm{u}}_k(h_k),a) = q_{T-k}\opt(s,\hat{\alpha}^{\mathrm{u}}_k(h_k),a), 
\end{align*}
The above inequalities must therefore be equalities, so equation \eqref{eq:alpha-construction} holds for $\hat{\alpha}^{\mathrm{u}}_k(h_k)$. As a result, the policy $\pi\opt$ constructed using $\hat{\alpha}^{\mathrm{u}}$ is optimal.
\end{proof}

\section{Proofs of Section \ref{sec:var-q-learning}}\label{sec:appProofSection5}

\subsection{Risk Measures}

\begin{definition}[Monetary risk measure]
    A monetary risk measure is a mapping $\varrho: \mathbb{X} \to \bar{\Real}$ satisfying the following properties:
    \begin{itemize}
        \item[1.]\textit{Translation invariance:} For all $\tilde{x}\in\mathbb{X}, c\in\Real, \varrho(\tilde{x} + c) = \varrho(x) +c$
        \item[2.]\textit{Monotonicity:} For all $\tilde{x}, \tilde{y}\in\mathbb{X}, \tilde{x}\leq \tilde{y} \implies \varrho(\tilde{x})\leq \varrho(\tilde{y}).$
    \end{itemize}
\end{definition}

\begin{lemma} \label{lem:risk-nonexpansive}
For any monetary risk measure $\varrho\colon \mathbb{X} \to \Real$ where $\mathbb{X}$ is defined in a finite outcome space $\omega\in\Omega$, it holds that:
\[
   |\varrho(\tilde{x}) - \varrho(\tilde{y})|
   \le
   \max_{\omega \in \Omega} |\tilde{x}(\omega) - \tilde{y}(\omega)| .
 \]
\end{lemma}
\begin{proof}
Define $\epsilon := \max_{\omega \in \Omega} |\tilde{x}(\omega) - \tilde{y}(\omega)|\geq 0$.  We prove that $\varrho(\tilde{x}) - \varrho(\tilde{y}) \le \epsilon$. The second inequality \(    \varrho(\tilde{y}) - \varrho(\tilde{x}) \le \epsilon \) follows analogously.
Let $\tilde{z} := \max \{ \tilde{x}, \tilde{y}\}$. Then, for all $\omega\in\Omega$:
\begin{align*}
    \tilde{x}(\omega) &\leq \max\{\tilde{x}(\omega),\tilde{y}(\omega)\}\\
    &=  \tilde{z}(\omega) \\
    &= \tilde{y}(\omega)+\max\{\tilde{x}(\omega) - \tilde{y}(\omega),0\} \\
    &\le \tilde{y}(\omega)+|\tilde{x}(\omega) - \tilde{y}(\omega)| & [\tilde{x}(\omega) - \tilde{y}(\omega)\leq |\tilde{x}(\omega) - \tilde{y}(\omega)| , 0 \leq |\tilde{x}(\omega) - \tilde{y}(\omega)|]\\
    &\leq \tilde{y}(\omega) + \max_{\omega'\in\Omega}|\tilde{x}(\omega') - \tilde{y}(\omega')|\\
    &= \tilde{y}(\omega) + \epsilon.
\end{align*}
As a result, $\tilde{x}\leq \tilde{z}\leq \tilde{y}+\epsilon$. Since $\varrho$ is a monetary risk measure, it is monotonous and translation invariant, so that
\begin{align*}
      \varrho(\tilde{x}) \le\
  \varrho(\tilde{z}) \le\
  \varrho(\tilde{y} + \epsilon)  =
  \varrho(\tilde{y}) + \epsilon,
\end{align*}
and $\varrho(\tilde{x}) - \varrho(\tilde{y}) \le \epsilon$.
\end{proof}

\subsection{Proof of Lemma \ref{lem:var-bounds}}

\begin{proof}
The inner inequalites follow from 
\begin{align*}
  \quant^+_{\ushort{f}(\alpha)}(\tilde{x})&\leq \quant^+_{\alpha}(\tilde{x}) &[\text{By monotonicity}]\\
  &= \varo_{\alpha}[\tilde{x}] \\
&= \max \left\{\tau \in \bar{\Real} \ss \Pr{  \tilde{x} <  \tau  } \le  \alpha\right\}\\
&\leq \sup \left\{\tau\in\bar{\Real} \ss \Pr{  \tilde{x} <  \tau  } <  \bar{f}(\alpha)\right\}\\
&=\quant^-_{\bar{f}(\alpha)}(\tilde{x}).  &[\text{\cite[Appx. A.3, Def. A.24]{Follmer2016}}]
\end{align*}
The outer inequalities follow from \cref{lem:var-elicitable}.
\end{proof}

\subsection{Proof of Theorem \ref{thm:bound_v2}}
\begin{proof}
The proof is broken down into two parts. First, we address the bounding on $q\opt$, then we demonstrate the stated properties of our constructed policy.

\textbf{Step 1: Upper and lower bound on $q\opt$.} We prove by induction on $t\in[T]$ that for all $(s,\alpha,a)\in \states\times(0,1)\times\actions$:
\[t\ushort{R}\leq \ushort{q}^{\mathrm{u}}_{t}(s,\alpha,a)\leq q_{t}\opt(s,\alpha,a) \leq\bar{q}^{\mathrm{u}}_{t}(s,\alpha,a)\leq t\bar{R},\]
or more succinctly, that $t\ushort{R}\leq \ushort{q}^{\mathrm{u}}_{t}\leq q_{t}\opt \leq\bar{q}^{\mathrm{u}}_{t}\leq t\bar{R}$.
At $t=0$, the bounds are obtained by definition given that $\bar{q}^{\mathrm{u}}_0=\ushort{q}^{\mathrm{u}}_0=q_0\opt=0$. 
Assume that the statement holds for some $t\in[T-1]$
and let's check if the proposition is preserved at $t+1$.

\textit{Case 1:  $\bar{f}(\alpha)<1$. } 
For all $\alpha\in(0,1)$ such that $\bar{f}(\alpha)<1$, we have:
\begin{align*}
q_{t+1}\opt(s,\alpha,a)&=\varo_{\alpha}^{a,s}[r(s,a) + \gamma \cdot  \max_{a'\in\actions}q_{t}\opt( \tilde{s}_1, \tilde{u},a')] &&[\text{By \cref{thm:q-optimal:new}}]\\
&=\quant_\alpha^{+\,a,s}[r(s,a) + \gamma \cdot  \max_{a'\in\actions}q_{t}\opt( \tilde{s}_1, \tilde{u},a')] &&[\text{By definition of $\varo_\alpha$}]\\    
&\leq\quant_\alpha^{+\,a,s}[r(s,a) + \gamma \cdot  \max_{a'\in\actions}\bar{q}^{\mathrm{u}}_{t}( \tilde{s}_1, \tilde{u},a')]  &&[\text{By inductive assumption}]\\
    &\leq\quant_{\bar{f}(\alpha)}^{-\,a,s}[r(s,a) + \gamma \cdot  \max_{a'\in\actions}\bar{q}^{\mathrm{u}}_{t}( \tilde{s}_1, \tilde{u},a')]&&[\text{By \cref{lem:var-bounds}}]
\end{align*}
Finally, we exploit the elicitability of $\quant_{\bar{f}(\alpha)}$ (see \cref{lem:var-elicitable}), due to 
the random variable $r(s,a)+\gamma\cdot \max_{a'\in \actions} \bar{q}^{\mathrm{u}}_{t}( s', \tilde{u},a')$ being supported on the interval 
$[(t+1)\ushort{R},(t+1)\bar{R}]$, to obtain: 
\begin{align*}
    q_{t+1}\opt(s,\alpha,a)&\leq\quant_{\bar{f}(\alpha)}^{-\,a,s}[r(s,a) + \gamma \cdot  \max_{a'\in\actions}\bar{q}^{\mathrm{u}}_{t}( \tilde{s}_1, \tilde{u},a')]\\
    &=\min \argmin_{q} \E^{a,s}\left[\ell_{\bar{f}(\alpha)}\left(r(s,a)+\gamma\cdot \max_{a'\in \actions} \bar{q}^{\mathrm{u}}_{t}( \tilde{s}_1, \tilde{u},a')-q\right)\right]\\
    &=\min \left\{ (\mathcal{B}^{\bar{f}}_{\mathrm{u}} \bar{q}_{t}^\mathrm{u})(s,\alpha, a) \right\} \\
    &\le \bar{q}_{t+1}^\mathrm{u}(s,\alpha,a),
\end{align*}
where the last inequality is due to the fact that $\bar{q}_{t+1}^\mathrm{u}\in (\mathcal{B}^{\bar{f}}_{\mathrm{u}} \bar{q}_{t}^\mathrm{u})$ by construction:
\begin{equation} \label{eq:q-learning-dp}
\bar{q}_{t+1}^{\mathrm{u}} \in (\mathcal{B}^{\bar{f}}_{\mathrm{u}} \bar{q}^{\mathrm{u}}_t),
\quad \ushort{q}^{\mathrm{u}}_{t+1} \in (\mathcal{B}^\ushort{f}_{\mathrm{u}} \ushort{q}^{\mathrm{u}}_t),\, \forall\, t \in [T-1]. 
\end{equation}

Moreover, we have that:
\[\argmin_{q} \E^{a,s}\left[\ell_{\bar{f}(\alpha)}\left(r(s,a)+\gamma\cdot \max_{a'\in \actions} \bar{q}^{\mathrm{u}}_{t}( \tilde{s}_1, \tilde{u},a')-q\right)\right]   \subset \left(-\infty, (t+1)\bar{R}\right]\]
since it is a quantile of 
$r(s,a)+\gamma\cdot \max_{a'\in \actions} \bar{q}^{\mathrm{u}}_{t}( s', \tilde{u},a')\leq (t+1)\bar{R}$. This confirms the statement for $\bar{f}(\alpha) < 1$.
\textit{Case 2: $\bar{f}(\alpha)=1$.} We instead rely on the following inequalities:
\begin{align*}
    q_{t+1}\opt(s,\alpha,a)&=\max_{\pi\in \PiHR}\varo^{(a,\pi_{1:t}),s}_{\alpha}\left[\sum_{k=0}^{t} \gamma^{k} r(\tilde{s}_{k},\tilde{a}_{k})\right] \\
    &\leq \sum_{k=0}^{t} \gamma^{k} \bar{R} 
    \leq (t+1) \bar{R} = \bar{R}+\max_{s\in\states,a'\in\actions}\bar{q}^\mathrm{u}_{t}(s,1,a')=\bar{q}^\mathrm{u}_{t+1}(s,\alpha,a).
\end{align*}

A similar set of arguments applies for the lower bound. Namely,
\begin{align*}
q_t\opt(s,\alpha,a)&=\quant_\alpha^{+\,a,s}[r(s,a) + \gamma \cdot  \max_{a'\in\actions}q_{t-1}\opt( \tilde{s}_1, \tilde{u},a')]&&[\text{By \cref{thm:q-optimal:new}}]\\
    &\geq\quant_\alpha^{+\,a,s}[r(s,a) + \gamma \cdot  \max_{a'\in\actions}\ushort{q}^{\mathrm{u}}_{t-1}( \tilde{s}_1, \tilde{u},a')]&&[\text{By inductive assumption}]\\
    &\geq\quant_{\ushort{f}(\alpha)}^{+\,a,s}[r(s,a) + \gamma \cdot  \max_{a'\in\actions}\ushort{q}^{\mathrm{u}}_{t-1}( \tilde{s}_1, \tilde{u},a')]&&[\text{By \cref{lem:var-bounds}}]\,,
\end{align*}
Similarly, we exploit the elicitability and the fact that $\ushort{q}^{\mathrm{u}}_{t+1} \in (\mathcal{B}^\ushort{f}_{\mathrm{u}} \ushort{q}^{\mathrm{u}}_t)$ to write
\begin{align*}
    q_{t+1}\opt(s,\alpha,a)&\geq\quant_{\ushort{f}(\alpha)}^{+\,a,s}[r(s,a) + \gamma \cdot  \max_{a'\in\actions}\ushort{q}^{\mathrm{u}}_{t}( \tilde{s}_1, \tilde{u},a')] \\ 
    &=\max \argmin_{q} \E^{a,s}\left[\ell_{\ushort{f}(\alpha)}\left(r(s,a)+\gamma\cdot \max_{a'\in \actions} \ushort{q}^{\mathrm{u}}_{t}( \tilde{s}_1, \tilde{u},a')-q\right)\right]\\
    &=\max \left\{ (\mathcal{B}^{\ushort{f}}_{\mathrm{u}} \ushort{q}_{t}^\mathrm{u})(s,\alpha, a) \right\} \\
    &\geq \ushort{q}_{t+1}^\mathrm{u}(s,\alpha,a) ,
\end{align*}
where the difference lies in the second inequality, i.e. one needs to exploit the monotonicity of $\alpha\mapsto\quant_\alpha^+(\tilde{x})$. The case where $\ushort{f}(\alpha)=0$ also follows naturally:
\[q_{t+1}\opt(s,\alpha,a)=\max_{\pi\in \PiHR}\varo^{(a,\pi_{1:t}),s}_{\alpha}\left[\sum_{k=0}^{t} \gamma^{k} r(\tilde{s}_{k},\tilde{a}_{k})\right]\geq \sum_{k=0}^{t} \gamma^{k} \ushort{R} \geq (t+1) \ushort{R} = \ushort{q}^\mathrm{u}_{t+1}(s,\alpha,a).\]

This completes our inductive proposition:
$t\ushort{R}\leq \ushort{q}^{\mathrm{u}}_{t}\leq q_{t}\opt\leq\bar{q}^{\mathrm{u}}_{t}\leq t\bar{R}.$

\textbf{Step 2: Bounds on the performance of $\upi_k(h_k)$.}
We start by proving that 
\[
  \max_{a\in\actions}\ushort{q}^{\mathrm{u}}_T(\inits,\alpha_0,a)\leq \varo_{\alpha_0}^{\upi,\inits}\left[\sum_{k=0}^{T-1} \gamma^{k} r(\tilde{s}_{k},\tilde{a}_{k})\right].
  \] 
  The rest follows based on the bounding properties of $\bar{q}^{\mathrm{u}}_T$ and $\ushort{q}^{\mathrm{u}}_T$.

    First, we can confirm that $\ushort{o}^k$ composed using $
    \ushort{o}_{s'}^k:=\ualpha^{\mathrm{u}}_{k+1}(\langle h_k,a,s'\rangle)$ is in $\mathcal{O}_{sa}(\ualpha^{\mathrm{u}}_k(h_k))$. To do so, we exploit the fact that:
    \begin{align*}
\ushort{q}^{\mathrm{u}}_{T-k}(s_k,\ualpha^{\mathrm{u}}_k(h_k),a)&\stackrel{\eqref{eq:q-learning-dp}}{\in} \argmin_{q\in\Real} \E^{a,s_k}\left[\ell_{\ushort{f}(\ualpha^{\mathrm{u}}_k(h_k))}\left(r(s_k,a) + \gamma \cdot   \max_{a'\in \actions} \ushort{q}^{\mathrm{u}}_{T-k-1}( \tilde{s}_1, \tilde{u},a')-q\right)\right]\\
&\stackrel{\ref{lem:var-bounds}}{\leq} \varo_{\ushort{f}(\ualpha^{\mathrm{u}}_k(h_k))}^{a,s_k}[r(s_k,a)+\gamma\max_{a'\in\actions}\ushort{q}^{\mathrm{u}}_{T-k-1}(\tilde{s}_1,\tilde{u}_1,a')]\\
&\stackrel{\ref{lem:var-bounds}}{\leq}\varo_{\ualpha^{\mathrm{u}}_k(h_k)}^{a,s_k}[r(s_k,a)+\gamma\max_{a'\in\actions}\ushort{q}^{\mathrm{u}}_{T-k-1}(\tilde{s}_1,\tilde{u}_1,a')]\\
&\stackrel{\ref{thm:var:decomp2}}{=}\max_{o\in\mathcal{O}_{s_k a}(\ualpha^{\mathrm{u}}_k(h_k))} \min_{s'\in\states} r(s_k,a)+\gamma\varo_{o_{s'}}[\max_{a'\in\actions}\ushort{q}^{\mathrm{u}}_{T-k-1}(s',\tilde{u}_1,a')]\\
&= \min_{s'\in\states} r(s_k,a)+\gamma\varo_{o_{s'}\opt}[\max_{a'\in\actions}\ushort{q}^{\mathrm{u}}_{T-k-1}(s',\tilde{u}_1,a')]\\
&\stackrel{\ref{thm:maxVaR}}{=}      \min_{s'\in\states} r(s_k,a)+\gamma\max_{a'\in\actions}\varo_{o_{s'}\opt}[\ushort{q}^{\mathrm{u}}_{T-k-1}(s',\tilde{u}_1,a')]\\
&\stackrel{\ref{thm:varRcNondecrease}}{=}      \min_{s'\in\states} r(s_k,a)+\gamma\max_{a'\in\actions}\ushort{q}^{\mathrm{u}}_{T-k-1}(s',o_{s'}\opt,a') \\
&\leq r(s_k,a)+\gamma\max_{a'\in\actions}\ushort{q}^{\mathrm{u}}_{T-k-1}(s',o_{s'}\opt,a') \quad,\forall s'\in\states
\end{align*}
for optimal $o\opt\in\mathcal{O}_{s_k a}(\ualpha^{\mathrm{u}}_k(h_k))$.
This implies that 
\[\max_{a'\in\actions} \ushort{q}^{\mathrm{u}}_{T-k-1}(s',o\opt_{s'},a')  \geq \frac{\ushort{q}^{\mathrm{u}}_{T-k}(s_k,\ualpha_{k}^{\mathrm{u}}(h_{k}),a) - r(s_k,a)}{\gamma}\]
By construction of $\ushort{o}^k$, we have that $\ushort{o}_{s'}^k\leq o_{s'}\opt$, hence 
\begin{align*}
\sum_{s'\in\states}\ushort{o}_{s'}^k p(s_k,a,s')
&\leq 
\sum_{s'\in\states}o_{s'}\opt p(s_k,a,s')\leq \ualpha_k^{\mathrm{u}}(h_k)    .
\end{align*}

Let us now construct a series of policies $\upi^{t}_k$ defined as:
\[\upi_k^{t}(h_k)\in \argmax_{a\in \actions} \ushort{q}^\mathrm{u}_{t-k}(s_k,\ushort{\alpha}^{\mathrm{u},t}_k(h_k),a), \] where $\ushort{\alpha}^{\mathrm{u},t-1}_0(s') := \ushort{\alpha}^{\mathrm{u},t}_1(\langle s,a,s'\rangle)$ and $\ushort{\pi}^{t-1}_{k-1}(h_{1:k}) := \ushort{\pi}^{t}_{k}(h_{k})$ are defined without loss of generality for all $k$, $\alpha$, $t$, $h_k$. The superscript $t$ represents the horizon of the sub-problem, we may omit the superscript when it represent the final objective decision horizon $T$ as $\ushort{\pi} = \ushort{\pi}^T$ and $\ushort{\alpha}^{\mathrm{u}} = \ushort{\alpha}^{\mathrm{u},T}$. 

Letting 
\begin{equation} \label{eq:var-of-pi}
    v_t^{\pi}(s,\alpha):=\varo_\alpha^{\pi,s}[\sum_{k=0}^{t-1}r(\tilde{s}_k,\tilde{a}_k)]
\end{equation} and $\ushort{v}^\mathrm{u}_t(s,\alpha):=\max_{a\in\actions}\ushort{q}^\mathrm{u}_t(s,\alpha,a)$, we now prove by induction on $t$ that $v_t^{\upi^{t}}(s,\alpha)\geq\ushort{v}^\mathrm{u}_t(s,\alpha) ~\forall t \in [T]$. This is obviously the case at $t=0$ since $v_0^{\pi}(s,\alpha)=0=\ushort{q}^\mathrm{u}_0(s,\alpha,a)$ for all $s$, $a$, $\alpha\in(0,1)$, and $\pi$. Now assuming that for some $t\in 1:T$, $v_{t-1}^{\upi^{t-1}}(s,\alpha)\geq\ushort{v}^\mathrm{u}_{t-1}(s,\alpha)$, letting  $\ushort{\alpha}^{\mathrm{u},T}_0(s) = \alpha$ and $a = \ushort{\pi}^t_0(s)$ for $v_t^{\upi^{t}}(s,\alpha)$ we obtain:
\begin{align*}
    v_t^{\upi^{t}}(s,\alpha)&\stackrel{\eqref{eq:var-of-pi}}{=}\varo_\alpha^{\upi^{t},s}(\sum_{k=0}^{t-1} \gamma^k r(\tilde{s}_k,\tilde{a}_k))\\
    &\stackrel{\ref{thm:var:decomp2}}{=}\max_{o\in\mathcal{O}_{sa}(
\alpha)}\min_{s'\in\states} r(s,a)+\gamma\varo_{o_{s'}}^{s}[\sum_{k=1}^{t-1}\gamma^{k-1}r(\tilde{s}_k,\pi^{t}_k(\tilde{h}_k)|\tilde{a}_0=a,\tilde{s}_1=s']
\end{align*}
Let rewrite the second term as follows
\begin{align*}
    &\varo_{o_{s'}}^{s}[\sum_{k=1}^{t-1}\gamma^{k-1}r(\tilde{s}_k,\upi^{t}_k(\tilde{h}_k))|\tilde{a}_0=a,\tilde{s}_1=s']\\
    &=\varo_{o_{s'}}^{s}[\sum_{k=1}^{t-1}\gamma^{k-1}r(\tilde{s}_k,\upi^{t-1}_{k-1}(h_{1:k} \mid \ushort{\alpha}_0^{\mathrm{u},t-1}(\tilde{s}_1) = \ushort{\alpha}_1^{\mathrm{u},t}(\langle s,\tilde{a}_0,\tilde{s}_1 \rangle )))~|~\tilde{a}_0=a,\tilde{s}_1=s']\\
    &=\varo_{o_{s'}}^{s'}[\sum_{k'=0}^{t-2}\gamma^{k'}r(\tilde{s}_{k'},\upi^{t-1}_{k'}(h_{k'}))]\\
    &=\varo_{o_{s'}}^{\upi^{t-1},s'}[\sum_{k'=0}^{t-2}\gamma^{k'}r(\tilde{s}_{k'},\tilde{a}_{k'})]
\end{align*}
Now we have
\begin{align*}
v_t^{\upi^{t}}(s,\alpha)&=\max_{o\in\mathcal{O}_{sa}(
\alpha)}\min_{s'\in\states} r(s,a)+\varo_{o_{s'}}^{\upi^{t-1},s'}[\sum_{k'=0}^{t-2}\gamma^{k'}r(\tilde{s}_{k'},\tilde{a}_{k'})] && [\text{From derivation above}]\\
    &\geq\min_{s'\in\states} r(s,a)+\varo_{\ualpha^{\mathrm{u},t}_1(\langle s,a,s'\rangle)}^{\upi^{t-1},s'}[\sum_{k'=0}^{t-2}\gamma^{k'}r(\tilde{s}_{k'},\tilde{a}_{k'})] && [\text{Property of max}]\\
    &= \min_{s'\in\states} r(s,a)+\gamma \ushort{v}^{\upi^{t-1}}_{t-1}(s',\ualpha^{\mathrm{u},t}_1(\langle s,a,s'\rangle))&& [\text{By \cref{eq:var-of-pi}}]\\
    &\geq \min_{s'\in\states} r(s,a)+\gamma \ushort{v}^\mathrm{u}_{t-1}(s',\ualpha^{\mathrm{u},t}_1(\langle s,a,s'\rangle))&& [\text{By Inductive assumption}]\\
    &\geq \min_{s'\in\states} \ushort{v}^\mathrm{u}_{t} (s,\alpha)=\ushort{v}^\mathrm{u}_{t}(s,\alpha)&& [\text{By \cref{eq:q-learning-dp}}]. \\
\end{align*}
We could conclude with induction that:
\[\varo_{\initalpha}^{\upi,\inits}\left[\sum_{k=0}^{T-1} \gamma^{k} r(\tilde{s}_{k},\tilde{a}_{k})\right]=v_T^{\upi^T}(s_0,\initalpha)\geq \ushort{v}^\mathrm{u}_T(s_0,\initalpha)= \max_{a\in\actions}\ushort{q}^\mathrm{u}_T(\inits,\initalpha,a).\]
\end{proof}

\subsection{Proof of Proposition \ref{lem:unif_discretize_bellman}}

\begin{proof}
This result follows by induction from showing that the constructed $\ushort{q}$ satisfies the conditions identified in \cref{thm:bound_v2} under $\ushort{f}$ as defined in \cref{exm:discrete-approx}. Naturally, the initial condition is satisfied:
\[
  \ushort{q}_0(s,\alpha,a)= \qud_0(s, J \cdot \ushort{f}(\alpha) ,a) = 0.
\]
To prove the inductive step tor any $t\in 1{:}T-1$, we have that if $\ushort{f}(\alpha)=0$ then :
\[    \ushort{q}_{t+1}(s,\alpha,a)=\qud_{t+1}(s, 0 ,a) = \ushort{R}+\min_{s\in\states,a\in\actions}\qud_t(s,0,a) = \ushort{R}+\min_{s\in\states,a\in\actions}\ushort{q}_t(s,0,a)\]
since $\ushort{f}(0)=0$. If $\ushort{f}(\alpha)= j/J$ with $j\geq 1$, then 
\begin{align*}
    \ushort{q}_{t+1}(s,\alpha,a)&=\qud_{t+1}(s, J \cdot \ushort{f}(\alpha) ,a)\\
&\in \argmin_{x\in \Real} \frac{1}{J} \sum_{j'=0}^{J-1}  \E^{a,s}\left[\ell_{\nicefrac{j}{J}}\left(r(s,a)+\gamma\cdot \max_{a'\in \actions}  \qud_t(\tilde{s}_1,j',a')-x\right)\right]\\
&= \argmin_{x\in \Real} \frac{1}{J} \sum_{j'=0}^{J-1}  \E^{a,s}\left[\ell_{\nicefrac{j}{J}}\left(r(s,a)+\gamma\cdot \max_{a'\in \actions}  \ushort{q}^\mathrm{u}_t(\tilde{s}_1,j'/J,a')-x\right)\right]\\
&= \argmin_{x\in \Real}\E^{a,s}\left[\ell_{\nicefrac{j}{J}}\left(r(s,a)+\gamma\cdot \max_{a'\in \actions}  \ushort{q}_t^\mathrm{u}(\tilde{s}_1,\ushort{f}(\tilde{u}),a')-x\right)\right]\\
&= \mathcal{B}^\ushort{f}_{\mathrm{u}} \ushort{q}_t.
\end{align*}
Moreover, since $\ushort{f}$ is right-continuous and non-decreasing $\ushort{q}$ is right-continuous and non-decreasing by construction.
\end{proof}

\subsection{Why not Huber's Loss} \label{sec:why-not-huber}

A common differentiable function used in quantile regression is the \emph{Huber's} quantile regression loss function and is commonly defined as (and in general differs from the Moreau envelope of the quantile loss function):
\begin{equation} \label{eq:hubers}
  e^{\kappa}_\alpha(\delta) \; := \; 
\begin{cases}
        -(1-\alpha)(\delta+\kappa) + \frac{1}{2}(1-\alpha)\kappa  & \text{if} ~\delta < -\kappa\\
        (1-\alpha)  \left( \frac{\delta^2}{2\kappa} \right)& \text{if} ~\delta \in [-\kappa, 0] \\
         \alpha   \left( \frac{\delta^2}{2\kappa} \right)& \text{if} ~\delta \in  (0, \kappa)\\
        \alpha(\delta-\kappa) + \frac{1}{2} \alpha \kappa  & \text{if}~\delta>\kappa
    \end{cases} 
  \end{equation}
Although some definitions of Huber's loss are scaled by a positive constant compared with \eqref{eq:hubers}, such scaling does not affect the minimizer.  

The next proposition demonstrates that it is not sufficient to consider the popular Huber's loss function in place of the quantile loss function. Suppose that one defines a risk measure as
\begin{equation} \label{eq:huber-loss-risk}
\xi(\tilde{x}) := \argmin_{m\in\mathbb{R}} \mathbb{E}[e_{\alpha}^\kappa(\tilde{x}-m)] 
\end{equation}
where Huber's loss $e^{\kappa}_{\alpha}$ is defined in \eqref{eq:hubers}.
\begin{proposition}\label{prop:must-strongly-convex}
The minimization problem in \eqref{eq:huber-loss-risk} may not have a unique solution.
\end{proposition}
\begin{proof}
Consider a risk level of $\alpha=0.5$, $\kappa=0.5$, and random variable $\tilde{x}$ such that $\P{ \tilde{x} = -1 } = \P{ \tilde{x} = 1 } = 1 / 2$. Then the risk measure defined in \eqref{eq:huber-loss-risk} becomes
\begin{equation*}
  \xi(\tilde{x})
  \; :=\;  \argmin_{m\in \Real} \mathbb{E}[e_{\alpha}^\kappa(\tilde{x}-m)]
  \; =\;  \argmin_{m\in \Real}\E \left[ e_{\nicefrac{1}{2}}^{\nicefrac{1}{2}} \left(\tilde{x}  - m \right)  \right] .
\end{equation*}
Simple algebraic manipulation shows that the minimization above does not have a unique optimal solution and instead $m\in [-\nicefrac{1}{2},\nicefrac{1}{2}]$ is optimal, as \cref{fig:strong-convex-counter} illustrates.
\end{proof}

\begin{figure}
  \centering
  \includegraphics[width=0.45\linewidth]{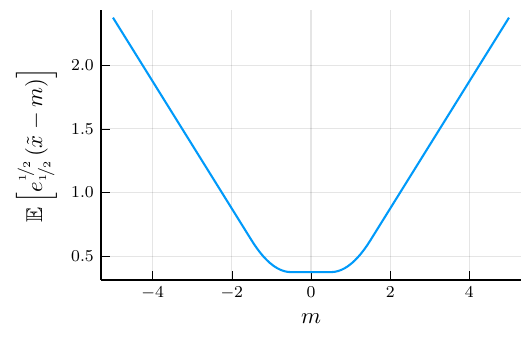}
  \caption{A example used to prove the non-uniqueness of an optimal solution in \cref{prop:must-strongly-convex}.}
  \label{fig:strong-convex-counter}
\end{figure}

\subsection{Proof of Lemma \ref{lem:elicitable-objective-lips-strong}}
\label{sec:elicitableObjectiveLipStrProof}
Before proving \cref{lem:elicitable-objective-lips-strong}, first define the notions it invokes. 

\begin{definition}
Given $L,\mu\in\Real_{++}$, a differentiable function $f\colon \Real \to  \Real$ is $\mu$-\emph{strongly convex} if
\[
  f(x) \; \ge\;
  f(y) + f'(y)(x-y) +\frac{1}{2} \mu (x - y)^2, \qquad \forall x,y\in \Real .
\]
It has an $L$-\emph{Lipschitz continuous derivative} if 
\[
    \left| f'(x) - f'(y)  \right| \;\le\;  L |x - y|, \qquad \forall x,y\in \Real.
\]
We always have $L \ge \mu$ \cite[Thm.~2.1.10]{Nesterov2018a}.
\end{definition}

The next lemma focuses on the properties of the loss function, after which we can prove \cref{lem:elicitable-objective-lips-strong}.

\begin{lemma} \label{lem:bounded_hessian}
Suppose that $\kappa\in(0,1]$. Then the function $\ell^{\kappa}_{\alpha}$ is continuously differentiable and satisfies for all $\delta_1,\delta_2 \in \Real$ that
\begin{equation} \label{eq:derivative-bounds}
  \displaystyle  \min \{\alpha,1-\alpha\} ~ \kappa
  \;\leq \;  \frac{|{\partial\ell^{\kappa}_\alpha}(\delta_2)-{\partial\ell^{\kappa}_\alpha}(\delta_1)|}{|\delta_2-\delta_1|}  
  \;\leq \;
  \frac{ \max \{\alpha,1-\alpha\} }{\kappa}.    
\end{equation}
Moreover, the function $\ell^{\kappa}_{\alpha }$ is \emph{strongly convex} with constant $\min \left\{ \alpha , 1-\alpha  \right\} \kappa$ and has a Lipschitz-continuous derivative with constant $\max \left\{ \alpha , 1-\alpha  \right\} \kappa^{-1}$. 
\end{lemma}
\begin{proof}
The inequality in \eqref{eq:derivative-bounds} follows from the fact that ${\partial\ell^{\kappa}_{\alpha}}$, derived in \eqref{eq:delageLoss2}, is a piecewise linear function and $\kappa  \le 1$. Then, the function $\ell_{\alpha }^{\kappa}$ is strongly convex from the strong monotonicity of its derivative in \eqref{eq:delageLoss2} by \cite[Exercise~12.59]{Rockafellar2009}. The function $\ell_{\alpha }^{\kappa}$ has a Lipschitz continous derivative immediately by the upper bound in \eqref{eq:derivative-bounds}.
\end{proof}

We now prove \cref{lem:elicitable-objective-lips-strong}.

\begin{proof}
By \cref{lem:bounded_hessian} and \cite[Lem.~2.1.6]{Nesterov2018a}, we have that the objective in \eqref{eq:elicitable-definition} is also strongly convex with constant $\mu$ given in the lemma. By \cref{lem:elicitable-objective-lips-strong} and the linearity of the expectation operator, the derivative of \eqref{eq:elicitable-definition} is also Lipschitz continuous with constant $L$ given in the lemma. 
\end{proof}

\subsection{Proof of Theorem \ref{thm:convergenceQlearning}}\label{sec:convergenceProof}

\subsubsection{Properties of Soft-Quantile Measure}

\begin{definition} \label{def:shortfall}
A risk measure $\hat{\quant}_{\alpha}^\kappa$ is a \emph{shortfall risk measure} (see \cite{follmerSchied:shortfalRisk,weber2006}) defined as:
    \begin{equation} \label{eq:shorfall-definition}
      \hat{\quant}_{\alpha}^\kappa(\tilde{x}):=\sup\{m\in\mathbb{R} : \mathbb{E}[{\partial\ell_\alpha^\kappa}(\tilde{x}-m)]\geq 0\},
\end{equation}
where ${\partial\ell_\alpha^\kappa}(\cdot)$ is  defined in \eqref{eq:delageLoss2}.
\end{definition}

\begin{lemma} \label{lem:softquant-risk-sortfall}
The measure $\hat{\quant}_{\alpha}^\kappa$ satisfies monotonicity, translation invariance, and is elicitable as:
\begin{equation} \label{eq:elicitable-definition}
    \hat{\quant}_{\alpha}^\kappa(\tilde{x}) := \argmin_{m\in\mathbb{R}} \mathbb{E}[\ell_\alpha^\kappa(\tilde{x}-m)]
\end{equation}
with $\ell_\alpha^\kappa$ as defined in \eqref{eq:delageLoss1}.
\end{lemma}
\begin{proof}
Monotonicity and translation invariance follow naturally from the properties of shortfall risk measures as defined in \cite{follmerSchied:shortfalRisk,weber2006}, after confirming that ${\partial\ell_\alpha^\kappa}(\delta)$ is increasing and non-constant. Elicitability follows from Theorem 4.3 in \cite{Bellini2015} after confirming that $ {\partial\ell_\alpha^\kappa}(\delta)$ is strictly increasing and left continuous.
\end{proof}

\subsubsection{Standard Operator Convergence Results}\label{sec:stand-oper-conv}

Our convergence analysis follows the framework presented in Section 4  of \cite{Bertsekas1996}, which we summarize in this section.  The Q-learning algorithm, consider the following iteration for some random sequence $\tilde{z}_i \colon \Omega \to  \Real^{\mathcal{N}}$ where $\mathcal{N} = \{1,\dots,n\}$ defined as
\begin{equation} \label{eq:qiteration-update}
  \begin{aligned}
  \tilde{z}_{i+1}(b)
  &\; =\;
  (1 - \tilde{\theta}_i(b)) \cdot \tilde{z}_i(b) + \tilde{\theta}_i(b) \cdot  ((H \tilde{z}_i)(b) + \tilde{\phi}_i(b)),
    \qquad i = 0, 1, \dots , \\
  &\; =\;
   \tilde{z}_i(b) + \tilde{\theta}_i(b) \cdot  ((H \tilde{z}_i)(b) + \tilde{\phi}_i(b) - \tilde{z}_i(b)),
  \qquad i = 0, 1, \dots ,
  \end{aligned}
\end{equation}
for all $b\in\mathcal{N}$, where $H\colon \Real^\mathcal{N} \to  \Real^\mathcal{N}$ is some possibly non-linear operator, $\tilde{\theta}_i \colon \Omega \to  \Real_{+}$
  is a step size, and $\tilde{\phi}_i \colon \Omega \to  \Real^{\mathcal{N}}$ is some random noise sequence. The \emph{random} history $\mathcal{F}_i$ at iteration $i = 1, \dots $ is denoted by
\[
 \mathcal{F}_i = \left( \tilde{z}_0, \dots , \tilde{z}_i, \tilde{\phi}_0, \dots , \tilde{\phi}_{i-1}, \tilde{\theta}_0, \dots , \tilde{\theta}_i \right).
\]

To study the convergence of the algorithm, we may also need to define a weighted maximum norm for $x\in \Real^\mathcal{N}$ and weights $w \in \Real^\mathcal{N}_{++}$, i.e. $w \in \Real^\mathcal{N}$ with $w(b)>0$ for all $b\in\mathcal{N}$, as
\[
 \| x\|_w \;:=\; \max_{b\in\mathcal{N}} \frac{|x(b)|}{w(b)}.  
\]
The weighted maximum norm is useful when analyzing the convergence of non-discounted MDPs. Its importance is in the fact that a non-negative matrix with a sub-unit spectral radius is a contraction in the weighted norm, but may not be a contraction in the plain maximum norm.

\begin{definition}[Star-contraction] \label{def:pseudo-contraction}
An operator $H\colon \Real^\mathcal{N} \to \Real^\mathcal{N}$ is a \emph{weighted maximum norm star-contraction} if there exist $z\opt \in \Real^\mathcal{N}$, $w\in \Real_{++}^\mathcal{N}$, $\chi\in [0,1)$ such that  
  \[
    \| H z - z\opt \|_w
    \; \le \;
    \chi \cdot \| z - z\opt  \|_w.
  \]
\end{definition}

Note that the original name for star-contraction is pseudo-contraction. We use the term star-contraction because of its close resemblance to star-convexity. 

The following assumption on the noise of the stochastic process will be needed to ensure the convergence of our algorithm. 
\begin{assumption}[Assm. 4.3~\cite{Bertsekas1996}] \label{assumption:expected_convergence}
  The noise terms in \eqref{eq:qiteration-update} satisfy for each $i = 1, 2, \dots$ that:
\begin{enumerate}
\item[(1)] Random errors are conditionally unbiased, a.s.:
 \[
   \E[\tilde{\phi}_i(b) \mid \mathcal{F}_i] = 0, \qquad \forall b \in 1{:}n,
  \]
\item[(2)] There exists a norm $\| \cdot \|$ on $\Real^\mathcal{N}$ and $c,g\in \Real$ such that 
    \[
    \E[\tilde{\phi}_i(b)^2 \mid \mathcal{F}_i] \le c + g \cdot \| \tilde{z}_i \|^2 , \qquad
    \forall b \in 1{:}n, \quad \text{a.s.}
    \]
\end{enumerate}
 \end{assumption}

 \begin{proposition}\label{thm:bertsekas}[Proposition 4.4~\cite{Bertsekas1996}]
   Let $\tilde{z}_i, i = 1, \dots $ be the sequence generated by the iteration in \eqref{eq:qiteration-update}. Assume that
   \begin{enumerate}
   \item The step-sizes $\tilde{\theta}_i, \forall i = 1, \dots $ satisfy almost surely that $\tilde{\theta}_i \ge 0$ and 
   \[
     \sum_{t=0}^{\infty} \tilde{\theta}_i(b) = \infty, \quad \sum_{i=0}^{\infty} \tilde{\theta}^2_i(b) < \infty, \qquad \forall b\in \mathcal{N}.
   \]
 \item The noise terms $\tilde{\phi}_i, i = 1, \dots$ satisfy \cref{assumption:expected_convergence}.
 \item The operator $H$ in \eqref{eq:qiteration-update} is a weighted maximum norm star-contraction as in \cref{def:pseudo-contraction}.
 \end{enumerate}
 Then, $\tilde{z}_i$ converges to $z\opt$, a fixed point of $H$, with probability 1:
 \[
   \Pr{\lim_{i \to \infty} \tilde{z}_i = z\opt} = 1.
 \]
\end{proposition}

\subsubsection{Operator Definitions}

We will need the following operators for any $\xi > 0$ and $\tilde{j}' \sim U([J-1])$. Let $b = (t, s, j, a)$ for
$t \in [T], s\in \mathcal{S}, j\in [J-1], a\in \mathcal{A}$:
\begin{equation} \label{eq:qlearning-operators}
\begin{aligned}
(Gq)(b)
&:=
  \begin{cases}
   q(b)-\ushort{R}\cdot t&\text{if } j = 0 \vee t = 0, \\ 
  \frac{d}{dx} \left. \E^{a,s}\left[\ell_{\nicefrac{j}{J}}^{\kappa} \left( r(s,a) + \gamma \cdot  \max_{a'\in \mathcal{A}} q(t-1, \tilde{s}_1, \tilde{j}', a') - x \right) \right] \right|_{x = q(t,s,j,a) } &\text{otherwise}, \\
    \end{cases} \\
&=\begin{cases}
q(b)-\ushort{R}\cdot t& \mbox{if }j=0 \vee t = 0,\\
-\E^{a,s}\left[{\partial\ell^{\kappa}_{\frac{j}{J}}}\left(r(s,a)+\gamma \max_{a'\in \actions} q(t-1,\tilde{s}_1,\tilde{j}',a')-q(t,s,j,a)\right) \right] & \mbox{otherwise},
\end{cases} \\ 
(G_{s'}q)(b)
&:=\begin{cases}
q(b)-\ushort{R}\cdot t& \mbox{if }j=0 \vee  t = 0,\\
- \E\left[{\partial\ell^{\kappa}_{\frac{j}{J}}}\left(r(s,a)+\gamma \max_{a'\in \actions} q(t-1,s',\tilde{j}',a')-q(t,s,j,a)\right)\right] & \mbox{otherwise},
\end{cases} \\
  H q
&:= q - \xi \cdot  G q, \\
  H_{s'}q
&:= q - \xi \cdot  G_{s'}q.
\end{aligned}
\end{equation}

Consider a random sequence of inputs $((\tilde{t}_i,\tilde{s}_i,\tilde{j}_i,\tilde{a}_i,\tilde{s}_i'),\tilde{\beta}_i,\tilde{\ushort{q}}_i)_{i=0}^\infty$ in \cref{alg:q-learning}. We can define a real-valued random variable $\tilde{\phi}$ for $t \in [T], s\in \mathcal{S}, j\in [J-1], a\in \mathcal{A}$ as 
\begin{equation} \label{eq:qlearning-noise-step}
\begin{aligned}
\tilde{\phi}_i(t,s,j,a)
&:=
\begin{cases}
(H_{\tilde{s}_i'}\tqud_i)(t,s,j,a) - (H \tqud_i)(t,s,j,a)
&\text{ if }  (\tilde{t}_i, \tilde{s}_i, \tilde{j}_i, \tilde{a}_i) = (t,s,j,a), \\
0 & \mbox{ otherwise}.
\end{cases} \\
\tilde{\theta}_i(t,s,j,a)
&:=
\begin{cases}
\frac{\tilde{\beta}_i}{\xi}
&\text{ if } (\tilde{t}_i,\tilde{s}_i,\tilde{j}_i,\tilde{a}_i)=(t,s,j,a), \\
0 &\mbox{ otherwise}.    
\end{cases}
\end{aligned}
\end{equation}

\begin{lemma} \label{lem:q-learning-update-H}
The random sequence of iterations followed by \cref{alg:q-learning} satisfies 
\begin{align*}
\tqud_0(t,s,j,a) &= t \ushort{R}, \text{ a.s. }, \\
  \tqud_{i+1}(t,s,j,a)
  &= \tqud_i(t,s,j,a) + \tilde{\theta}_i(t,s,j,a) \cdot  (H \tqud_i + \tilde{\phi}_i - \tqud_i)(t,s,j,a), \quad \forall i \in \Natural, ~ \text{a.s.},
\end{align*}
where the terms are defined in \eqref{eq:qlearning-operators} and \eqref{eq:qlearning-noise-step}.
\end{lemma}
\begin{proof}
We prove the claim by induction on $i$. The base case holds immediately from the definition. To prove the inductive case, suppose that $i \in \Natural$ and we prove the result in the following cases. 
  
\emph{Case 1a}: Suppose that $\tilde{b} = (\tilde{t}_i, \tilde{s}_i, \tilde{j}_i, \tilde{a}_i) = (t, s, j, a) = b$ and $t_i > 0 , j_i > 0$, then by algebraic manipulation:
\begin{align*}
  \tqud_{i+1}(b)
  &= \tqud_i(b) + \theta_i(b) ((H \tqud_i)(b) + \tilde{\phi}_i(b) - \tqud_i(b)) \\
  &= \tqud_i(b) + \theta_i(b) ((H \tqud_i)(b) + (H_{\tilde{s}_i'} \tqud_i)(b) - (H \tqud_i)(b) - \tqud_i(b)) \\
  &= \tqud_i(b) + \theta_i(b) ((H_{\tilde{s}_i'} \tqud_i)(b) - \tqud_i(b)) \\
  &= \tqud_i(b) + \theta_i(b) ((\tqud_i - \xi G_{\tilde{s}_i'} \tqud_i)(b) - \tqud_i(b)) \\
  &= \tqud_i(b) + \frac{\beta_i }{J}\sum_{j' \in[J-1]} \partial\ell_{\frac{j}{J}}^\kappa \Bigl(r(s, a) + \gamma  \max_{a'\in \mathcal{A}}\tqud_i(t-1, \tilde{s}_i',j',a') - \tqud_i(b) \Bigr) .  
\end{align*}

\emph{Case 1b}: Suppose that $\tilde{b} = (\tilde{t}_i, \tilde{s}_i, \tilde{j}_i, \tilde{a}_i) = (t, s, j, a) = b$ and $t_i = 0 \vee j_i = 0$, then by algebraic manipulation:
\begin{align*}
  \tqud_{i+1}(b)
  &= \tqud_i(b) + \theta_i(b) ((H \tqud_i)(b) + \tilde{\phi}_i(b) - \tqud_i(b)) \\
  &= \tqud_i(b) + \theta_i(b) ((H \tqud_i)(b) + (H_{\tilde{s}_i'} \tqud_i)(b) - (H \tqud_i)(b) - \tqud_i(b)) \\
  &= \tqud_i(b) + \theta_i(b) ((H_{\tilde{s}_i'} \tqud_i)(b) - \tqud_i(b)) \\
  &= \tqud_i(b) + \theta_i(b) ((\tqud_i - \xi G_{\tilde{s}_i'} \tqud_i)(b) - \tqud_i(b)) \\
  &= \tqud_i(b) + \theta_i(b) (\tqud_i(b) - \xi \tqud_i(b) + \xi \ushort{R} t - \tqud_i(b)) \\
  &= \tqud_i(b) - \theta_i(b) \xi ( \tqud_i(b) - \ushort{R} t ) \\
  &= \tqud_i(b) - \beta_i ( \tqud_i(b) - \ushort{R} t ) .
\end{align*}
\emph{Case 2}: Suppose that $\tilde{b} = (\tilde{t}_i, \tilde{s}_i, \tilde{j}_i, \tilde{a}_i) \neq  (t, s, j, a) = b$, then by algebraic manipulation, the algorithm does not change the q-function:
\begin{align*}
  \tqud_{i+1}(b)
  &= \tqud_i(b) + \theta_i(b) \cdot  ((H \tqud_i)(b) + \tilde{\phi}_i(b) - \tqud_i(b)) \\
  &= \tqud_i(b) + 0 \cdot  ((H \tqud_i)(b) + \tilde{\phi}_i(b) - \tqud_i(b)) \\
  &= \tqud_i(b).
\end{align*}
\end{proof}

\subsubsection{Operator H is a contraction}

We use the following weights $w\in \Real^{[ T ] \times  \states \times [ J - 1 ] \times \actions}$ with a weighted max norm to prove the contraction properties in this section:
where the weights  are defined as
\begin{equation} \label{eq:weight-definition}
  w(t,s,j,a) := 2^t, \quad
  \forall t\in [T], s\in \mathcal{S}, j\in [ J-1 ], a\in \mathcal{A}.
\end{equation}

\begin{lemma} \label{lem:sim-bellman-contraction}
The operator $\budk$ is a weighted max norm contraction for $w$ defined in \eqref{eq:weight-definition}:
\begin{equation*} 
  \|\budk x - \budk y\|_w \;\le\;  \frac{1}{2} \| x - y \|_w, \quad \forall x,y\in \Real^{\states}.
\end{equation*}
\end{lemma}
\begin{proof}
The operator $\budk$ is equivalently defined using a shortfall risk measure $\hat{\quant}_{\alpha}^{\kappa}$ from \cref{lem:softquant-risk-sortfall} for each $b = (t,s,j,a)$ as
\[
(\budk q)(b)
:= \begin{cases}
\ushort{R} \cdot t & \mbox{if} ~j=0 \vee t = 0,\\
(\hat{\quant}_{\frac{j}{J}}^{\kappa})^{a,s}\left[r(s,a)+\gamma \max_{a'\in \actions} q(t-1,\tilde{s}_1,\tilde{j}',a') \right] & \mbox{otherwise},
\end{cases} 
\]
  
We analyze the following two cases.

\emph{Case 1}: For each $t\in 1{:}T, s\in \mathcal{S}, j\in 1{:}(J-1), a\in \mathcal{A}$:
\begin{align*}
0 &\le \frac{1}{w(t,s,j,a)} \left| (\budk x)(t,s,j,a) - (\budk y)(t,s,j,a) \right|
\\
&\stepname{a}{=} \frac{1}{w(t,s,j,a)} \left| (\hat{\quant}^{\kappa}_{j / J})^{a,s}\left[ \max_{a'\in \mathcal{A}} \gamma x(t-1,\tilde{s}_1,\tilde{j}',a')\right] - (\hat{\quant}^{\kappa}_{j / J})^{a,s}\left[ \max_{a'\in \mathcal{A}}\gamma y(t-1,\tilde{s}_1,\tilde{j}',a')\right] \right| \\
&\stepname{b}{\le} \frac{1}{w(t,s,j,a)} \max_{a'\in \mathcal{A}^{\states \times [J-1]}}
   \left| (\hat{\quant}^{\kappa}_{j / J})^{a,s}\left[\gamma x(t-1,\tilde{s}_1,\tilde{j}',a'(\tilde{s}_1,\tilde{j}'))\right] - (\hat{\quant}^{\kappa}_{j / J})^{a,s}\left[\gamma y(t-1,\tilde{s}_1,\tilde{j}',a'(\tilde{s}_1, \tilde{j}'))\right] \right| \\
&\stepname{c}{\le} \frac{1}{w(t,s,j,a)}    \cdot  \max_{s'\in \mathcal{S}, j'\in [J-1], a'\in \mathcal{A}} |\gamma x(t-1, s', j', a') -\gamma y(t-1, s', j', a')| \\
&= \frac{1}{w(t,s,j,a)} \cdot    \gamma\cdot  \max_{s'\in \mathcal{S}, j'\in [J-1], a'\in \mathcal{A}} |x(t-1, s', j', a') - y(t-1, s', j', a')| \\
&\stepname{d}{\le} \max_{s'\in \mathcal{S}, j'\in [J-1], a'\in \mathcal{A}} \frac{|x(t-1, s', j', a') - y(t-1, s', j', a')|}{2 w(t-1,s',j',a')}  .
\end{align*}
In step (a), we use the translation invariance of the risk measure to cancel out $r(s,a)$. Step (b) follows by upper bounding the difference using the monotonicity of $\hat{\quant}_{\frac{j}{J}}^{\kappa}$, step (c) follows from \cref{lem:risk-nonexpansive,lem:softquant-risk-sortfall}, step (d) follows from the definition of $w$, its dependence on $t$ only, and $\gamma \le 1$. 

\emph{Case 2}: For $t = 0 \vee j = 0$:
\[
  \frac{1}{w(t,s,j,a)} \left| (\budk x)(t,s,j,a) - (\budk y)(t,s,j,a) \right| =
  \frac{1}{2^t} |\ushort{R} t - \ushort{R} t| = 0 .
\]
Then, using the equalities above and $\mathcal{K} = [T] \times \mathcal{S} \times [J-1] \times  \mathcal{A}$ we get that
\begin{align*}
  \|\budk x - \budk y\|_w
  &= \max_{b\in \mathcal{K}} \frac{1}{w(b)} \left| (\budk x)(b) - (\budk y)(b) \right| \\
  &\le \max_{t\in 1{:}T, s'\in \mathcal{S}, j'\in 1:J-1, a'\in \mathcal{A}} \frac{|x(t-1, s', j', a') - y(t-1, s', j', a')|}{2 \cdot w(t-1,s',j',a')} \\
  &\le  \max_{t\in [ T ], s'\in \mathcal{S}, j'\in [J-1], a'\in \mathcal{A}} \frac{|x(t, s', j', a') - y(t, s', j', a')|}{2 \cdot w(t,s',j',a')} \\
  &=\frac{1}{2} \| x - y \|_w.
\end{align*}
\end{proof}

The following lemma establishes that the gradient update is a convex combination of the starting value and an optimal solution for an appropriate step size. 
\begin{lemma} \label{lem:gradient-step-convex}
Suppose that $f\colon \Real \to \Real$ is a differentiable $\mu$-strongly convex function with an $L$-Lispchitz continuous gradient. Consider $x_i\in \Real $ and a gradient update for any step size $\xi\in (0,\nicefrac{1}{L}]$:
\[
 x_{i+1} := x_i - \xi \cdot  f'(x_i). 
\]
Then $\exists  l \in \left[\nicefrac{1}{L}, \nicefrac{1}{\mmu}\right]$ such that $\nicefrac{\xi}{l} \in (0,1]$ and
\[
 x_{i+1} = \left( 1- \nicefrac{\xi}{l} \right) \cdot  x_i + \nicefrac{\xi}{l} \cdot  x\opt ,
\]
where $x\opt = \argmin_{x\in \Real} f(x)$ (unique from strong convexity).
\end{lemma}
\begin{proof}
Assume that $f'(x_i) \neq 0$ hence $x_i\neq x\opt$; otherwise the result holds trivially. Then construct $l \neq  0$ as
\[
l \;:=\;  \frac{x_i - x\opt }{f'(x_i)}. 
\]
Substituting the definition of $l$ into the gradient update, we get that
\[
x_{i+1} \;:=\;  x_i - \xi \cdot  f'(x_i) = \left( 1- \nicefrac{\xi}{l} \right) x_i + \nicefrac{\xi}{l} \cdot  x\opt, 
\]
as desired.

It remains to show that $l \in \left[\nicefrac{1}{L}, \nicefrac{1}{\mmu}\right]$ and  $\nicefrac{\xi}{l} \in (0,1]$. Using that $f'(x\opt) = 0$ and strong convexity~\cite[Thm. 2.1.10]{Nesterov2018a} and Lipschitz continuity of the derivative:
  \begin{equation} \label{eq:condition-convexity}
    \frac{1}{L} | f'(x_i) |  = \frac{1}{L} | f'(x_i) - f'(x\opt) |\; \le \;
    |x_i - x\opt| \; \le \;
    \frac{1}{\mmu} | f'(x_i) - f'(x\opt) | = \frac{1}{\mmu} |f'(x_i)|.
\end{equation}

Next, we analyze two cases.

\emph{Case 1}: Suppose that $x\opt  > x_i$. Then $f'(x_i) > f'(x\opt) = 0$ because $f'$ is increasing for a strongly convex $f$, and therefore, \eqref{eq:condition-convexity} becomes
  \[
    \begin{array}{rcl}
  -\frac{1}{L} f'(x_i) \le& x\opt -x_i &\le -\frac{1}{\mmu} f'(x_i), \\
  -\frac{1}{L}  \ge& \frac{x\opt -x_i}{f'(x_i)} &\ge -\frac{1}{\mmu}, \\
  -\frac{1}{L}  \ge& -l &\ge -\frac{1}{\mmu}, \\
  \frac{1}{L}  \le& l &\le \frac{1}{\mmu}. 
    \end{array}
\]
In addition, $\xi\in (0,\nicefrac{1}{L}] \implies \nicefrac{\xi}{l} \in (0,1]$.

\emph{Case 2}: Suppose that $x\opt  < x_i$. Then $f'(x_i) > f'(x\opt) = 0$ because $f'$ is increasing for a strongly convex $f$, and therefore, \eqref{eq:condition-convexity} becomes
\[
\begin{array}{rcl}
\frac{1}{L} f'(x_i) \le&  x_i - x\opt &\le \frac{1}{\mmu} f'(x_i), \\
\frac{1}{L}  \le& \frac{x_i  - x\opt }{f'(x_i)} &\le \frac{1}{\mmu}, \\
\frac{1}{L}  \le& l &\le \frac{1}{\mmu}. 
\end{array}
\]
In addition, $\xi\in (0,\nicefrac{1}{L}] \implies \nicefrac{\xi}{l} \in (0,1]$.
\end{proof}

\begin{theorem} \label{thm:h-contraction}
A fixed point $x\opt  = \budk x\opt$ exists and satisfies $x\opt  = H x\opt$. Let $\bar{\mmu} :=  J^{-1} \kappa$,  $\bar{L} := \kappa^{-1}$, and $\xi \in (0, \min(1,\nicefrac{1}{\bar{L}}))$ in the definition of $H$. Then, $H$ is a weighted max norm star contraction:
\[   \| H x  - H x\opt \|_w  \;\le\;  \left(1 - \frac{\bar{\mmu} \xi }{2} \right)    \cdot \| x - x\opt  \|_w, \]
for $w$ defined in \eqref{eq:weight-definition}.
\end{theorem}
\begin{proof}

The fixed point $x\opt$ exists from \cref{lem:sim-bellman-contraction} and the Banach fixed point theorem. The operator $H$ takes a gradient step towards $\budk$.

\emph{Case 1}: Fix some $b = (t,s,j,a)$ with $t\in [ T ], s\in \mathcal{S}, j\in [ J-1 ], a\in \mathcal{A}$ and suppose that $t > 0$ and $j > 0$. Fix some $q$ and define
\begin{equation} \label{eq:contraction-f}
f(y) =  \E^{s,a} \left[\ell_{\nicefrac{j}{J}}^{\kappa} \left( r(s,a) + \gamma \cdot  \max_{a'\in \mathcal{A}} q(t-1, \tilde{s}_1, \tilde{j}', a') - y \right) \right] ,
\end{equation}
The function $f$ is strongly convex with Lipschitz gradient with parameters $\bar{\mmu}$ and $\bar{L}$ based on \cref{lem:elicitable-objective-lips-strong} and since
\[
  \bar{\mmu} \le \min \left\{\frac{j}{J}, 1- \frac{j}{J} \right\} \kappa
  \le  \max \left\{\frac{j}{J}, 1- \frac{j}{J} \right\} \kappa^{-1} \le \bar{L}, \forall j\in 1{:}J-1.
\]
Let $y\opt \in \arg\min_{y\in \Real} f(y)$. Then, $\exists l\in [\frac{1}{\bar{L}}, \frac{1}{\bar{\mmu}}]$ such that  
\begin{equation} \label{eq:h-convex-combo}
  \begin{aligned}
    (H q)(b)  
    &= (q - \xi G q)(b) 
      = q(b) - \xi \cdot  f'(q(b))  
      = \left(1-\nicefrac{\xi}{l}\right) \cdot q(b) + \nicefrac{\xi}{l} \cdot  y\opt  \\
    &= \left(1-\nicefrac{\xi}{l}\right) \cdot  q(b) + \nicefrac{\xi}{l} \cdot  (\budk q)(b),
  \end{aligned}
\end{equation}
from algebraic manipulation and application of \cref{lem:gradient-step-convex} to the function $f$ in \eqref{eq:contraction-f} which satisfies the requisite strong convexity and Lipschitz continuity properties.

The fixed point of $x\opt$ of $\budk$ is a fixed point of $H$ from \eqref{eq:h-convex-combo} 
\[
 (H x\opt)(b) = \left(1-\frac{\xi}{l}\right) x\opt (b) + \frac{\xi}{l} (\budk x\opt)(b)
=\left(1-\frac{\xi}{l}\right) x\opt(b) + \frac{\xi}{l} (x\opt)(b) = x\opt (b). 
\]
Finally, we get using \eqref{eq:h-convex-combo} that
\begin{align*}
|(Hx)(b) - (H x\opt) (b)|
&=| (1-\nicefrac{\xi}{l}) x(b) + \nicefrac{\xi}{l}(\budk x)(b) - ((1-\nicefrac{\xi}{l}) x\opt(b) + \nicefrac{\xi}{l}(\budk x\opt)(b))|\\
&=|(1-\nicefrac{\xi}{l})(x-x\opt)(b) + \nicefrac{\xi}{l}(\budk x - \budk x\opt)(b)|\\
&\le(1-\nicefrac{\xi}{l})|(x-x\opt)(b)| + \nicefrac{\xi}{l}|(\budk x - \budk x\opt)(b)|\\
&\le(1-\nicefrac{\xi}{l})|(x-x\opt)(b)| + \frac{1}{2}\nicefrac{\xi}{l}|( x - x\opt)(b)|\\
&=(1-\nicefrac{\xi}{2l})|(x-x\opt)(b)| .
\end{align*}
Here, we used the triangle inequality for absolute values and the fact that $x\opt $ is a fixed point of $\budk$ and \cref{lem:sim-bellman-contraction}. Hence,
\begin{equation} \label{eq:bound-changing}
\frac{1}{w(b)} |(H x - H x\opt  )(b)| \le \left(1-\frac{\bar{\mmu}\xi}{2}\right) \frac{1}{w(b)}|(x-x\opt)(b)|.
\end{equation}

\emph{Case 2}:  Fix some $b = (t,s,j,a)$ with $t\in [ T ], s\in \mathcal{S}, j\in [ J-1 ], a\in \mathcal{A}$ and suppose that $t = 0 \vee j = 0$. Then, from $x\opt(b) = \ushort{R}\cdot t$, we have that if $\xi<1$:
\begin{align*}
  |(Hx)(b) - (H x\opt) (b)|
  &= |x(b) - \xi (x(b) - \ushort{R} \cdot t) - \ushort{R} \cdot t| \\
  &= (1-\xi) |x(b) - \ushort{R} \cdot t| 
  = (1-\xi) |x(b) - x\opt(b)|.
\end{align*}
 Hence,
\begin{equation} \label{eq:bound-changing-2}
\frac{1}{w(b)} |(H x - H x\opt  )(b)| \le \left(1-\xi \right) \frac{1}{w(b)}|(x-x\opt)(b)| \le \left(1-\bar{\mmu}\xi/2\right) \frac{1}{w(b)}|(x-x\opt)(b)|
\end{equation}
since $\bar{\mmu}\xi/2\leq \kappa J^{-1}\xi/2\leq \kappa \xi \le \xi$ because $\kappa \le 1$.

\emph{Conclusion}: Putting \eqref{eq:bound-changing} and \eqref{eq:bound-changing-2} together with the definition of the weighted norm, we get the desired star contraction rate. 
\end{proof}

\subsubsection{Noise Properties}

The history $\mathcal{F}_i$ at an iteration $i \in \Natural$ is defined as
\begin{equation} \label{eq:history-qlearning}
  \mathcal{F}_i :=
  \left( \tqud_0, \dots , \tqud_i, \tilde{\phi}_0, \dots , \tilde{\phi}_{i-1}, \tilde{\theta}_0, \dots , \tilde{\theta}_i \right).
\end{equation}

Recall from \cref{ass:assTransitions} that 
\[
  \mathcal{G}_{i} :=
  (\tilde{\beta}_l,(\tilde{t}_l, \tilde{s}_l,\tilde{j}_l,\tilde{a}_l,\tilde{s}_l'))_{l=0}^{i}, 
\]
and
\begin{align*} 
  \P{\tilde{s}_i'=s' \mid \mathcal{G}_{i-1}, \tilde{b}_i, \tilde{\beta}_i} 
  =
  p(\tilde{s}_i, \tilde{a}_i, s'), \; \forall s'\in \mathcal{S}, 
\end{align*}
almost surely, where $\mathcal{G}_{i-1} := (\tilde{\beta}_l,(\tilde{t}_l, \tilde{s}_l,\tilde{j}_l,\tilde{a}_l,\tilde{s}_l'))_{l=0}^{i-1}$.

\begin{lemma} \label{lem:deterministic}
Let $\Omega$ be an appropriate sample space. Then for each $\omega_1, \omega_2\in \Omega$ and $i = 1, \dots $ :
\[
  \begin{gathered}
 (\mathcal{G}_{i-1}(\omega_1) = \mathcal{G}_{i-1}(\omega_2)) \wedge (
 \tilde{b}_i(\omega_1) = \tilde{b}_i(\omega_2)) \wedge 
 (\tilde{\beta}_i(\omega_1) = \tilde{\beta}_i(\omega_2))\\
 \implies\\
 \mathcal{F}_i(\omega_1) = \mathcal{F}_i(\omega_2)=\bar{\mathcal{F}}_i(\mathcal{G}_{i-1}(\omega_1),\tilde{b}_i(\omega_1),\tilde{\beta}_i(\omega_1)), \text{a.s.}
  \end{gathered}
\]
for some $\bar{\mathcal{F}}_i$ operator that maps a tuple $((\beta_l,(t_l, s_l,j_l,a_l,s_l'))_{l=0}^{i-1},  (t_i,s_i,j_i,a_i), \beta_i)$ to some $\left( \qud_0, \dots , \qud_i, \phi_0, \dots , \phi_{i-1}, \theta_0, \dots , \theta_i \right)$.
\end{lemma}
\begin{proof}
We proceed by induction on $i$. To prove the base step for $i = 0$:
\[
  \mathcal{F}_0(\omega_1) = (\tqud_0(\omega_1), \tilde{\theta}_0(\omega_1)) =
  (t \ushort{R}, \tilde{\theta}_0(\omega_1)) =
  (t \ushort{R}, \tilde{\theta}_0(\omega_2)) = \mathcal{F}_0(\omega_2) .
\]
Here, $\tilde{\theta}_0(\omega_1) = \tilde{\theta}_0(\omega_2)$ because $\tilde{\beta}_0(\omega_1) = \tilde{\beta}_0(\omega_2)$, and $\tilde{b}_0(\omega_1) = \tilde{b}_0(\omega_2)$.

To prove the inductive step, assume that the property holds for $i$ and prove it for $i+1$. That is, suppose that $l = i+1$
\[
 (\mathcal{G}_{l-1}(\omega_1) = \mathcal{G}_{l-1}(\omega_2)) \wedge 
 (\tilde{b}_l(\omega_1) = \tilde{b}_l(\omega_2)) \wedge (
 \tilde{\beta}_l(\omega_1) = \tilde{\beta}_l(\omega_2)).
\]
Then from the inductive assumption:
\[
 \mathcal{F}_l = \mathcal{F}_l, \quad\forall\, l = 1, \dots, i, 
\]
and for $b = \tilde{b}_{i+1}(\omega_1) = \tilde{b}_{i+1}(\omega_2)$:
\begin{align*}
  (\tilde{\phi}_i(b))(\omega_1)
  &= (H_{\tilde{s}_i'(\omega_1)}\tqud_i(\omega_1))(b) - (H \tqud_i(\omega_1))(b)\\
  &= (H_{\tilde{s}_i'(\omega_2)}\tqud_i(\omega_2))(b) - (H \tqud_i(\omega_2))(b)\\
  &= (\tilde{\phi}_i(b))(\omega_2).
  \\
    (\tilde{\theta}_{i+1}(b))(\omega_1) &= \frac{\tilde{\beta}_{i+1}(\omega_1)}{\xi} = \frac{\tilde{\beta}_{i+1}(\omega_2)}{\xi} = (\tilde{\theta}_{i+1}(b))(\omega_2),
\end{align*}
and for $b \neq \tilde{b}_{i+1}(\omega_1) = \tilde{b}_{i+1}(\omega_2)$:
\begin{align*}
  (\tilde{\phi}_i(b))(\omega_1) &= 0 = (\tilde{\phi}_i(b))(\omega_2).
  \\
    (\tilde{\theta}_{i+1}(b))(\omega_1) &= 0 = (\tilde{\theta}_{i+1}(b))(\omega_2).
\end{align*}
In addition, 
\begin{align*}
  (\tqud_{i+1}(b))(\omega_1)
  &=  (\tqud_i(b))(\omega_1) + (\tilde{\theta}_i(b))(\omega_1) \cdot  (H \tqud_i(\omega_1) + \tilde{\phi}_i(\omega_1) - \tqud_i(\omega_1))(t,s,j,a) \\
  &=  (\tqud_i(b))(\omega_2) + (\tilde{\theta}_i(b))(\omega_2) \cdot  (H \tqud_i(\omega_2) + \tilde{\phi}_i(\omega_2) - \tqud_i(\omega_2))(b), \\
  &= (\tqud_{i+1}(b))(\omega_2).
\end{align*}
Putting the inequalities above together we get the desired equality:
\begin{align*}
  \mathcal{F}_{i+1}(\omega_1)
  &= 
    \left( \tqud_0(\omega_1), \dots , \tqud_{i+1}(\omega_1), \tilde{\phi}_0(\omega_1), \dots , \tilde{\phi}_i(\omega_1), \tilde{\theta}_0(\omega_1), \dots , \tilde{\theta}_{i+1}(\omega_1) \right) \\
  &= 
    \left( \tqud_0(\omega_2), \dots , \tqud_{i+1}(\omega_2), \tilde{\phi}_0(\omega_2), \dots , \tilde{\phi}_i(\omega_2), \tilde{\theta}_0(\omega_1), \dots , \tilde{\theta}_{i+1}(\omega_1) \right) \\
  &= \mathcal{F}_{i+1}(\omega_2).
\end{align*}
\end{proof}

\begin{lemma} \label{lem:f-fun-of-g}
Under \cref{ass:assTransitions}:
\[
  \P{\tilde{s}_i'=s' \mid \mathcal{G}_{i-1}, \tilde{b}_i, \tilde{\beta}_i,\mathcal{F}_i}=p(\tilde{s}_i, \tilde{a}_i, s'), \mbox{ a.s.},
\]
for each $s'\in \mathcal{S}$ and $i\in \Natural$.
\end{lemma}
\begin{proof}
Using \cref{lem:deterministic}, we have that
    \[\P{\mathcal{F}_i=\bar{\mathcal{F}}(\mathcal{G}_{i-1},\tilde{b}_i,\tilde{\beta}_i)\mid \mathcal{G}_{i-1},\tilde{b}_i,\tilde{\beta}_i}=1.\]
    Hence, from \cref{ass:assTransitions} and the law of total probability, for each $s'\in \mathcal{S}, i\in \Natural$:
    \begin{align*}
        p(\tilde{s}_i, \tilde{a}_i, s') &= \P{\tilde{s}_i'=s' \mid \mathcal{G}_{i-1}, \tilde{b}_i, \tilde{\beta}_i}\\
        &=\P{\tilde{s}_i'=s' \mid \mathcal{G}_{i-1}, \tilde{b}_i, \tilde{\beta}_i,\mathcal{F}_i=\bar{\mathcal{F}}(\mathcal{G}_{i-1},\tilde{b}_i,\tilde{\beta}_i)} \P{\mathcal{F}_i=\bar{\mathcal{F}}(\mathcal{G}_{i-1},\tilde{b}_i,\tilde{\beta}_i)\mid  \mathcal{G}_{i-1},\tilde{b}_i,\tilde{\beta}_i)}\\
        &=\P{\tilde{s}_i'=s' \mid \mathcal{G}_{i-1}, \tilde{b}_i, \tilde{\beta}_i,\mathcal{F}_i=\bar{\mathcal{F}}(\mathcal{G}_{i-1},\tilde{b}_i,\tilde{\beta}_i)} \\
        &=\P{\tilde{s}_i'=s' \mid \mathcal{G}_{i-1}, \tilde{b}_i, \tilde{\beta}_i,\mathcal{F}_i} \mbox{ a.s.} \\        
    \end{align*}
\end{proof}

\begin{lemma} \label{lem:noise-unbiased}
The noise $\tilde{\phi}_i$ in \eqref{eq:qlearning-noise-step} satisfies almost surely
  \[
    \E[\tilde{\phi}_{i}(t,s,j,a) \mid \mathcal{F}_i]  = 0, \quad \forall  t \in [T], s\in \mathcal{S}, j\in [J-1], a\in \mathcal{A}, i \in \Natural,
  \]
where $\mathcal{F}_i$ is the history defined in \eqref{eq:history-qlearning}.
\end{lemma}
\begin{proof}
Let $b := (t,s,j,a)$ and $i\in\Natural$ be arbitrary. We decompose the expectation using the law of total expectation to get thta 
\begin{equation} \label{eq:total-exp-bound}
\begin{aligned}
  \E[\tilde{\phi}_{i}(b) \mid \mathcal{F}_i]
  = 
  \E[\tilde{\phi}_{i}(b) \mid \mathcal{F}_i, \tilde{b}_i \neq  b ] \cdot \mathbb{P}[ \tilde{b}_i \neq  b \mid \mathcal{F}_i] +
  \E[\tilde{\phi}_{i}(b) \mid \mathcal{F}_i, \tilde{b}_i = b ] \cdot \mathbb{P}[ \tilde{b}_i = b  \mid \mathcal{F}_i] ~ \text{a.s.},
\end{aligned}
\end{equation}
where $\tilde{b}_{i} := (\tilde{s}_i, \tilde{a}_i, \tilde{t}_i, \tilde{j}_i)$.

The first r.h.s. term in \eqref{eq:total-exp-bound} is, from the definition of $\tilde{\phi}_i(b)$,
\begin{equation} \label{eq:bias-rhs-one}
  \E[\tilde{\phi}_{i}(b) \mid \mathcal{F}_i, \tilde{b}_i \neq b]
  = \E[0 \mid \mathcal{F}_i, \tilde{b}_i \neq b] = 0, \quad \text{a.s.}.
\end{equation}

We now analyze two cases to evaluate the second r.h.s. term in \eqref{eq:total-exp-bound}.

\emph{Case 1}: $j > 0 \wedge t > 0$. Then almost surely:
\begin{equation} \label{eq:bias-rhs-two-c1}
\begin{aligned}
  \E[\tilde{\phi}_{i}(b) \mid \mathcal{F}_i,\tilde{b}_i = b ]
  &= \E[(H_{\tilde{s}_i'} \tqud_i)(b) - (H \tqud_i)(b) \mid \mathcal{F}_i, \tilde{b}_i = b] \\
  &= \xi \cdot \E[-(G_{\tilde{s}_i'} \tqud_i)(b) + (G \tqud_i)(b) \mid \mathcal{F}_i, \tilde{b}_i = b]  \\
  &= \xi \cdot \E[-(G_{\tilde{s}_i'} \tqud_i)(b) +   (G \tqud_i)(b) \mid \mathcal{F}_i, \tilde{b}_i = b]   \\
  &= \xi \cdot \E[ \E[-(G_{\tilde{s}_i'} \tqud_i)(b) \mid  \mathcal{F}_i, \tilde{b}_i = b, \tilde{\beta}_i, \mathcal{G}_{i-1}] +   (G \tqud_i)(b) \mid \mathcal{F}_i, \tilde{b}_i = b ]   \\
&\stackrel{\text{(a)}}{=} \xi \cdot \E[\E^{ a,s}[-(G_{\tilde{s}_1} \tqud_i)(b)] +   (G \tqud_i)(b)    \mid \mathcal{F}_i, \tilde{b}_i = b ] \\
  &= \xi \cdot \E[ - (G \tqud_i)(b) +   (G \tqud_i)(b)  \mid \mathcal{F}_i, \tilde{b}_i = b ]  \\
  &=  0.
\end{aligned}
\end{equation}
To clarify, when $\tilde{s}_1$ is used in an expectation with a superscript, such as $\E^{a,s}$, then it does not represent a sample $\tilde{s}_i$ with $i=1$, but instead it represents the transition from $\tilde{s}_0=s$ to $\tilde{s}_1$ distributed as $p(s,a, \cdot )$.

Step (a) above follows from \cref{lem:f-fun-of-g} given that the randomness of $(G_{\tilde{s}_i'} \tqud_i)(b)$ only comes from $\tilde{s}_i'$ when conditioning on $\mathcal{F}_i, \tilde{b}_i = b, \tilde{\beta}_i$, and $\mathcal{G}_{i-1}$.

\emph{Case 2}: $j = 0 \vee t = 0$.
Directly from the definition of the operators in \eqref{eq:qlearning-operators}:
\begin{equation} \label{eq:bias-rhs-two-c2}
\begin{aligned}
  \E[\tilde{\phi}_{i}(b) \mid \mathcal{F}_i, \tilde{b}_i = b ]
  &= \E[(H_{\tilde{s}_i'} \tqud_i)(b) - (H \tqud_i)(b) \mid \mathcal{F}_i, \tilde{b}_i = b] \\
  &= \xi \cdot  \E[-(G_{\tilde{s}_i'} \tqud_i)(b) + (G \tqud_i)(b) \mid \mathcal{F}_i, \tilde{b}_i = b]  \\
  &= 0.
\end{aligned}
\end{equation}
Substituting \eqref{eq:bias-rhs-one}, and the appropriate case, \eqref{eq:bias-rhs-two-c1} or \eqref{eq:bias-rhs-two-c2}, into \eqref{eq:total-exp-bound} proves the desired equality. 
\end{proof}

\begin{lemma} \label{lem:noise-unvarianced}
The noise $\tilde{\phi}_i$ in \eqref{eq:qlearning-noise-step} satisfies
  \[
    \E[(\tilde{\phi}_{i}(t,s,j,a))^2 \mid \mathcal{F}_i]  \le c + g \| \tqud_i \|_{\infty}^2, \quad \forall  t \in [T], s\in \mathcal{S}, j\in [J-1], a\in \mathcal{A}, i \in \Natural ,
  \]
almost surely, for some $c,g\in \Real_+$ where $\mathcal{F}_i$ is the history defined in \eqref{eq:history-qlearning}.
\end{lemma}
\begin{proof}
Let $b := (t,s,j,a)$ and $i\in\Natural$ be arbitrary. We decompose the expectation using the law of total expectation to get almost surely
\begin{equation} \label{eq:total-exp-bound-var}
\begin{aligned}
  \E[\tilde{\phi}_{i}(b)^2 \mid \mathcal{F}_i]
  = 
  \E[\tilde{\phi}_i(b)^2 \mid \mathcal{F}_i, \tilde{b}_i \neq  b ] \cdot \mathbb{P}[ \tilde{b}_i \neq  b \mid \mathcal{F}_i] +
  \E[\tilde{\phi}_i(b)^2 \mid \mathcal{F}_i, \tilde{b}_i = b ] \cdot \mathbb{P}[ \tilde{b}_i = b  \mid \mathcal{F}_i],
\end{aligned}
\end{equation}
where $\tilde{b}_{i} := (\tilde{t}_i, \tilde{s}_i, \tilde{a}_i, \tilde{j}_i, \tilde{a}_i)$.

The first r.h.s. term in \eqref{eq:total-exp-bound-var} is, from the definition of $\tilde{\phi}_i(b)$,
\begin{equation} 
  \E[\tilde{\phi}_{i}(b)^2 \mid \mathcal{F}_i, \tilde{b}_i \neq b]
  = \E[0 \mid \mathcal{F}_i, \tilde{b}_i \neq b] = 0, \quad \text{a.s.}.
\end{equation}
  
We now analyze two cases to evaluate the second r.h.s. term in \eqref{eq:total-exp-bound-var}.

\emph{Case 1}: Assume that $j > 0, t > 0$. Then from the definitions of the operators in \eqref{eq:qlearning-operators}:
\begin{align*}
  \E[(\tilde{\phi}_{i}(b))^2 \mid \mathcal{F}_i, \tilde{b}_i = b]
  &= \E\left[ \left((H_{\tilde{s}_i'} \tqud_i)(b) - (H \tqud_i)(b)\right)^2 \mid \mathcal{F}_i, \tilde{b}_i = b \right] \\
  &= \xi^2 \E\left[\left(-(G_{\tilde{s}_i'} \tqud_i)(b) + (G \tqud_i)(b)\right)^2 \mid \mathcal{F}_i, \tilde{b}_i = b\right]  \\
  &= \xi^2 \E\left[ \E \left[ \left(-(G_{\tilde{s}_i'} \tqud_i)(b) + (G \tqud_i)(b)\right)^2 \mid  \mathcal{F}_i, \tilde{b}_i = b, \tilde{\beta}_i, \mathcal{G}_{i-1} \right] \mid \mathcal{F}_i, \tilde{b}_i = b\right]  \\
  &\stackrel{\text{(a)}}{=} \xi^2 \E \left[ \E^{a,s}\left[\left((G_{\tilde{s}_1} \tqud_i)(b) - (G \tqud_i)(b)\right)^2 \right] \mid \mathcal{F}_i, \tilde{b}_i = b \right].
\end{align*}
To clarify, when $\tilde{s}_1$ is used in an expectation with a superscript, such as $\E^{a,s}$, then it does not represent a sample $\tilde{s}_i$ with $i=1$, but instead it represents the transition from $\tilde{s}_0=s$ to $\tilde{s}_1$ distributed as $p(s,a, \cdot )$.

Step (a) above follows from \cref{lem:f-fun-of-g} given that the randomness of $-(G_{\tilde{s}_i'} \tqud_i)(b) + (G \tqud_i)(b)$ only comes from $\tilde{s}_i'$ when conditioning on $\mathcal{F}_i, \tilde{b}_i = b, \tilde{\beta}_i$, and $\mathcal{G}_{i-1}$.
Then, continuing the derivation:
\begin{align*}
  \E[(\tilde{\phi}_{i}(b))^2 \mid \mathcal{F}_i, \tilde{b}_i = b]
  &= \xi^2   \E \left[ \E^{a,s}\left[\left((G_{\tilde{s}_1} \tqud_i)(b) - (G \tqud_i)(b)\right)^2 \right] \mid \mathcal{F}_i, \tilde{b}_i = b \right]\\
    &\stackrel{\text{(a)}}{=} \xi^2 \E\left[  \E^{a,s}\left[ \left(\E\left[{\partial\ell^{\kappa}_{\frac{j}{J}}}\left(\tilde{\delta}_i(\tilde{s}_1,\tilde{j}')\right) | \tilde{s}_1 \right]   -  \E^{a,s}\left[{\partial\ell^{\kappa}_{\frac{j}{J}}}\left(\tilde{\delta}_i(\tilde{s}_1,{\tilde{j}'})  \right) \right]\right)^2  \right] \mid \mathcal{F}_i, \tilde{b}_i = b \right] \\    
  &\stackrel{\text{(b)}}{=}  \xi^2 \E\left[ \left(  \E^{a,s}\left[ \left(\E\left[{\partial\ell^{\kappa}_{\frac{j}{J}}}\left(\tilde{\delta}_i(\tilde{s}_1,\tilde{j}')\right) \mid  \tilde{s}_1 \right] \right)^2\right]  -  \left(\E^{a,s}\left[{\partial\ell^{\kappa}_{\frac{j}{J}}}\left(\tilde{\delta}_i(\tilde{s}_1,\tilde{j}')  \right) \right]\right)^2 \right)  \mid \mathcal{F}_i, \tilde{b}_i = b \right]\\
  &\leq  \xi^2 \E\left[  \E^{a,s}\left[ \left(\E\left[{\partial\ell^{\kappa}_{\frac{j}{J}}}\left(\tilde{\delta}_i(\tilde{s}_1,\tilde{j}')\right) \mid  \tilde{s}_1 \right] \right)^2\right]  \mid \mathcal{F}_i, \tilde{b}_i = b \right] \\
  &\stackrel{\text{(c)}}{\le} \xi^{2} \E\left[ \max_{j'\in [J-1],s'\in \mathcal{S}} {\partial\ell^{\kappa}_{\frac{j}{J}}}(\tilde{\delta}_i(s',j') )^2  \mid \mathcal{F}_i, \tilde{b}_i = b \right]\\
  &\stackrel{\text{(d)}}{\le}  \xi^{2} \E\left[ \max_{j'\in [J-1],s'\in \mathcal{S}} \left( | {\partial\ell^{\kappa}_{\frac{j}{J}}}(\tilde{\delta}_i(s',j') ) - {\partial\ell^{\kappa}_{\frac{j}{J}}}(0 ) | \right)^2 \mid \mathcal{F}_i, \tilde{b}_i = b \right] \\
  &\stackrel{\text{(e)}}{\le}  \xi^{2}  \E \left[\max_{j'\in [J-1],s'\in \mathcal{S}} \left( \max \left\{ \frac{j}{J}, 1- \frac{j}{J} \right\} \kappa^{-1}  |\tilde{\delta}_i(s',j')|  \right)^2  \mid \mathcal{F}_i, \tilde{b}_i = b \right] \\
  &\leq  \xi^{2} \kappa^{-2} \cdot \E \left[\max_{j'\in [J-1],s'\in \mathcal{S}} \tilde{\delta}_i(s',j')^2  \mid \mathcal{F}_i, \tilde{b}_i = b \right] \\  
  &\stackrel{\text{(f)}}{\le}  \xi^{2} \kappa^{-2} \cdot ( 2\|r\|_{\infty}^2   + 8 \cdot \|\tqud_i\|_{\infty}^2).  
\end{align*}
Step (a) above follows by substituting $(G_{\tilde{s}_1} \tqud_i)(b) - (G \tqud_i)(b)$ and replacing
\begin{equation}
\tilde{\delta}_i(s',j') := r(s,a)+ \gamma \max_{a'\in \actions} \tqud_i(t-1,s',j',a')-\tqud_i(t,s,j,a).
\end{equation}
The equality in step (b) holds because for a random variable $\tilde{x}:=\E[{\partial\ell^{\kappa}_{\frac{j}{J}}}\left(\delta_i(\tilde{s}_1,\tilde{j}')\right) | \tilde{s}_1 ]$, the variance satisfies $\E[(\tilde{x}-\E[\tilde{x}])^2]=\E[\tilde{x}^2]-(\E[\tilde{x}])^2$. Step (c) upper bounds the expectation by a supremum, and step (d) uses ${\partial\ell^{\kappa}_{\frac{j}{J}}}(0 ) = 0$ from the definition in \eqref{eq:delageLoss2}. Step (e) uses \cref{lem:bounded_hessian} to bound the derivative difference as a function of the step size. Finally, step (f) derives the final upper bound since
\[
  \| r \|_{\infty} = \max_{s\in \mathcal{S}, a\in \mathcal{A}} |r(s,a)|,
  \qquad
  \| \tqud_i \|_{\infty} = \max_{b\in \mathcal{K}} |\tqud_i(b)|,
\]
where $\mathcal{K} = [T] \times \mathcal{S} \times [J-1] \times  \mathcal{A}$ and
\begin{align*}
  \max_{j'\in [J-1],s'\in \mathcal{S}}\tilde{\delta}_i(s',j')^2
  &\leq (\| r \|_{\infty}+2\| \tqud_i \|_{\infty})^2 \leq (\| r \|_{\infty}+2\| \tqud_i \|_{\infty})^2 + (\| r \|_{\infty}-2\| \tqud_i \|_{\infty})^2 \\
    &= 2\| r \|_{\infty}^2+8\| \tqud_i \|_{\infty}^2,
\end{align*}
and because $\tqud_i$ is measurable on $\mathcal{F}_i$.

\emph{Case 2}:  Assume that $t = 0 \vee j = 0$. Directly from the definition of the operators in \eqref{eq:qlearning-operators}:
\begin{align*}
  \E[(\tilde{\phi}_{i}(b))^2 \mid \mathcal{F}_i, \tilde{b}_i = b]
  &= \E[((H_{\tilde{s}_i'} \tqud_i)(b) - (H \tqud_i)(b) )^2 \mid \mathcal{F}_i, \tilde{b}_i = b] \\
  &= \xi \cdot  \E[(-(G_{\tilde{s}_i'} \tqud_i)(b) + (G \tqud_i)(b) )^2 \mid \mathcal{F}_i, \tilde{b}_i = b]  \\
  &= 0.
\end{align*}

Finally, we can confirm that 
\begin{align*}
  \E[\tilde{\phi}_{i}(b)^2 \mid \mathcal{F}_i]
  &= 
  \E[\tilde{\phi}_i(b)^2 \mid \mathcal{F}_i, \tilde{b}_i \neq  b ] \cdot \mathbb{P}[ \tilde{b}_i \neq  b \mid \mathcal{F}_i] +
  \E[\tilde{\phi}_i(b)^2 \mid \mathcal{F}_i, \tilde{b}_i = b ] \cdot \mathbb{P}[ \tilde{b}_i = b  \mid \mathcal{F}_i]\\
  &\leq 0\cdot \mathbb{P}[ \tilde{b}_i \neq  b \mid \mathcal{F}_i] +
  \xi^{2} \kappa^{-2} \cdot ( 2\|r\|_{\infty}^2   + 8 \cdot \|\tqud_i\|_{\infty}^2) \cdot \mathbb{P}[ \tilde{b}_i = b  \mid \mathcal{F}_i]\\
  &\leq \xi^{2} \kappa^{-2} \cdot ( 2\|r\|_{\infty}^2   + 8 \cdot \|\tqud_i\|_{\infty}^2).
\end{align*}
\end{proof}

\subsubsection{Main Proof}

\begin{proof}[Proof of \cref{thm:convergenceQlearning}]
We verify that the sequence of our q-learning iterates satisfies the properties in \cref{thm:bertsekas}.
\begin{itemize}
\item The step size condition in \cref{thm:convergenceQlearning} guarantees that we satisfy property (a) in \cref{thm:bertsekas}
\item We satisfy property (b) in \cref{thm:bertsekas} because
\begin{itemize}
\item \cref{lem:noise-unbiased} shows that we satisfy property (1) in \cref{assumption:expected_convergence}
\item \cref{lem:noise-unvarianced} shows that we satisfy property (2) in \cref{assumption:expected_convergence}
\end{itemize}
\item \cref{thm:h-contraction} shows that we satisfy property (c) in  \cref{thm:bertsekas}
\end{itemize}
\end{proof}

\section{Empirical Results} \label{sec:more_experiments}

The code used to generate all plots can be found in \url{https://github.com/MonkieDein/DRA-Q-LA}.

\subsection{Domain Details} \label{sec:more_domain_details}
For each domain, we provide CSV files and julia JLD files in the supplementary material with the exact specifications of the domains we use. Domain detail for six out of seven of our domains include (Machine Replacement (MR), Gambler's Ruin (GR), Inventory1 (INV1), Inventory2 (INV2), Riverswim (RS) and Population Management (POP)) can be found in  \citep[Appx. E]{hau2023entropic}. The Cliffwalk (CW) domain is similar to the one described in  \citep[Ex. 6.6]{sutton2018reinforcement}, with a minor modification. In this version, the agent transitions to each adjacent direction with a $10\%$-probability instead of always following the selected direction (see \cref{fig:cliffwalk_domain}). The initial state $s_0$ specification can be found in \cref{tab:domains_initial_state}. We initialize all environments with a discount factor of $\gamma = 0.9$ and a horizon $T=100$. 

\begin{figure}
\centering
\includegraphics[width=.7\linewidth]{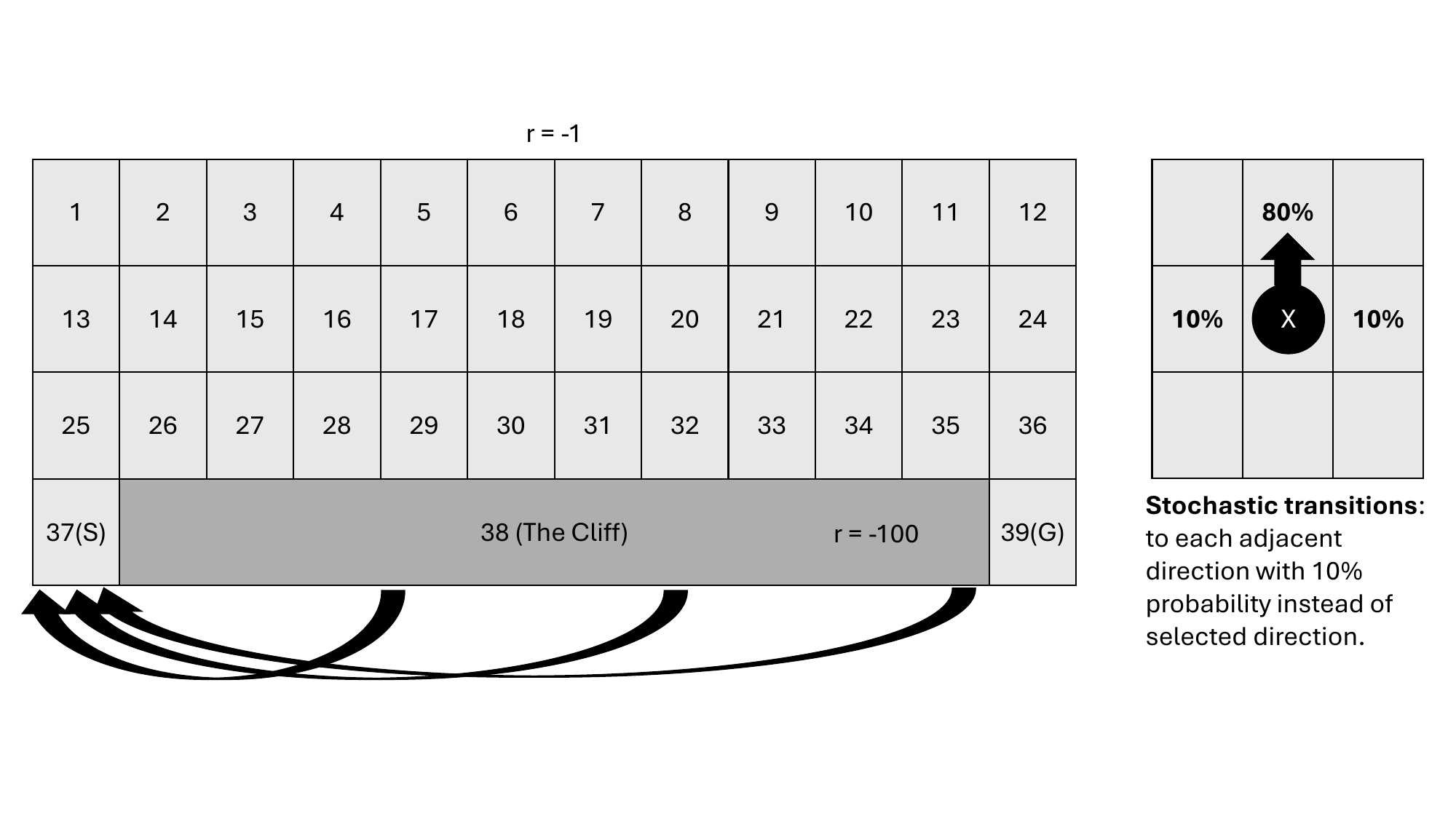}
\caption{The cliffwalk domain} \label{fig:cliffwalk_domain}
\end{figure}

\begin{table}[ht]
\centering
\begin{tabular}{|l|l|l|l|l|l|l|l|}
\hline
        & MR & GR & INV1 & INV2 & RS & POP & CW \\
        \hline
$s_0$   & $1$ & $5$   & $10$ & $20$ & $9$ & $44$ & $37$   \\
\hline
\end{tabular}
\caption{Initial state for each domain} \label{tab:domains_initial_state}
\end{table}

\subsection{Algorithmic Details} \label{sec:more_algorithm_details}
\paragraph{\cref{alg:quantile_policy_exe}} Line 6 is implemented with $\epsilon= 10^{-14}$ to account for the non-associative property of floating point arithmetic with $\tau = \frac{ \ushort{q}_t^\discretized(s,j,a\opt)-r}{\gamma}$ as:
\[j\gets \argmin \left\{j' \in [J-1]\mid\max_{a'\in\actions}\ushort{q}_{t-1}^\discretized(s',j',a')\geq \tau - \epsilon | \tau |\right\} \;.\] 

\paragraph{\cref{alg:q-learning}} Without loss of generality, the VaR-Q-value function is trained with a standardized scaled reward function $\hat{r}(s,a) \gets \frac{r(s,a) - \ushort{R}}{\bar{R}-\ushort{R}}$. We also remove the time indices to reduce computational overhead and initialize the q-value function with
$\hat{q}^\discretized \gets (1-\gamma)^{-1}$. The VaR-Q-value function is then unscaled via $\tilde{q}^\discretized \gets \hat{q}^\discretized \cdot (\bar{R}-\ushort{R}) + \frac{\ushort{R}}{1-\gamma}$ before being compared with the DP variant $\ushort{q}^\discretized$. The learning rate is defined as $\beta_i \gets 100 \cdot(0.1^{i \cdot 0.0003})$ for $i$-th occurrence of sample $(s,a)$ across all domains. For a fair comparison between domains, we sample a transition for every $(s,a)$-pair at each iteration.

\paragraph{Nested VaR (nVaR)} 

Also known as dynamic VaR, nVaR is solved via the following DP for each $s\in \states$ and $t\in[T-1]$ as
\[
 v_{t+1} (s) = \max_{a\in \actions} \, \var{\initalpha}{r(s,a) + \gamma \cdot v_{t} (\tilde{s}')}, 
\]
where $\tilde{s}' \sim p(s,a,\cdot)$. 
Then, we evaluate a greedy policy $\pi_t\colon \states \to \actions, k \in [T-1]$ constructed to satisfy
\[
 \pi_k(s) \in \argmax_{a\in \actions} \, \var{\initalpha}{r(s,a) + \gamma \cdot v_{T-k}(\tilde{s}')}. 
\]
\paragraph{Distributional VaR (dVaR)} It uses the Markov action-selection strategy proposed by \cite{dabney2018implicit,keramati2020being} with $J=4096$ uniform quantile discretization \citep{dabney2018distributional,rowland2023analysis}: 
\[
q_{t+1}(s,\alpha_j,a) = \var{\alpha_j}{r(s,a) + \gamma \max_{a' \in \actions} \var{\initalpha}{q_t(\tilde{s}',\tilde{u},a')}} \; \qquad \forall \alpha_j = \frac{2j+1}{J} \text{ where } j \in [J -1] \;, 
\]
where $\tilde{u}$ refer to the discretized uniform distribution satisfy $\mathbb{P}[\tilde{u} = \frac{2j+1}{J}] = \frac{1}{J} ~\forall j \in [J-1]$ follows greedy policy
\[
 \pi_k(s) \in \argmax_{a\in \actions} \, \var{\initalpha}{q_{T-k}(s,\tilde{u},a)}. 
\]
In contrast to our algorithm, an optimal action $a$ is selected w.r.t the initial Markov risk level of interest $\initalpha$ instead of the quantile-dependent risk level $\alpha_j$ \citep{lim2022distributional}.

\paragraph{CVaR} 

We implemented the algorithm described in \citep{chow2015risk,hau2023dynamic} with $J=4096$ uniform discretization as 
\[    (B_{\max} q)(s,\alpha_j, a) := r(s,a) + \gamma \cdot \min_{\zeta \in \mathcal{Z}_{sa}(\alpha_j)} \sum_{s' \in \states} \zeta_{s'} \cdot  \max_{a'\in \actions} q( {s}',\frac{\alpha_j \zeta_{{s}'}}{p(s,a,{s}')},a')  \qquad \forall \alpha_j = \frac{j}{J} \text{ where } j \in [J] \;,  
\]\[\mathcal{Z}_{sa}(\alpha) \;:=\;  \left\{ \xi \in \probs{\states} \mid \alpha \zeta_{s'} \le p(s,a,s') 
  \right\}.
\]
which follow a greedy history-dependent policy described in \citep{chow2015risk}. We use the same discretization level as our algorithm. It is important to note that this algorithm could over-approximate the true static CVaR value function due to the duality gap, so it may perform badly \citep{hau2023dynamic}.

\paragraph{EVaR} Algorithm described in \citep{hau2023entropic} and known to perform well when evaluated with static CVaR and EVaR. We implemented the algorithm with the time-dependent policy described there, and with fixed ERM discretization $\beta_j = 100 \cdot (0.99^j) ~\forall j \in [3000]$, instead of a domain-dependent discretization.

\subsection{More empirical results}

\begin{figure} 
    \centering
        \includegraphics[width=0.65\linewidth]{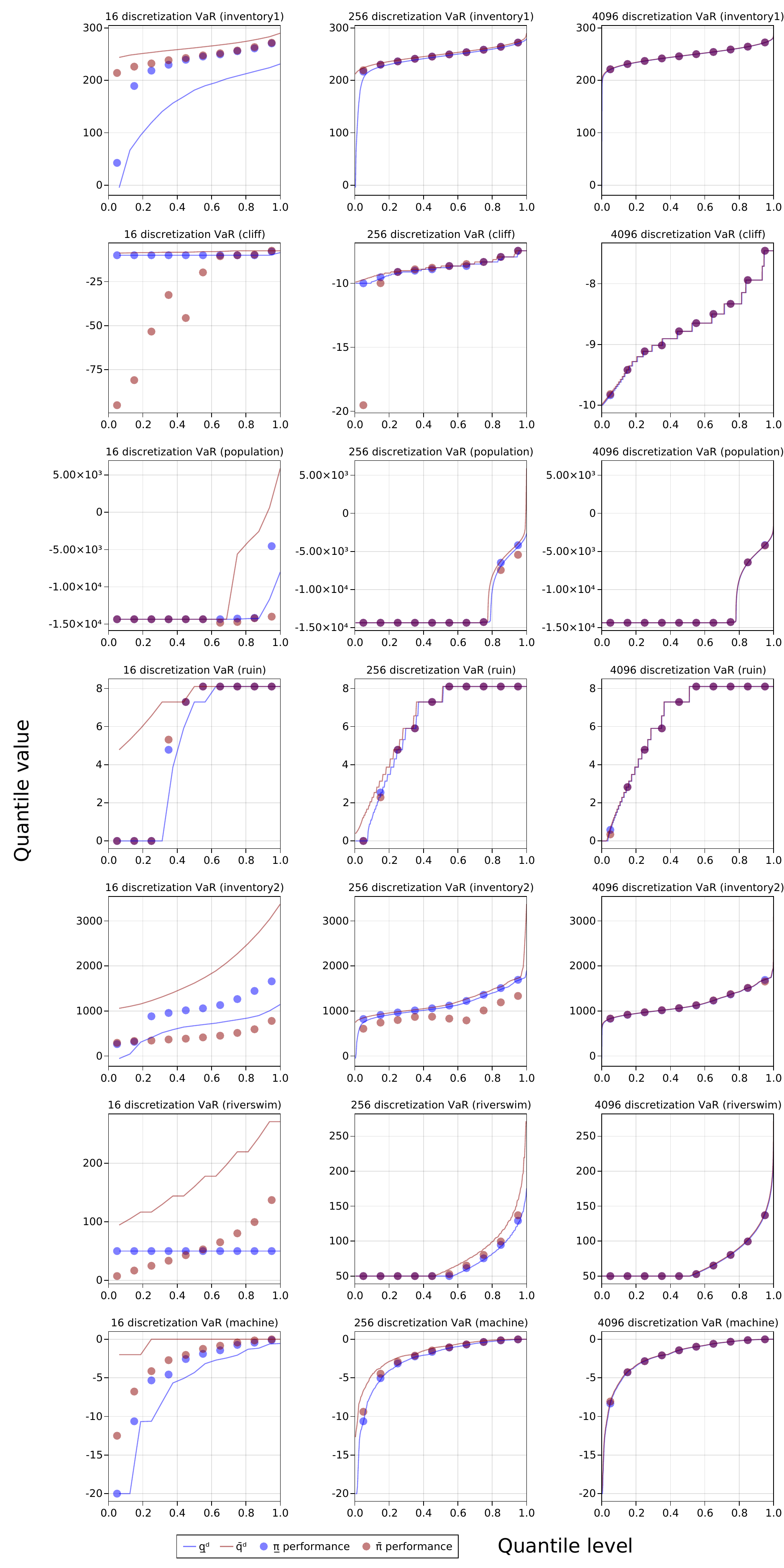}
        \caption{Approximation bound $q^\discretized$ and its respective $\pi$ policy performance} \label{fig:all_multiple_discretization}
\end{figure}

\cref{fig:all_multiple_discretization} extends \cref{fig:inv2_multiple_discretization} demonstrates that the performance of $\ushort{\pi}$ across all the domains to understand how the selection of $\ushort{q}^\discretized$ in \cref{alg:quantile_policy_exe} contribute to the quality of the solution. More specifically, $\ushort{\pi}$ lies within $[\ushort{q}^\discretized,\bar{q}^\discretized]$, whereas $\bar{\pi}$ may performs worse than $\ushort{q}^\discretized$. Furthermore, as the discretization level increases, the bounding gap $\bar{q}^\discretized-\ushort{q}^\discretized$ shrinks, suggesting that  $\ushort{\pi}$ converges to $\pi \opt$.

\begin{figure} 
    \centering
        \includegraphics[width=0.7\linewidth]{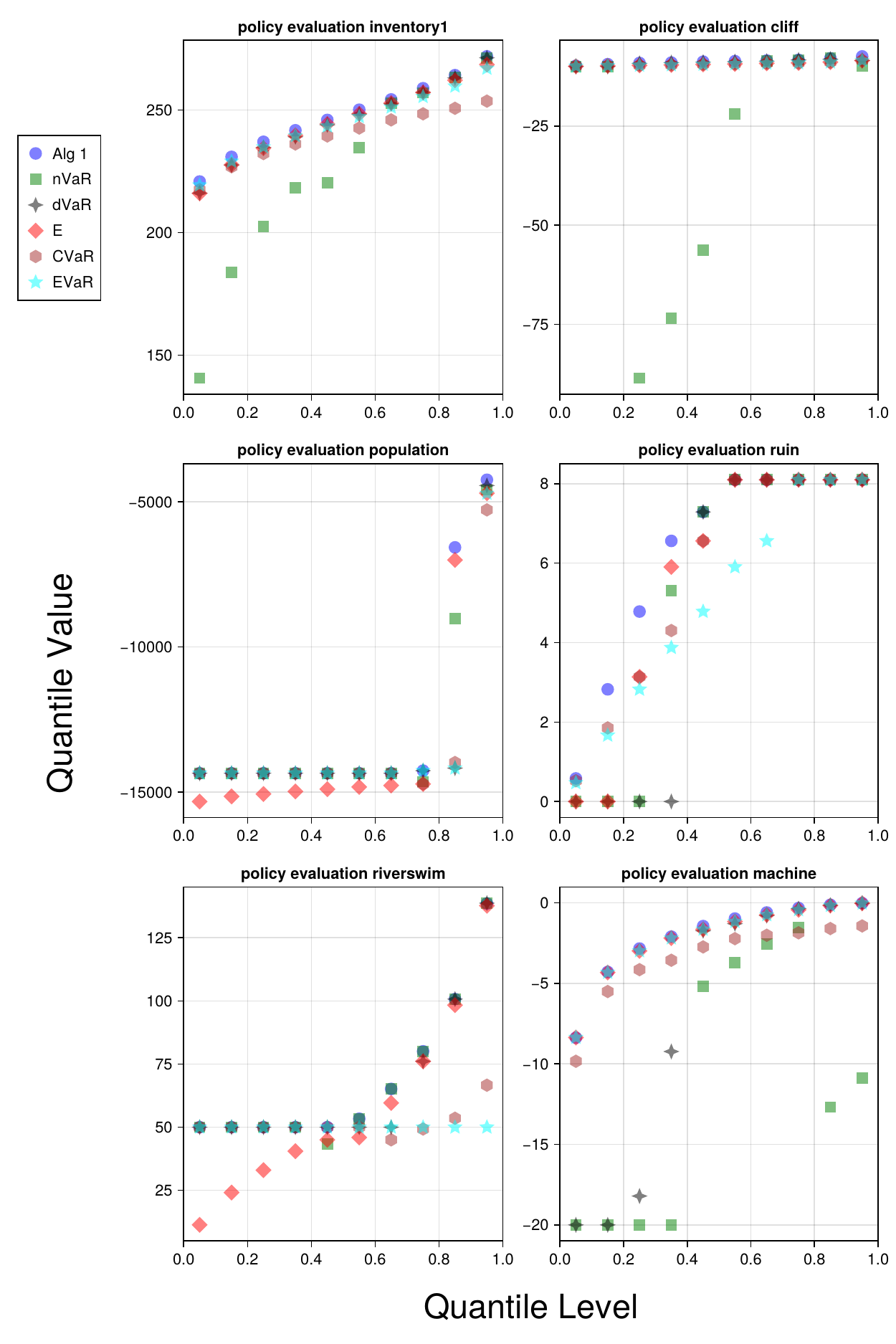}
        \caption{Policy performance evaluation } \label{fig:all_algorithms_compare}
\end{figure}

\cref{fig:all_algorithms_compare} extends \cref{fig:inv2_comparison} compares our algorithm with other related algorithms (detailed in \cref{sec:more_algorithm_details}) for all the domains on quantile levels $\initalpha \in \{0.05,0.15,\dots,0.85,0.95\}$. As we can see, our algorithm consistently outperforms all other algorithms across all tested domains and quantile levels illustrating the robustness of our algorithm.

\begin{figure} 
    \centering
        \includegraphics[width=0.86\linewidth]{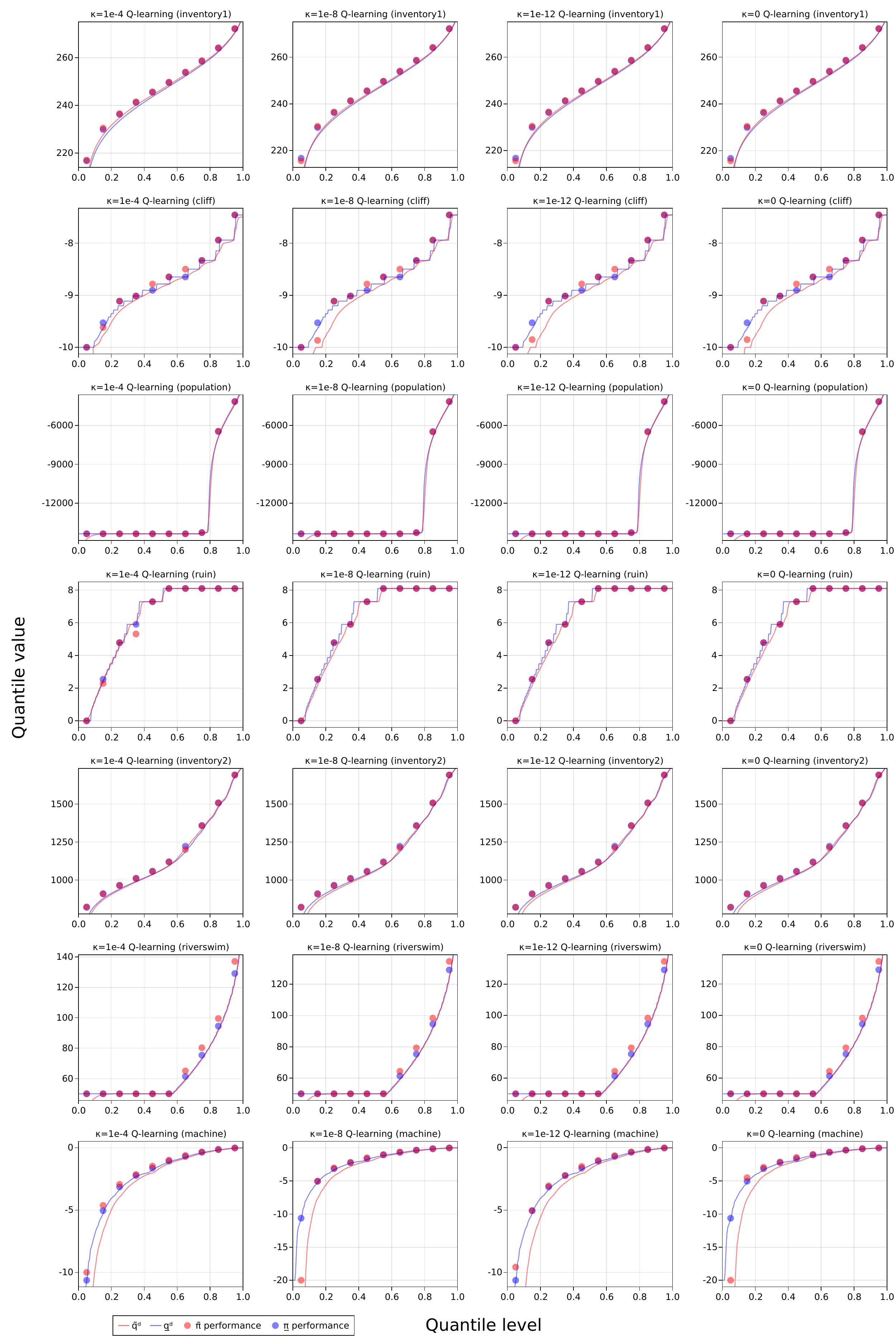}
        \caption{Q learning $\tilde{q}$ vs DP $\ushort{q}$ value function and policy performance after $20,000$ iterations ($J = 256$)} \label{fig:all_algorithms_Q_learning_value}
\end{figure}

\cref{fig:all_algorithms_Q_learning_value} shows that for all the domains, both the value function and the performance of the policy for the $\kappa$-soft quantile Q-learning (\cref{alg:q-learning}) with $\kappa \in \{10^{-4},10^{-8},10^{-12},0\}$ and uniform discretization of $J=256$ converges to DP variant \cref{eq:discreteBellmanEq} after 20,000 iterations. Not only the value functions for the Q-learning converges closely to the DP's, the performance of the policy implied by the Q-learning value function also matches the policy from the DP variant.

\begin{figure}
    \centering
    \begin{subfigure}[b]{0.48\textwidth}
        \centering
        \includegraphics[width=.7\textwidth]{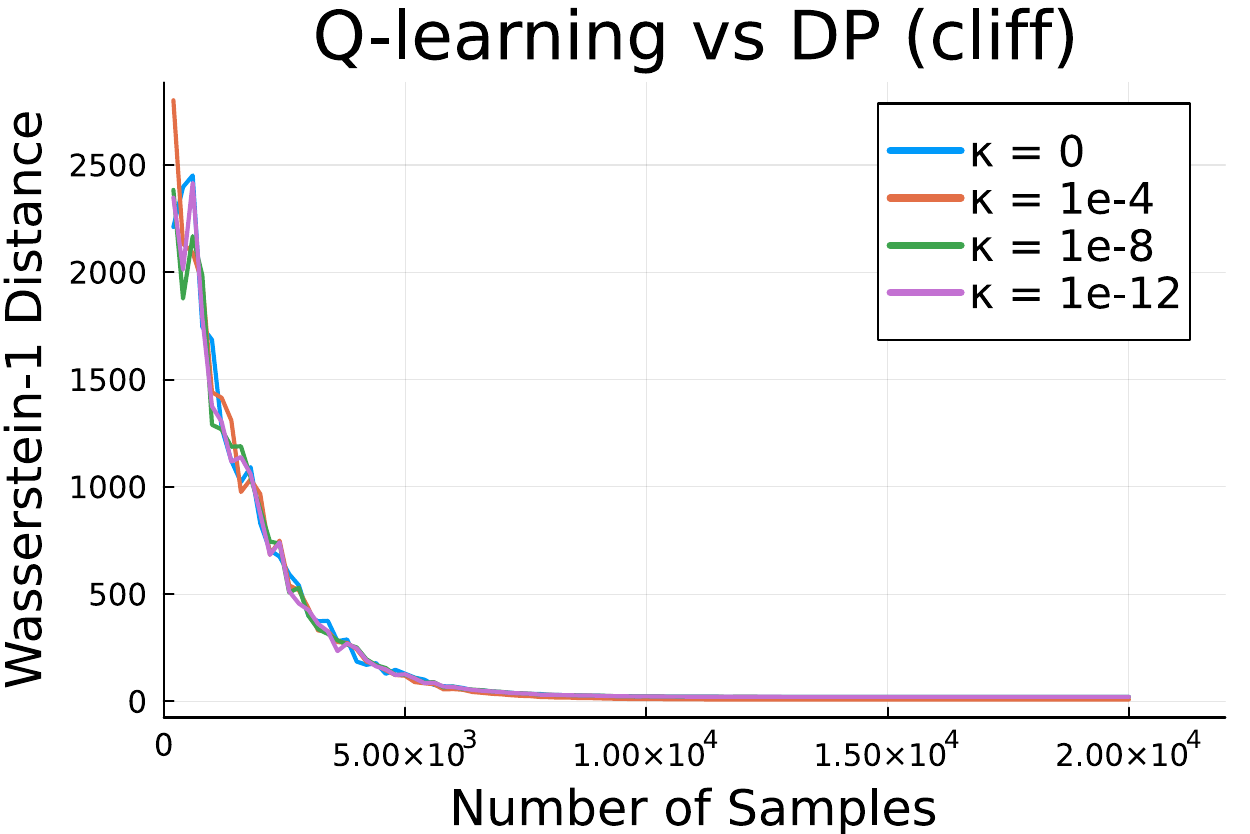}
    \end{subfigure}
    \hfill
    \begin{subfigure}[b]{0.48\textwidth}
        \centering
        \includegraphics[width=.7\textwidth]{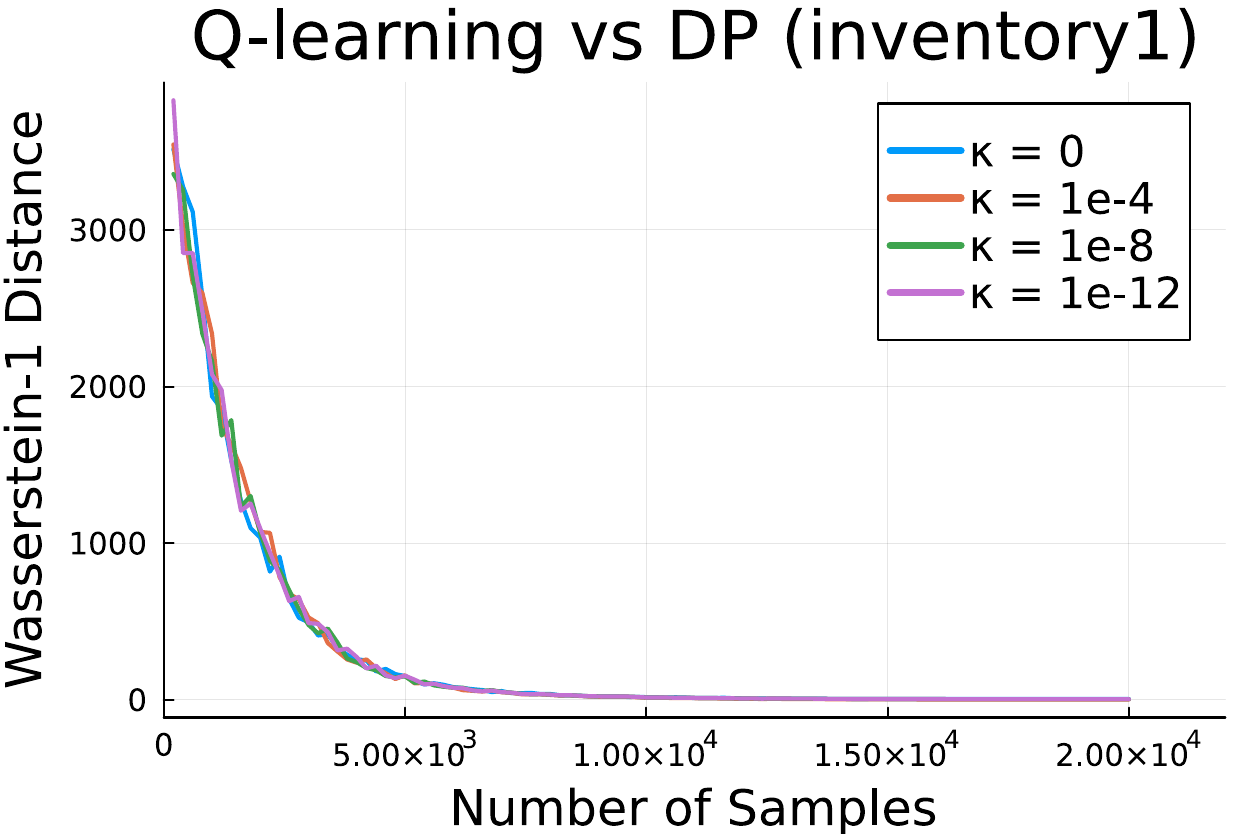}
    \end{subfigure}
    
    \begin{subfigure}[b]{0.48\textwidth}
        \centering
        \includegraphics[width=.7\textwidth]{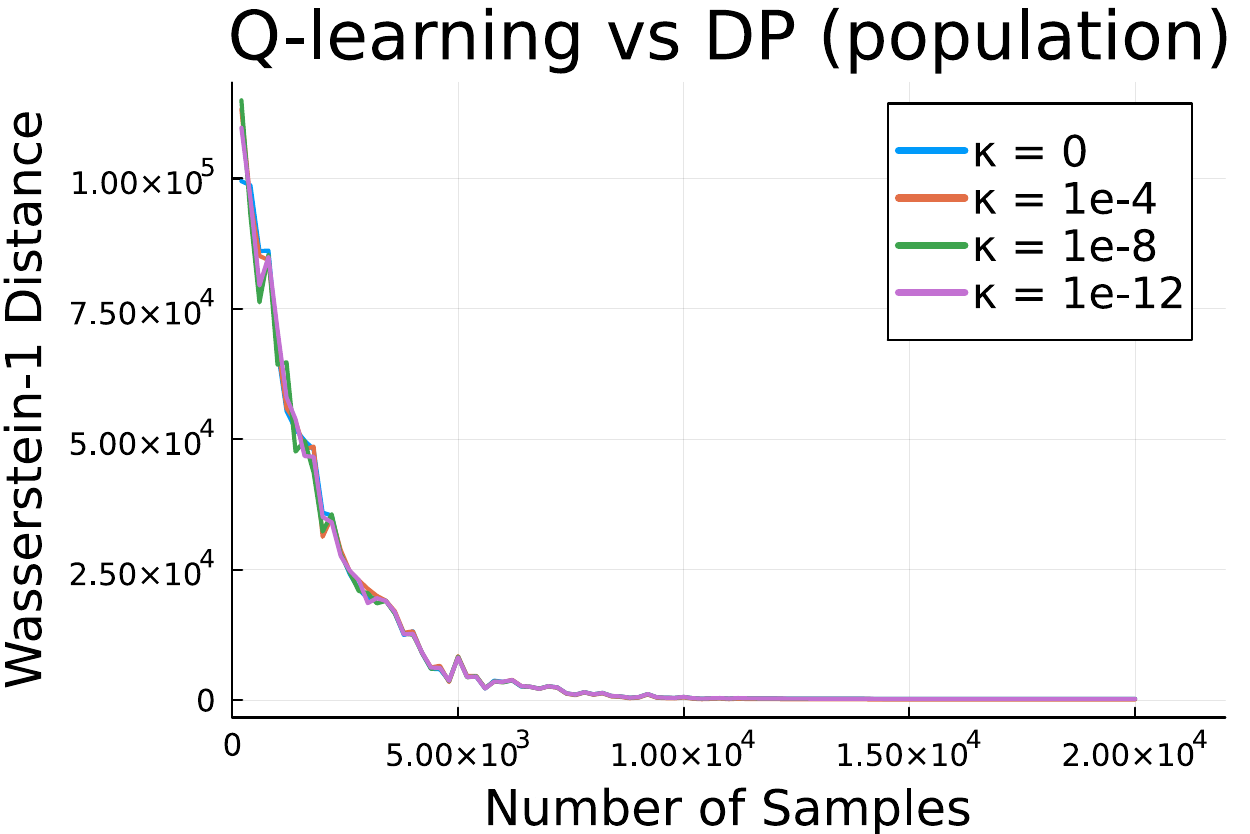}
    \end{subfigure}
    \hfill
    \begin{subfigure}[b]{0.48\textwidth}
        \centering
        \includegraphics[width=.7\textwidth]{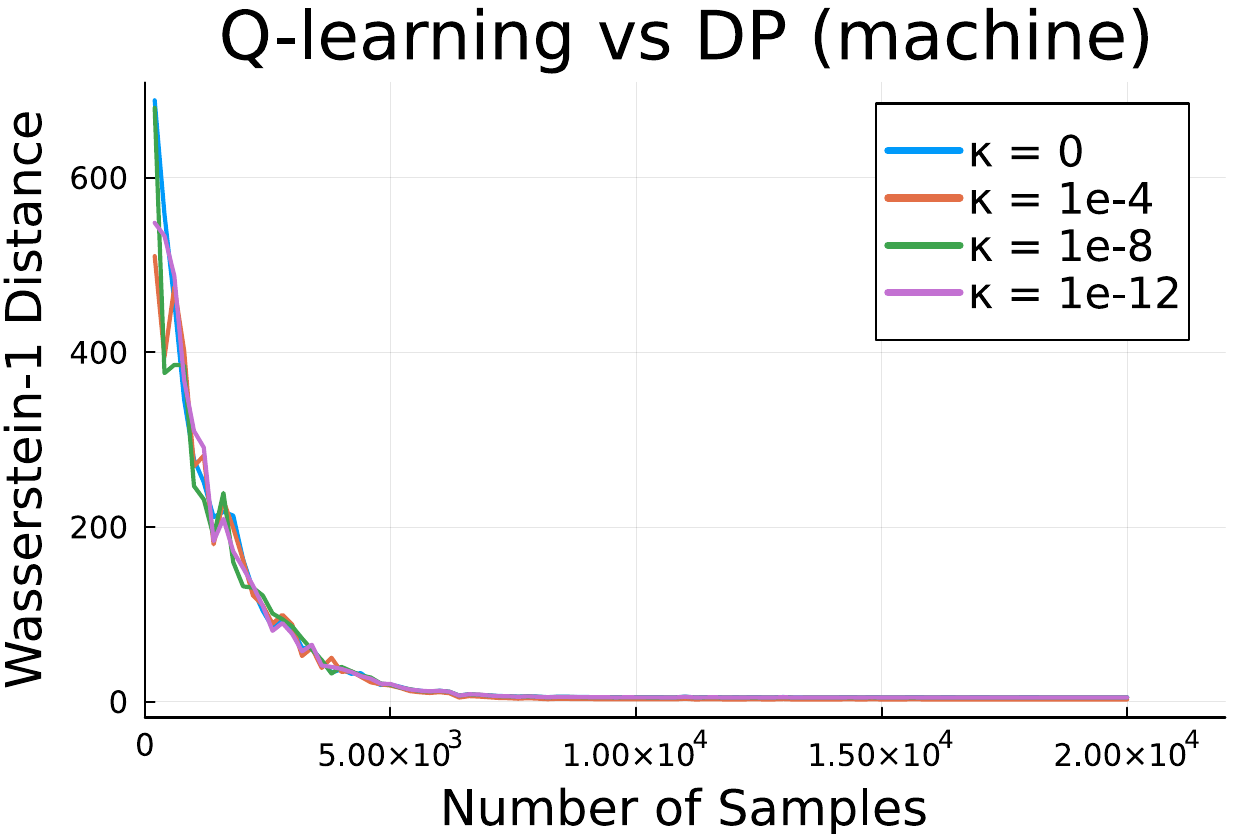}
    \end{subfigure}
    
    \begin{subfigure}[b]{0.48\textwidth}
        \centering
        \includegraphics[width=.7\textwidth]{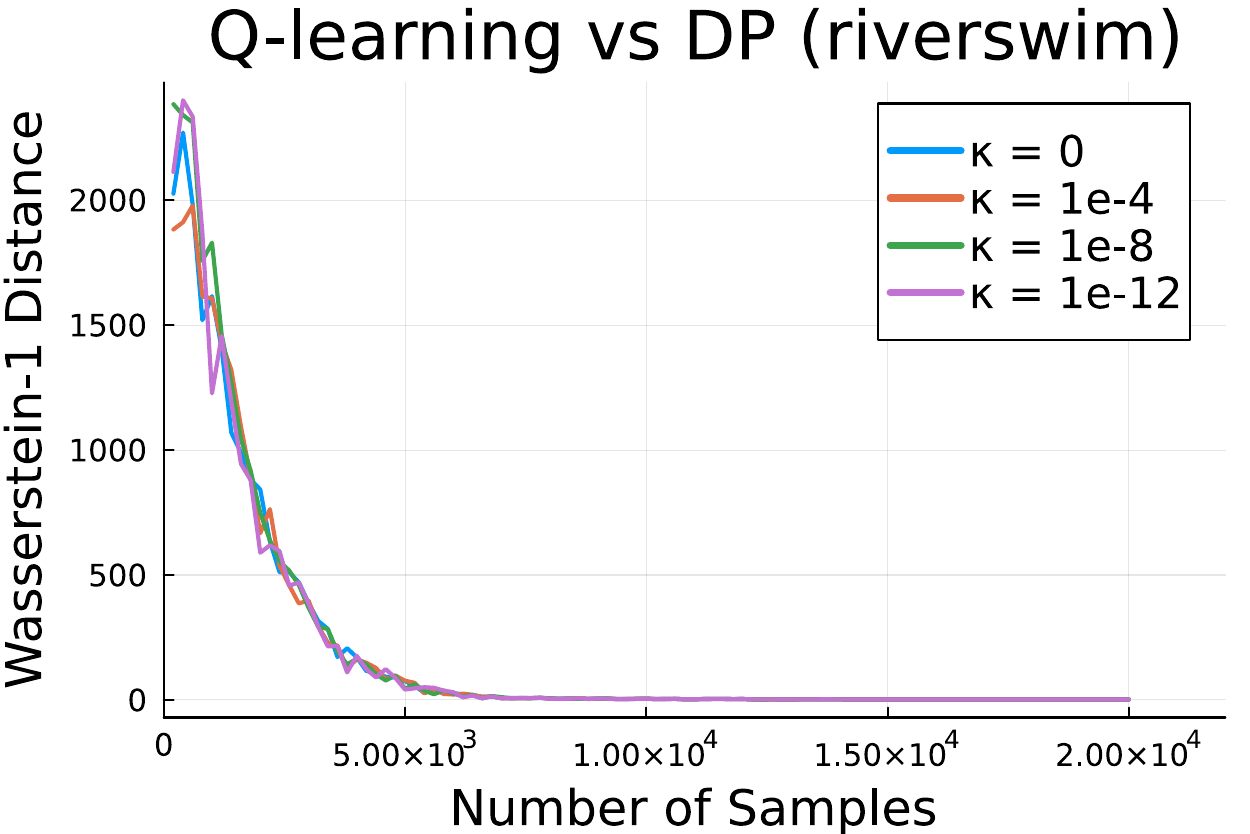}
    \end{subfigure}
    \hfill
    \begin{subfigure}[b]{0.48\textwidth}
        \centering
        \includegraphics[width=.7\textwidth]{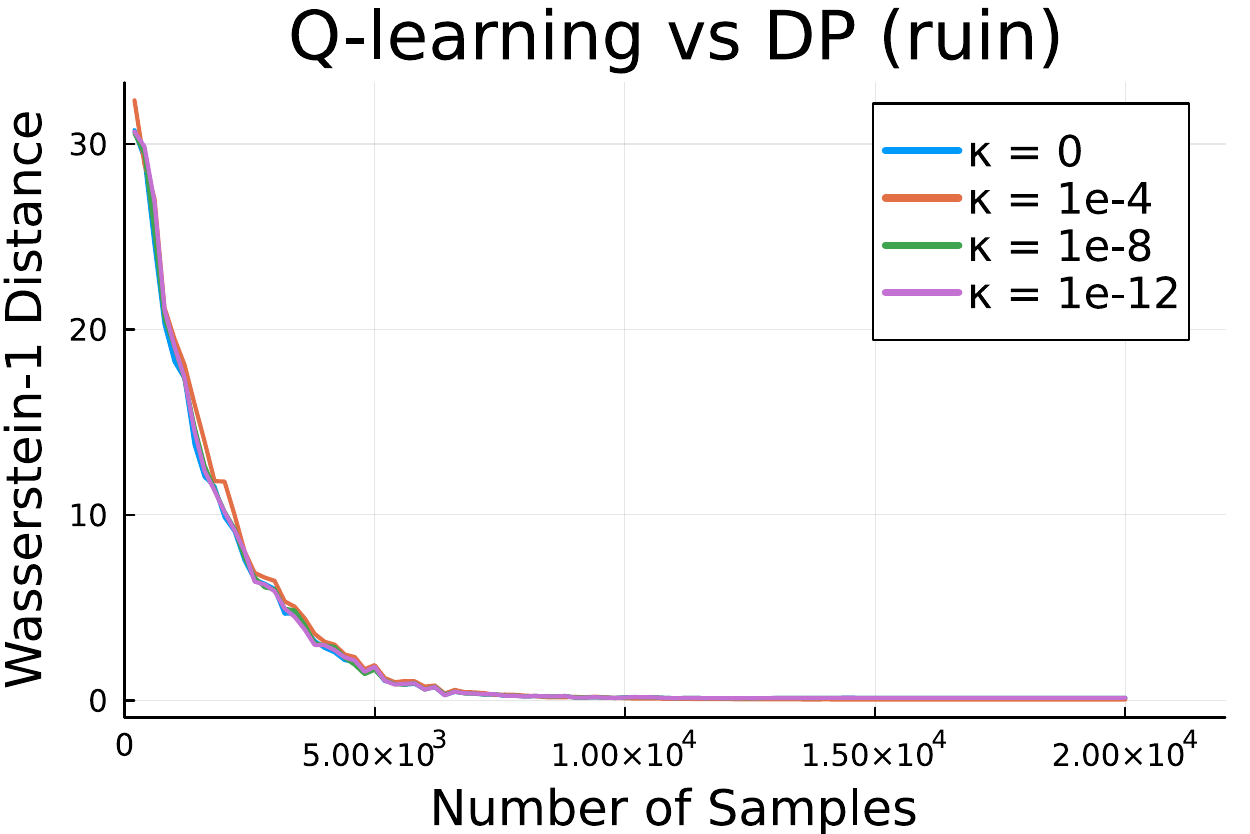}
    \end{subfigure}  
    \caption{Q-learning vs DP value function Wasserstein distance.} \label{fig:all_wasserstein}
\end{figure}

\cref{fig:all_wasserstein}  extends \cref{fig:inv2_wasserstein} demonstrates for all the domains that the value function achieve from $\kappa$-soft quantile Q-learning (\cref{alg:q-learning}) for $\kappa \in \{10^{-4},10^{-8},10^{-12},0\}$ and uniform discretization of $J=256$  converges to the DP variant \cref{eq:discreteBellmanEq}. We use the Wasserstein-1 distance of quantile function defined as \[
W_1(\ushort{q}^\discretized,\ushort{\tilde{q}}^\discretized) := \frac{1}{J} \sum_{j \in [J-1]} | \max_{a \in \actions}\{ \ushort{q}^\discretized(s_0,j,a) \} - \max_{a \in \actions}\{\ushort{\tilde{q}}^\discretized(s_0,j,a)\} |
\] to evaluate the differences between two value functions. From \cref{fig:all_wasserstein} we can see that for all domains, the Wasserstein-1 distances of the value functions were far apart at the beginning and quickly converges to zero as the number of samples for each $(s,a)$ pair increase.

\newpage
\section{Comparison of VaR-Q-learning and IQN}\label{sec:appIQN}
\newcommand{\qdhat}{\hat{q}^\discretized}

We now compare the Implicit Quantile Network (IQN) Q-learning algorithm proposed in \cite{dabney2018implicit}  to a variant of our Q-learning algorithm that stochastically approximates the expected value operation over the sampling of next risk $j'\sim U([J-1])$ using $K'$ sampled risk level. Specifically, we focus on a version of IQN Q-learning that considers a finite horizon problem (using an additional $t$ state dimension), models the state-value function using a piecewise constant function of the risk level, i.e.  $q(t,s,\alpha,a) := \qdhat(t,s, \lfloor \alpha J\rfloor,a)$  with $\qdhat\in\mathbb{R}^{[T] \times\states\times  [J-1]\times \actions}$, and models the risk aversion using a distorted risk measure parameterized by some non-decreasing $\beta_{IQN}\colon [0,1]\rightarrow[0,1]$ and implied $\Gamma(j):=\E[\beta_{IQN}(\tilde{u})\mid j\leq\tilde{u}\leq (j+1)/J]$ for $\tilde{u}\sim U([0,1])$\mm{ It is not clear what $\Gamma$ means and how it is used. It is referring a quantity in the IQN paper? It is not used}. \Cref{alg:riskBasedQLearning,alg:q-learningIQN} present in a comparable format how a quantile MDP and IQN approach compute their respective loss when updating their respective approximate state-value functions.\mm{I think that we should make it clearer which one is the original IQN and which one is ours.}

The biggest distinction between the two algorithms lies in the computation of the action or actions associated to state $s'$. On one hand, IQN seeks for each sampled $(s,t,a,r,s')$ tuple a single action that captures a form of risk aversion portrayed by $\argmax_{a'\in\actions}\E[\qdhat(t-1,s',\beta_{IQN}(\tilde{u}),a')]=\argmax_{a'\in\actions}(1/J)\sum_{j'=0}^{J-1} \Gamma(j)\qdhat(t-1,s',j',a')$, where $\tilde{u}\sim U([0,1])$. On the other hand, the variant of our Quantile Q-learning algorithm seeks an optimal action for each sampled next state quantile level $j_{k'}'$. The latter reflects our finding that the quantile MDP can be solved by solving a nested VaR DP where the risk level is independently sampled from a uniform distribution at each time step. In comparison, it is not clear what criterion of optimality is satisfied by the policy evaluated by IQN; see \cite{lim2022distributional} for a discussion regarding the case where $\beta_{IQN}(\cdot)$ reflect a CVaR measure.

As part of the finer differences between the two algorithms, one can observe that our quantile Q-learning employs our $\kappa$-soft quantile loss, whereas IQN uses a Huber quantile loss, denoted by $\ell_\alpha^h(\cdot)$. We also use the quantile loss function associated to the discretized level $\lfloor \tau_k J\rfloor /J$ instead of using $\tau_k$ directly in order to guarantee a conservative approximation (instead of an estimation) of the value-at-risk of level $\tau_k$. We finally handle samples with $\tau_k< 1/J$ differently than the rest given that our \cref{alg:q-learning} prescribes steering the value of $\qud(t,s,0,a)$ towards the value of  $\ushort{R} t$ in order to produce a natural lower bound on the value-at-risk for risk levels in that lower range.

\begin{algorithm*}
\SetAlgoLined 
\textbf{Require:} $K$,$K'$,$\kappa$, and functions $\qud$\\
 \KwIn{$t$, $s$, $a$, $s'$}
 $\#$ Sample current quantile thresholds\\
$\tau_k\sim U([0,1])$, $\quad 1\leq k\leq K$\\
 $\#$ Sample next quantile thresholds\\
$\tau_{k'}'\sim U([0,1]), \quad 1\leq k'\leq K'$\\
 $\boldsymbol{\#}$ \textbf{Compute greedy next quantile-based actions}\\
 $\boldsymbol{a_{k'}\opt\gets \argmax_{a'\in\actions}\qud(t-1,s',\lfloor \tau_{k'}'J\rfloor,a'), \quad 1\leq k'\leq K'}$\\
 $\#$ Compute distributional temporal differences\\
 $\delta_{kk'}\gets r(s,a)+\gamma \qud(t-1,s',\lfloor \tau_{k'}'J\rfloor,\boldsymbol{a_{k'}\opt})- \qud(t,s,\lfloor\tau_k J\rfloor,a),\quad 1\leq k\leq K,\; 1\leq k'\leq K'$\\
$\#$ Compute $\kappa$-soft quantile loss\\
\textbf{Output:} $\sum_{k=1}^{K}(1/K)\sum_{k'=1}^{K'} \ell_{\boldsymbol{\lfloor\tau_k J\rfloor/J}}^{\boldsymbol{\kappa}}(\delta_{kk'})\boldsymbol{\cdot \mathbb{I}_{\tau_k\in[1/J,1]} + (\qud(t,s,0,a)-\ushort{R}t)^2\cdot \mathbb{I}_{\tau_k\in[0,1/J)}}$
 \caption{Risk Sampled-based Quantile Q-learning Loss (adapted from Algorithm \ref{alg:q-learning})}\label{alg:riskBasedQLearning}
\end{algorithm*}

\begin{algorithm*}
\SetAlgoLined 
\textbf{Require:} $K$,$K'$,$h$, and functions $\Gamma$, $\qdhat$\\
 \KwIn{$t$, $s$, $a$, $s'$}
 $\#$ Sample current quantile thresholds\\
$\tau_k\sim U([0,1])$, $\quad 1\leq k\leq K$\\
 $\#$ Sample next quantile thresholds\\
$\tau_{k'}'\sim U([0,1]), \quad 1\leq k'\leq K'$\\
  $\boldsymbol{\#}$ \textbf{Compute greedy next uniform action}\\
 $\boldsymbol{a\opt\gets \argmax_{a'\in\actions}(1/J)\sum_{j'=0}^{J-1} \Gamma(j')\qdhat(t-1,s',j',a')}$\\
 $\#$ Compute distributional temporal differences\\
 $\delta_{kk'}\gets r(s,a)+\gamma \qdhat(t-1,s',\lfloor \tau_{k'}'J\rfloor,\boldsymbol{a\opt})- \qdhat(t,s,\lfloor\tau_k J\rfloor,a),\quad 1\leq k\leq K,\; 1\leq k'\leq K'$\\
$\#$ Compute Huber quantile loss\\
\textbf{Output:} $\sum_{k=1}^{K}(1/K')\sum_{k'=1}^{K'} \ell_{\boldsymbol{\tau_k}}^{\boldsymbol{h}}(\delta_{kk'})$
 \caption{Implicit Quantile Network Loss (adapted from \cite{dabney2018implicit})}
\label{alg:q-learningIQN}
\end{algorithm*}

\newpage

\end{document}